\documentclass[10pt,journal,compsoc]{IEEEtran}
\usepackage{epsfig}
\usepackage{graphicx}
\usepackage{amsmath}
\usepackage{amsthm}
\usepackage{amssymb}             
\usepackage{amsfonts}            
\usepackage{mathrsfs}            
\usepackage{algorithm}
\usepackage{algorithmic}
\usepackage{makecell,rotating,multirow,diagbox}
\usepackage[normalem]{ulem}
\usepackage{subfigure}
\usepackage{bm}

\newtheorem{proposition}{Proposition} 

\begin{document}
%
\title{Deep Cross-modal Hashing via Margin-dynamic-softmax Loss}
%
%
%

\author{Rong-Cheng~Tu, Xian-Ling~Mao, Rong-Xin Tu, Binbin Bian, Wei~Wei and Heyan Huang
	\IEEEcompsocitemizethanks{\IEEEcompsocthanksitem R. Tu, X. Mao, B. Bian, and H. Huang are with the Department of Computer Science and Technology, Beijing Institute of Technology, Beijing 100081, China.\protect\\
		E-mail: \{tu\_rc, maoxl, 3220180838, hhy63\}@bit.edu.cn.
		\IEEEcompsocthanksitem R.X. Tu is with the school of Information Management, Jiangxi University of Finance and Economics, Nanchang 330013, China.\protect\\
		E-mail: rongxin\_tu@163.com.
		\IEEEcompsocthanksitem W. Wei is with School of Computer Science, Huazhong University of Science and Technology, Wuhan 430074, China.\protect\\
		E-mail: weiw@hust.edu.cn}
	\thanks{}
}

%
%

\markboth{Journal of \LaTeX\ Class Files,~Vol.~14, No.~8, August~2015}%
{Shell \MakeLowercase{\textit{et al.}}: Bare Demo of IEEEtran.cls for IEEE Journals}
%



\IEEEtitleabstractindextext{
\begin{abstract}
	Due to their high retrieval efficiency and low storage cost for cross-modal search task, cross-modal hashing methods have attracted considerable attention.  For the  supervised cross-modal hashing methods, how to make the learned hash codes preserve semantic information sufficiently contained in the label of datapoints is the key to further enhance the retrieval performance. Hence, almost all supervised cross-modal hashing methods usually depends on defining a similarity between datapoints with the label information to guide the hashing model learning fully or partly. However, the defined similarity between datapoints can only capture the label information of datapoints partially and misses abundant semantic information,   then hinders the further improvement of retrieval performance.  Thus, in this paper, different from previous works, we propose a novel cross-modal hashing method without defining the similarity between datapoints, called Deep Cross-modal Hashing via \textit{Margin-dynamic-softmax Loss} (DCHML). Specifically, DCHML first trains a proxy hashing network to transform each category information of  a dataset into a semantic discriminative hash code, called proxy hash code. Each proxy hash code can preserve the semantic information of its corresponding category well. Next, without defining the  similarity between datapoints to supervise the training process of the modality-specific hashing networks , we propose a novel \textit{margin-dynamic-softmax loss} to directly utilize the proxy hashing codes as supervised information. Finally, by minimizing the novel \textit{margin-dynamic-softmax loss}, the modality-specific hashing networks can be trained to generate hash codes which can simultaneously preserve the cross-modal similarity and abundant semantic information well. Extensive experiments on three benchmark datasets show that the proposed method outperforms the state-of-the-art baselines in cross-modal retrieval task.
\end{abstract}

}

\maketitle


%
\IEEEpeerreviewmaketitle

\section{Introduction}
With the fast development of the Internet, tremendous amounts of multimedia data such as images, texts and videos have been generated every day. Commonly, relevant data from different modalities may have semantic correlations. Thus, the cross-modal search techniques become more and more important. The goal of cross-modal search is to search semantically similar
instances from one modality by using a query
from another modality, for example, searching images (texts) with texts (images) as queries. Among the existing cross-modal search techniques, similarity-preserving cross-modal hashing methods are advantageous due to their high retrieval efficiency and low storage cost.  The core of similarity-preserving cross-modal hashing methods is to map datapoints from different modalities into a common Hamming space with their original semantic similarity preserved well.

Generally speaking, existing cross-modal hashing methods can be divided into two categories: unsupervised methods and supervised methods. Unsupervised cross-modal hashing methods \cite{da2018nonlinear,hu2020creating,jiang2019discrete,li2016linear,liu2018fast} mainly learn hashing functions from unlabeled data, typical methods including  Cross View Hashing (CVH) \cite{kumar2011learning}, Inter-Media Hashing (IMH) \cite{song2013inter}, Collective Matrix Factorization Hashing (CMFH) \cite{ding2014collective} and Deep Joint Semantics Reconstructing Hashing (DJSRH) \cite{su2019deep}. Compared with unsupervised cross-modal hashing methods, supervised methods \cite{cao2017collective,cao2016correlation,chen2018dual,lin2020semantic,liong2018deep,tu2020deep} can achieve better performance in the retrieval task by using the semantic labels to supervise the training process of hashing models. Recently, by integrating the feature learning and hash codes learning into end-to-end deep networks, deep supervised cross-modal hashing methods
\cite{li2018self,lin2020semantic,shen2017deep,shi2019equally} have shown state-of-the-art performance in cross-modal retrieval task.

For deep supervised cross-modal hashing methods, how to make the learned hash codes preserve semantic information sufficiently is a key point to further enhance the retrieval performance. 
As far as we know, almost all supervised cross-modal hashing methods usually depends on defining a similarity between datapoints with the label information to guide the hashing model learning fully or partly. However, the defined similarity between datapoints can only capture the label information of datapoints partially and misses abundant semantic information,   then hinders the further improvement of retrieval performance. Furthermore, with the defined similarity between datapoints, existing supervised methods usually need to construct  pairwise or triplet loss function to optimize the hashing networks, which leads to the following two problems. First, they cannot make full use of all available datapairs, because it is hard to exhaustively use all the similarities of datapoints to train the hashing networks. Second, it is hard to find the informative datapair or triplet to construct the loss function to train the hashing models.  Thus, based on the similarity of datapoints as guide information, existing supervised methods cannot generate hash codes with the cross-modal similarity and abundant semantic label information preserved well.

Thus, to tackle this problem, in this paper, we propose a novel \textit{margin-dynamic-softmax loss}, then we propose a novel Deep Cross-Modal Hashing via Margin-dynamic-softmax Loss, called DCHML. Specifically, inspired by the class-level code based single-modal hashing methods \cite{huang2016class,luo2019discrete}, DCHML first trains a proxy hashing network to learn a hash code for each category, and the learned hash code contains sufficiently semantic information of its corresponding category. Thus, the discriminative hash code can be treated as a proxy of the category, called proxy hash code. Then, the novel \textit{margin-dynamic-softmax loss} utilizes the learned proxy hash codes as supervised information to train the modality-specific hashing networks without defining the semantic similarity between datapoints. By minimizing the novel \textit{margin-dynamic-softmax loss}, the modality-specific hashing networks will be trained to map their corresponding modal datapoints into a common Hamming space where the cross-modal similarity and abundant semantic label information will be preserved well.
Extensive experiments on three benchmark datasets show that the proposed method outperforms the state-of-the-art baselines in cross-modal retrieval task. 
\section{Related Work}
In general, cross-modal hashing methods can be roughly divided into two categories: unsupervised cross-modal hashing methods and supervised cross-modal hashing methods.

\subsection{Unsupervised Cross-modal Hashing Methods}
Unsupervised cross-modal hashing methods \cite{song2013inter,hu2020creating,su2019deep,wang2015semantic,xu2017learning} mainly use unlabeled data to learn hashing functions, such as Inter-Media Hashing (IMH) \cite{song2013inter}, Collective Matrix Factorization Hashing (CMFH) \cite{ding2014collective},  Cross View Hashing (CVH) \cite{kumar2011learning}, Robust and Flexible Discrete Hashing (RFDH) \cite{wang2017robust}, and Deep Joint Semantics Reconstructing Hashing (DJSRH) \cite{su2019deep}. IMH maps heterogeneous
multimedia data into hash codes by constructing graphs. It learns the hash functions by linear regression for new instances. Its joint learning scheme can effectively preserve the inter- and intra- modality consistency.
CMFH, which fuses the multi-modal information to improve the retrieval accuracy, learns unified binary hash codes by doing matrix factorization with latent factor model in the training phase. 
CVH is a typical graph-based unsupervised cross-modal hashing method extended from the single-modal spectral hashing, and it learns hash functions to map the original cross-modal datapoints into binary hash codes with cross-modality similarity preserved by minimizing the weighted Hamming distances.  
RFDH first learns unified hash codes for the training datapoints by using discrete collaborative matrix factorization, and then it enhance the robustness and flexibility of the hash codes by jointly adopting l2,1-norm and adaptively weight of each modality.
DJSRH proposes a joint-semantics affinity matrix to learns the binary hash codes.

\subsection{Supervised Cross-modal Hashing Methods}
Supervised cross-modal hashing methods \cite{wang2015semantic,zhang2014large,cao2018cross,erin2017cross,lin2020semantic,luo2018sdmch,shi2019equally,wang2017adversarial} mainly exploit the label information of datapoints to learn hashing functions to improve performance. Early works mainly utilize the hand-crafted features of datapoint to learn shallow architecture based hashing functions to project them into  binary hash codes, such as Semantic Correlation Maximization (SCM) \cite{zhang2014large}, Semantic Topic Multimodal hashing (STMH) \cite{wang2015semantic}, Discrete Cross-modal Hashing (DCH) \cite{xu2017learning}, Label Consistent Matrix Factorization Hashing \cite{wang2018label} and Discrete Latent Factor hashing (DLFH) \cite{jiang2019discrete}. SCM learns to project datapoints into hash codes with maximum semantic information preserved by taking semantic labels into consideration.
STMH leverages semantic modeling to capture different semantic topics for texts and images respectively, and then projects the captured semantic features to the binary hash codes.
DCH  learns a set of modality-dependence hash projections as well as discriminative binary codes to keep the classification consistent with the label for multi-modal data.
LCMFH utilizes a auxiliary matrix to map the original datapoints to the low-dimensional latent space, and then quantifies them to the binary hash codes with semantic label supervised.
DLFH directly optimizes hash codes without continuous relaxation by proposing an efficient hash learning algorithm based on the discrete latent factor model.

Recently, by integrating the feature learning and hash codes learning into end-to-end deep networks, deep supervised cross-modal hashing methods
\cite{cao2018cross,jiang2017deep,li2018self,shi2019equally,xu2019graph,lin2020semantic} have shown state-of-the-art performance. For example, Deep cross-modal hashing (DCMH) \cite{jiang2017deep} utilizes a negative log-likelihood loss to generate cross-modal similarity preserving hash codes by an end-to-end deep learning framework. 
Cross-modal deep variational hashing (CMDVH) \cite{erin2017cross} uses a two step framework. In the first step the method learns unified hash code for image-text pair in a database, and utilize the learned unified hash codes to learn hashing functions in the second step. SSAH \cite{li2018self} generates binary hash codes by utilizing two adversarial networks to jointly model different modalities and capture their semantic relevance under the supervision of the learned semantic feature. 
EGDH \cite{shi2019equally} cooperates hashing-based retrieval with classification to generate hash codes with the intra-class aggregated and inter-class relationship preserved. 
SDCH \cite{lin2020semantic} utilizes a semantic label branches to preserve semantic information of the learned features by integrating with inter-modal pairwise loss, cross-entropy loss and quantization loss.

\begin{figure*}[tb]
	\centering
	\includegraphics[width=\textwidth]{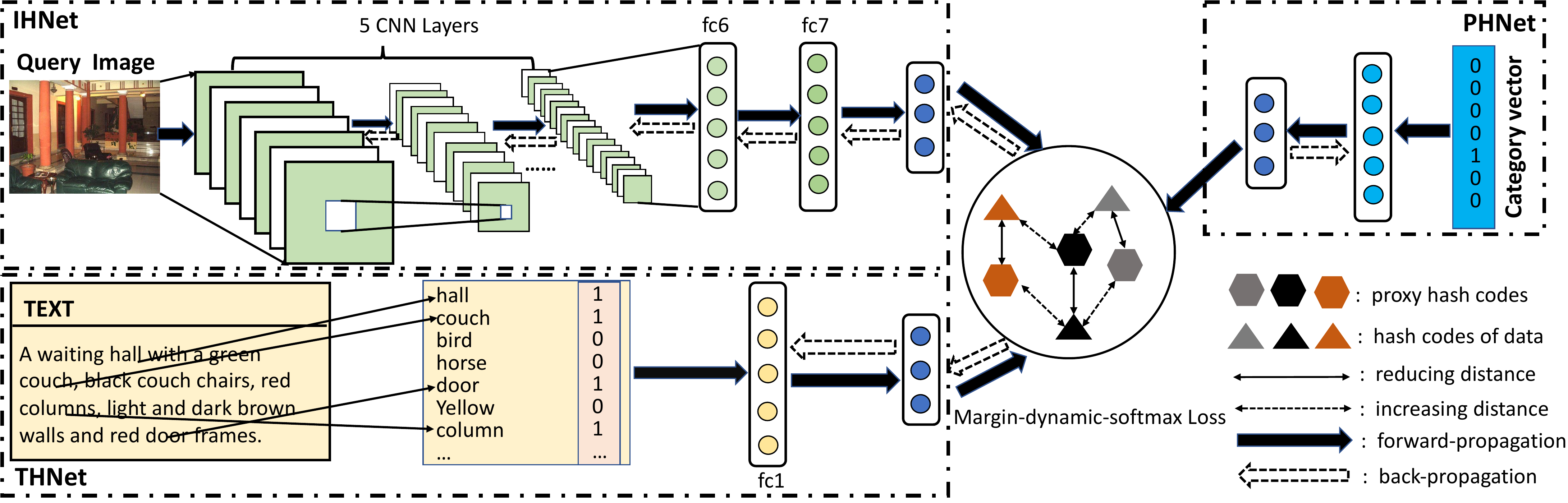}
	\caption{The learning framework of DCHML. It consists of three modules:  a proxy hashing network (PHNet), an image hashing network (IHNet) and a text hashing network (THNet). Specifically, the PHNet is used to map each category in the dataset into compact proxy hash code which preserve the semantic information of its corresponding category well. Then with the learned proxy hash codes as the supervised information, the modality-specific hashing networks, IHNet and THNet, can be trained well to map their corresponding modal datapoints into hash codes with the semantic structure information preserved well.  
	}
	\label{fig_architecture} 
\end{figure*}

Almost all deep supervised cross-modal hashing methods learn hashing functions depending on defining a similarity between datapoints with the label information fully \cite{cao2018cross,jiang2017deep} or partly \cite{shi2019equally}. However, the defined similarity between datapoints can only capture the label information of datapoints partially and misses abundant semantic information,  then hinders the further improvement of retrieval performance. Different from these previous works, inspired by the class-level code based single-modal hashing methods \cite{huang2016class,luo2019discrete}, our proposed method first transforms each category information into proxy hash code; then, by using the proxy hash codes as supervised information, a novel \textit{margin-dynamic-softmax loss} is first proposed to optimize hashing model to generate sufficiently semantic structure information preserved hash codes. Please note that our proposed method DCHML is really different from the class-level code based single-modal hashing methods \cite{huang2016class,luo2019discrete}: First, such class-level code based single-modal hashing methods are shallow hashing methods, but our proposed method DCHML is deep architecture that both the proxy hash codes and image hash codes are generated by deep hashing network; Second, DCHML proposes a novel \textit{margin-dynamic-softmax loss}. By minimizing the \textit{margin-dynamic-softmax loss}, the modality-specific hashing networks can generate hash codes which preserve semantic structure information sufficiently.

\section{Our Method}
In this section, we will first give the problem definition in subsection \ref{pf}. The whole architecture of DCHML will be introduced in subsection \ref{ar}. Then, in subsection \ref{pm}, we will introduce how the PHNet generates proxy hash code for each category in detail. Furthermore, we will introduce how to train the modality-specific hashing networks by utilizing the proxy hash codes as supervised information in subsection \ref{sm}. Finally, we will introduce the out-of-sample extension in subsection \ref{oe}. 
\subsection{Problem Formulation}
\label{pf}
Without loss of generality, we focus on cross-modal retrieval for image-text datasets. Suppose we have a dataset contained $n$ instances with image-text pairs, denoted by $\boldsymbol{O}=\{\boldsymbol{o}_i\}_{i=1}^n$, and the corresponding label matrix is given as $\boldsymbol{L} = \{\boldsymbol{l}_i\}_{i=1}^n \in \{0, 1\}^{c\times n}$, where $c$ is the total number of categories in the dataset. For each instance $\boldsymbol{o}_i = (\boldsymbol{x}_i^v, \boldsymbol{x}_i^t)$, $\boldsymbol{x}_i^v$ and $\boldsymbol{x}_i^t$ denote the image and text datapoints in the $i^{th}$ instance $\boldsymbol{o}_i$ respectively, and the corresponding label vector is $\boldsymbol{l}_i = [l_{i1}, l_{i2}, \dots, l_{ic}]^T$. If $\boldsymbol{o}_i$ belongs to the $j^{th}$ category, then $l_{ij} = 1$, otherwise $l_{ij} = 0$. Furthermore, we use $\boldsymbol{Y} = \{\boldsymbol{y}\}_{i = 1}^c \in \{0, 1\}^{c\times c}$ to define the representation of categories, where $\boldsymbol{y}_i$ is the one-hot vector of the $i^{th}$ category, i.e., only $y_{ii} = 1$, and the other elements of $\boldsymbol{y}_i$ are $0$.

The goal of deep cross-modal hashing is to learn two modality-specific hashing functions $H^v:\boldsymbol{x}_i^v \rightarrow \boldsymbol{b}_i^v \in \{-1,1\}^k$ and $H^t:\boldsymbol{x}_i^t \rightarrow \boldsymbol{b}_i^t \in \{-1,1\}^k$, where $k$ is the length of hash codes, to map image and text datapoints into hash codes with the cross-modal semantic similarity preserved well.
\subsection{Architecture}
\label{ar}
As shown in Figure 1, our framework consists of three components: a proxy hashing network (PHNet), an image hashing network (IHNet) and a text hashing network (THNet). 

The PHNet is introduced to generate a discriminative compact proxy hash code for each category. The input of PHNet is the one-hot vector $\boldsymbol{y}$ of each category. PHNet contains two fully-connected layers: the first fully-connected layer has 512 hidden units, and the second fully-connected layer has $k$ hidden units where $k$ is the length of hash codes. The activation function for the first fully-connected layer is RELU \cite{krizhevsky2012imagenet}, and for the second fully-connected layer, a $tanh(\cdot)$ function is used as an activation function.

The IHNet is used to map images to informative hash codes. IHNet contains five convolutional layers and three fully-connected layers. Moreover, the five convolutional layers and the first two fully-connected layers are the same with the first seven layers of AlexNet \cite{krizhevsky2012imagenet}. The third fully-connected layer consists of $k$ hidden units where $k$ is the length of hash codes, and a $tanh(\cdot)$ function is used as an activation function.

The THNet is used to map text datapoints to informative hash codes. The THNet is a two-fully-connected neural network. The first fully-connected layer consists of $2,048$ hidden units, and the activation function for it is RELU \cite{krizhevsky2012imagenet}. The second layer consists of $k$ nodes where $k$ is the length of hash codes, and a $tanh(\cdot)$ function is used as an activation function.

\subsection{PHNet Learning}
\label{pm}
The goal of PHNet is to project each category $\boldsymbol{y}_i$ into a proxy hash code $\boldsymbol{g}_i\in \{-1,1\}^k$. $k$ is the length of hash codes. To make the proxy hash code $\boldsymbol{g}_i$  be the proxy of category $\boldsymbol{y}_i$, the proxy hash codes of different categories should be dissimilar, i.e., the hamming distance between different proxy hash codes should be large. As the Hamming distance between $\boldsymbol{g}_i$ and $\boldsymbol{g}_j$ is defined as $d_H(\boldsymbol{g}_i, \boldsymbol{g}_j) = \frac{1}{2}(k - \boldsymbol{g}_i^T\boldsymbol{g}_j)$, thus the smaller inner products of $\boldsymbol{g}_i$ and $\boldsymbol{g}_j$ is, the larger $d_H(\boldsymbol{g}_i, \boldsymbol{g}_j)$ is. Thus, we can formalize the loss function of PHNet as follows:
\begin{equation}
\begin{aligned}
\min\limits_{\boldsymbol{\Theta}} \mathcal{L} &=  \sum\limits_{i = 1}^c \sum\limits_{j \not= i}\max(0, \boldsymbol{g}_i^T\boldsymbol{g}_j) + \alpha\sum\limits_{i = 1}^k\left\| \boldsymbol{a}_i\right\|_F^2.\\
&\ \ \ \ s.t. \ \ \ \ \boldsymbol{g}_i = sgn(\mathcal{F}(\boldsymbol{y}_i;\boldsymbol{\Theta}))
\end{aligned}
\end{equation}
where $\mathcal{F}(\boldsymbol{y}_i;\boldsymbol{\Theta})$ is the output of PHNet with the input category $\boldsymbol{y}_i$, and $\boldsymbol{\Theta}$ denotes the set of parameters of PHNet; $sgn(\cdot)$ is an element-wise function which returns $``1"$ if the element is positive and returns $``-1"$ otherwise; $\alpha$ is a hyper-parameter; and $\boldsymbol{a}_i$ is formulated as:
\begin{equation}
\begin{aligned}
\forall i = 1, 2,\cdots, k \ \ \ \ \boldsymbol{a}_i = \sum\limits_{j = 1}^c \boldsymbol{b}_j(i)
\end{aligned}
\end{equation}
where $ \boldsymbol{b}_j(i)$ denotes the $i^{th}$ element of the hash code $ \boldsymbol{b}_j$.

Specifically, it can be found that by minimizing the first term of $\mathcal{L}$, the inner product of $\boldsymbol{g}_i$ and $\boldsymbol{g}_j$ is smaller than zero, i.e., the Hamming distance $d_H(\boldsymbol{g}_i, \boldsymbol{g}_j)$ is larger than $\frac{k}{2}$. It means each proxy hash code $\boldsymbol{g}_i$ is distinguished to the other proxy hash codes, thus  the proxy of category $\boldsymbol{y}_i$ can be denoted as $\boldsymbol{g}_i$. Moreover, the second term of $\mathcal{L}$ is a widely used regular term which can make the learned hash codes express more information.

Furthermore,as the $sgn(\cdot)$ function is non-differentiable at zero and the derivation of it will be zeros for a non-zero input, thus the parameters of PHNet will not be updated with the back-propagation algorithm when minimizing the loss function $\mathcal{L}$. Similar to previous methods \cite{shi2019equally,tumls3rduh}, we directly discard the $sgn(\cdot)$ function to ensure the parameters of PHNet can be updated, and add a quantization loss to make each element of output of PHNet can be close to $``+1"$ or $``-1"$. Then, the final loss function can be formulated as follows:
\begin{equation}
\begin{aligned}
\min\limits_{\boldsymbol{\Theta}}\mathcal{L}_{PHNet} &=  \sum\limits_{i = 1}^c \sum\limits_{j \not= i}\max(0, \boldsymbol{\hat{g}}_i^T\boldsymbol{\hat{g}}_j)  + \alpha\sum\limits_{i = 1}^k\left\| \boldsymbol{\hat{a}}_i\right\|_F^2\\
&\ \ \ \ \ + \beta\sum\limits_{i = 1}^c\left\| \boldsymbol{\hat{g}}_i - sgn(\boldsymbol{\hat{g}}_i)\right\|_F^2\\
&\ \ \ s.t. \ \ \ \boldsymbol{\hat{g}}_i = \mathcal{F}(\boldsymbol{y}_i;\boldsymbol{\Theta}).
\end{aligned}
\label{ploss}
\end{equation}
where $\boldsymbol{\hat{g}}_i = \mathcal{F}(\boldsymbol{y}_i;\boldsymbol{\Theta})$ is the output of PHNet with the input $\boldsymbol{y}_i$, and the proxy hash code $\boldsymbol{g}_i=sgn(\boldsymbol{\hat{g}}_i)$; $\beta$ is a hyper-parameter.

Finally, by minimizing the loss function $\mathcal{L}_{PHNet}$, the parameters of PHNet can be updated by the back propagation algorithm. Thus, after training the PHNet, we can use it to generate the proxy hash codes $\boldsymbol{G} =\{\boldsymbol{g}_i\}_{i=1}^c$ according $\boldsymbol{g}_i = sgn(\mathcal{F}(\boldsymbol{y}_i;\boldsymbol{\Theta}))$.

\subsection{Modality-specific Hashing Network Learning}
\label{sm}
As the learned proxy hash codes preserve abundant semantic information of their corresponding categories, thus these proxy hash codes can be explored as the novel supervised information and be incorporated into the training process of modality-specific hashing networks, i.e., IHNet and THNet. Specifically, when the average of the Hamming distances between $\boldsymbol{b}_i^m$ ($m=v$ or $t$) and the proxy hash codes of categories that $\boldsymbol{x}_i^m$ belongs to is smaller than the Hamming distance between $\boldsymbol{b}_i^m$ and each proxy hash code of category that $\boldsymbol{x}_i^m$ does not belong to, then the learned hash code $\boldsymbol{b}_i^m$ will simultaneously preserve the cross-modal similarity and abundant semantic information well. Moreover, it can be formulated as follows:
\begin{equation}
\forall\ \  q \in \Omega_i, \ \ \ \ \ \frac{1}{|\Upsilon_i|}\sum\limits_{e\in\Upsilon_i}d_H(\boldsymbol{b}_i^m, \boldsymbol{g}_e) + \mu k \leq d_H(\boldsymbol{b}_i^m, \boldsymbol{g}_q).
\label{mutil}
\end{equation}
where $\mu \in [0, 1]$ is a predefined margin, and $k$ is the length of a hash code. $\Upsilon_i$ is the set of indexes of categories which the datapoint $\boldsymbol{x}_i^m$ belongs to, i.e., the indexes of element $``1"$ in the label vector $\boldsymbol{l}_i$, and $\Omega_i$ is the set of indexes of categories which the datapoint $\boldsymbol{x}_i^m$ does not belong to, i.e., the indexes of element $``0"$ in the label vector $\boldsymbol{l}_i$; $d_H(\boldsymbol{b}_i^m, \boldsymbol{g}_e)$ is the Hamming distance between $\boldsymbol{b}_i^m$ and $ \boldsymbol{g}_e$, which is defined as $d_H(\boldsymbol{b}_i^m, \boldsymbol{g}_e) = \frac{1}{2}(k - \boldsymbol{b}_i^{mT}\boldsymbol{g}_e)$.

\begin{figure}[tb]
	\centering
	\includegraphics[width=\linewidth]{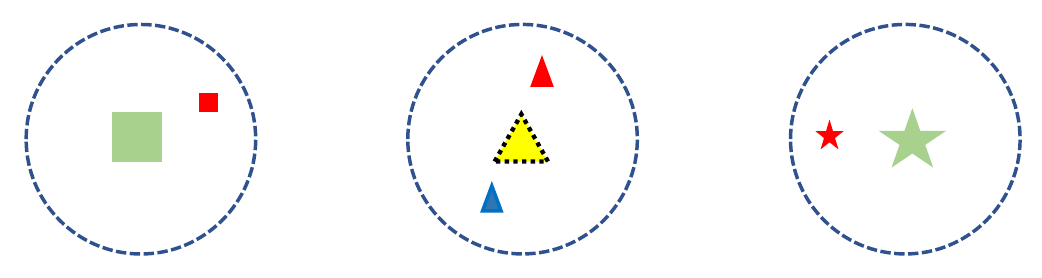}
	\caption{Suppose there are two categories: \textit{``square"} and \textit{``pentagram"}, and their corresponding proxy hash codes are the \textit{``green square"} and \textit{``green pentagram"}, respectively. The \textit{``red square"} is the hash code of an image point that only belongs to the \textit{``square"} category, and \textit{``red pentagram"} is the hash code of an image point that only belongs to the \textit{``pentagram"} category. Both the \textit{``red triangle"} and \textit{``blue triangle"} denote the hash codes of images that belong to the two categories: \textit{``square"} and \textit{``pentagram"}, and the \textit{``yellow triangle "} is the surrogate proxy point of \textit{``green square"} and \textit{``green pentagram"}.
	}
	\label{fig_illustrate} 
\end{figure} 
Here, we give the reason why the learned hash code $\boldsymbol{b}_i^m$ will simultaneously preserve the cross-modal similarity and abundant semantic structure information well, when Formula (\ref{mutil}) holding. First,  the Formula (\ref{mutil}) can be further reformulated as follows:
\begin{equation}
\forall\ \  q \in \Omega_i, \ \ \ \ \ \boldsymbol{b}_i^{mT}(\frac{1}{|\Upsilon_i|}\sum\limits_{e\in \Upsilon_i} \boldsymbol{g}_e) - \boldsymbol{b}_i^{mT}\boldsymbol{g}_q \geq \mu k.
\label{mutil1}
\end{equation}
where $\frac{1}{|\Upsilon_i|}\sum\limits_{e\in \Upsilon_i} \boldsymbol{g}_e$ is the average of the similar proxy hash codes of $\boldsymbol{b}_i^m$, and it can be treated as surrogate proxy point which $\boldsymbol{b}_i^m$ should be similar to. Furthermore, if $\boldsymbol{b}_i^m$ can be more similar with its corresponding surrogate proxy point $\frac{1}{|\Upsilon_i|}\sum\limits_{e\in \Upsilon_i} \boldsymbol{g}_e$ than the proxy hash codes in the set $\{\boldsymbol{g}_q|q\in \Upsilon_i\}$, $\boldsymbol{b}_i^m$ will simultaneously preserve the semantic similarity and semantic label information well. For example, as shown in Figure \ref{fig_illustrate}, when Formula (\ref{mutil1}) is true, \textit{``red square"} is similar to \textit{``blue triangle"} and \textit{``blue triangle"},  and is dissimilar to the \textit{``red pentagram"}, i.e., the hash codes preserve the semantic similarity; \textit{``red triangle"} and \textit{``blue triangle"} are both similar to the \textit{``yellow triangle "} with the \textit{``green square"} and \textit{``green pentagram"} information preserved, i.e., the hash codes preserve semantic label information well.

Now the key challenge is how to  make hash codes $\boldsymbol{b}_i^m$ to satisfy Formula (\ref{mutil1}).
Inspired by \cite{qian2019softtriple,tu2021partial}, we have the following proposition. 
\begin{proposition}
	Given a hash codes $\boldsymbol{b}_i^m$, a set of proxy hash codes $\boldsymbol{G} = \{ \boldsymbol{g}_i\}_{i=1}^c$ and the following functions:
	\begin{equation}
	\begin{aligned}
	\mathcal{L}_{S}^m(\boldsymbol{b}_i^m)= -log\frac{exp(\eta \boldsymbol{u}_i^m)}{exp(\eta \boldsymbol{u}_i^m) + \sum\limits_{q\in\Omega_i}exp(\eta \boldsymbol{b}_i^{mT}\boldsymbol{g}_q)}
	\end{aligned}
	\label{ls}
	\end{equation}
	\begin{equation}
	\begin{aligned}
	\mathcal{J}^m(\boldsymbol{b}_i^m) &= \max\limits_{\boldsymbol{p} \in \Phi}(\sum\limits_{q\in\Omega_i} \eta\boldsymbol{p}_q(\boldsymbol{b}_i^{mT}\boldsymbol{g}_q - \boldsymbol{u}_i^m) \\
	&+ \eta\boldsymbol{p}_{0}(\boldsymbol{u}_i^m - \boldsymbol{u}_i^m) + H(\boldsymbol{p}))
	\end{aligned}
	\label{p1}
	\end{equation}
	where $\boldsymbol{u}_i^m = \boldsymbol{b}_i^{mT}\boldsymbol{\bar{g}}_i - \mu k$, $\boldsymbol{\bar{g}}_i = \frac{1}{|\Upsilon_i|}\sum\limits_{e\in \Upsilon_i} \boldsymbol{g}_e$ and suppose $\boldsymbol{\bar{g}}_i$ is the proxy hash code of a new category that $\boldsymbol{x}_i^m$ belong to, called surrogate category ``0''; $\boldsymbol{p} \in \mathcal{R}^{c +1}$ is a distribution over the surrogate category ``0'' and the categories of the dataset, where $c$ is the total number of categories of the dataset; $\Phi$ is the simplex as $\Phi = \{\boldsymbol{p}|\sum\limits_{j \not\in \Omega_i \cup\{0\}} \boldsymbol{p}_j = 0, \sum\limits_{j \in \Omega_i \cup\{0\}} \boldsymbol{p}_j = 1,\  \forall  j, \boldsymbol{p}_j\geq0\}$; $H(\boldsymbol{p})$ denotes the entropy of the distribution $\boldsymbol{p}$.
	Then the following equation holds:
	\begin{equation}
	\mathcal{L}_{S}^m(\boldsymbol{b}_i^m) = \mathcal{J}(\boldsymbol{b}_i)
	\end{equation}
	\label{proposition}
\end{proposition}
\begin{proof}
	According to the \textbf{K.K.T.} condition \cite{boyd2004convex}, the distribution $\boldsymbol{p}$ in Equation (\ref{p1}) has the closed-form solution:
	\begin{equation}
	\boldsymbol{p}_j = \left\{
	\begin{array}{lrc}
	\frac{exp(\eta (\boldsymbol{b}_i^{mT}\boldsymbol{g}_j - \boldsymbol{u}_i^m))}{1 + \sum\limits_{j\in\Omega_i}exp(\eta (\boldsymbol{b}_i^{mT}\boldsymbol{g}_j - \boldsymbol{u}_i^m))},   && {j \in \Omega_i},\\
	\frac{1}{1 + \sum\limits_{j\in\Omega_i}exp(\eta (\boldsymbol{b}_i^{mT}\boldsymbol{g}_j - \boldsymbol{u}_i^m))},    && {j=0}, \\
	0, &&{otherwise}
	\end{array} \right.
	\label{p}
	\end{equation}
	Thus, we have:
	\begin{equation}
	\begin{aligned}
	\mathcal{J}^m(\boldsymbol{b}_i^m) &= \sum\limits_{q\in\Omega_i} \eta\boldsymbol{p}_q(\boldsymbol{b}_i^{mT}\boldsymbol{g}_q - \boldsymbol{u}_i^m) + \eta\boldsymbol{p}_0(\boldsymbol{u}_i^m - \boldsymbol{u}_i^m) + H(\boldsymbol{p})\\
	&= log(1 + \sum\limits_{q\in\Omega_i}exp(\eta(\boldsymbol{b}_i^{mT}\boldsymbol{g}_q - \boldsymbol{u}_i^m)))\\
	&=-log\frac{exp(\eta \boldsymbol{u}_i^m)}{exp(\eta \boldsymbol{u}_i^m) + \sum\limits_{q\in\Omega_i}exp(\eta \boldsymbol{b}_i^{mT}\boldsymbol{g}_q)}\\
	&=\mathcal{L}_{S}^m(\boldsymbol{b}_i^m).
	\end{aligned}
	\end{equation}
\end{proof}
Based on Proposition \ref{proposition},  it can be found that minimizing the Formula (\ref{ls}) is equivalent to minimizing the Formula (\ref{p1}). Furthermore, the Formula (\ref{p1}) and Formula (\ref{mutil1}) have the following relationship: Primarily, when the Formula (\ref{p1}) without the entropy regularizer, it becomes the follows:
\begin{equation}
\begin{aligned}
\mathcal{J}_1^m(\boldsymbol{b}_i^m) &= \max\limits_{\boldsymbol{p} \in \Phi}(\sum\limits_{q\in\Omega_i} \eta\boldsymbol{p}_q(\boldsymbol{b}_i^{mT}\boldsymbol{g}_q - \boldsymbol{u}_i^m) + \eta\boldsymbol{p}_{0}(\boldsymbol{u}_i^m - \boldsymbol{u}_i^m))
\end{aligned}
\end{equation}
which is equivalent to:
\begin{equation}
\mathcal{J}_1^m(\boldsymbol{b}_i^m)= \left\{
\begin{array}{lrc}
0,   && {if\ condition1},\\
\max\limits_{q\in\Omega_i}(\eta\boldsymbol{b}_i^{mT}\boldsymbol{g}_q)-\eta\boldsymbol{u}_i^m,    && {otherwise}.
\end{array} \right.
\end{equation}
where $condition1$ is: $\forall\ \  q \in \Omega_i,  \boldsymbol{b}_i^{mT}\boldsymbol{g}_q - \boldsymbol{u}_i^m \leq 0$. Thus, the goal of minimizing $ \mathcal{J}_1^m(\boldsymbol{b}_i^m)$ is to make  $\forall\ \  q \in \Omega_i, \  \boldsymbol{b}_i^{mT}\boldsymbol{\bar{g}}_i  - \boldsymbol{b}_i^{mT} \boldsymbol{g}_q \geq \mu k$, i.e., Formula (\ref{mutil1}) hold. Furthermore, to reduce the influence from outliers and makes the $\mathcal{J}_1^m(\boldsymbol{b}_i^m)$ more robust, an entropy regularizer $H(\boldsymbol{p})$ is added to $\mathcal{J}_1^m(\boldsymbol{b}_i^m)$ that it becomes Formula (\ref{p1}).  
Thus, minimizing Formula (\ref{p1}) is to make the Formula (\ref{mutil1}) hold. Finally, it means minimizing the Formula (\ref{ls}) can make hash codes $\boldsymbol{b}_i^{m}$ satisfy Formula (\ref{mutil1}).

Thus, to  make hash codes $\boldsymbol{b}_i^m$ to satisfy Formula (\ref{mutil1}), we propose the following  novel $margin$-$dynamic$-$softmax$ loss $\mathcal{L}_S^m$ as the loss function for the hashing network of modality $m $ ($m = t$ or $v$, i.e., text modality or image modality), which can be formulated as:
\begin{equation}
\begin{aligned}
\min\limits_{\boldsymbol{W}^m}\mathcal{L}_{S}^m &= \sum\limits_{i=1}^n -log\frac{exp(\eta \boldsymbol{u}_i^m)}{exp(\eta \boldsymbol{u}_i^m) + \sum\limits_{q\in\Omega_i}exp(\eta \boldsymbol{b}_i^{mT}\boldsymbol{g}_q)} \\
&\ \ \ s.t. \ \ \ \ \boldsymbol{u}_i^m = \boldsymbol{b}_i^{mT}\boldsymbol{\bar{g}}_i -\mu k,\\
& \ \ \ \ \ \ \ \ \ \ \ \ \ \boldsymbol{b}_i^m = sgn(\mathcal{H}^m(\boldsymbol{x}_i^m;\boldsymbol{W}^m)).
\end{aligned}
\label{f1}
\end{equation}
Where $\mathcal{H}^m(\boldsymbol{x}_i^m;\boldsymbol{W}^m)$ is the output of modality-specific hashing network  with a datapoint $\boldsymbol{x}_i^m$ as input, and $\boldsymbol{W}^m$ represents the set of parameters in the hashing network of modality $m$; $\mu $ is a predefined margin; $k$ is the hash code length; $\eta$ is a hyper-parameter; $\boldsymbol{\bar{g}}_i = \frac{1}{|\Upsilon_i|}\sum\limits_{e\in \Upsilon_i} \boldsymbol{g}_e$. As  the denominator of softmax-like loss of a datapoint is dynamic according to the label information of the datapoint, and there is a margin for the positive categories of the datapoint, thus we called the Formula (\ref{f1}) as $margin$-$dynamic$-$softmax$ loss.

Furthermore, it can be found the Formula (\ref{f1}) only make the hashing codes of datapoints from different modality keep the semantic label information separately without the inter-modal interaction. Considering that for an instance $\boldsymbol{o}_i$, its image modal datapoint $\boldsymbol{x}_i^v$ and its text modal datapoint $\boldsymbol{x}_i^t$ describe the same object, thus their corresponding hash codes should be  as similar as possible.  Then, we further add a inter-modal loss which can be formula as follows:
\begin{equation}
\begin{aligned}
\min\limits_{\boldsymbol{W}^v, \boldsymbol{W}^t}\mathcal{L}_{I} &= \sum\limits_{i=1}^n -log\frac{exp(\eta \boldsymbol{u}_i^v)}{exp(\eta \boldsymbol{u}_i^t) + \sum\limits_{q\in\Omega_i}exp(\eta \boldsymbol{b}_i^{tT}\boldsymbol{g}_q)} \\
&+ \sum\limits_{i=1}^n -log\frac{exp(\eta \boldsymbol{u}_i^t)}{exp(\eta \boldsymbol{u}_i^v) + \sum\limits_{q\in\Omega_i}exp(\eta \boldsymbol{b}_i^{vT}\boldsymbol{g}_q)} \\
&\ \ \ s.t. \ \ \ \ \boldsymbol{u}_i^v = \boldsymbol{b}_i^{vT}\boldsymbol{\bar{g}}_i -\mu k,\\
&\ \ \ \ \ \ \ \ \ \ \ \ \ \boldsymbol{u}_i^t = \boldsymbol{b}_i^{tT}\boldsymbol{\bar{g}}_i -\mu k,\\
& \ \ \ \ \ \ \ \ \ \ \ \ \ \boldsymbol{b}_i^v = sgn(\mathcal{H}^v(\boldsymbol{x}_i^v;\boldsymbol{W}^v)),\\
& \ \ \ \ \ \ \ \ \ \ \ \ \ \boldsymbol{b}_i^t = sgn(\mathcal{H}^t(\boldsymbol{x}_i^t;\boldsymbol{W}^t)).
\end{aligned}
\label{f2}
\end{equation}
It can be found the Formula (\ref{f2}) is evolved from the Formula (\ref{f1}), that for an instance, its corresponding $margin$-$dynamic$-$softmax$ loss is constructed with its two different modal hash codes, i.e., if the numerator is about the hash code of its image (text), then then the denominator is about the hash code of its text (image). Then the objective function is the combination of the intra-modal $margin$-$dynamic$-$softmax$ loss and inter-modal loss, which can be formulated as follows:
\begin{equation}
\mathcal{L} = \mathcal{L}_{S}^v + \mathcal{L}_{S}^t + \lambda\mathcal{L}_{I} 
\label{f3}
\end{equation}
where $\lambda$ is a hyper-parameter.

Moreover, similar to PHNet,  we add a quantization loss with the $sgn(\cdot)$ discarded to make sure that the parameters of modality-specific hashing networks can be updated by the back-propagation algorithm when minimizing the objective function, i.e., the Formula (\ref{f3}), then the final objective function can be formulated as follows:
\begin{equation}
\begin{aligned}
\min\limits_{\boldsymbol{W}^v,\boldsymbol{W}^t}\mathcal{L} &= \mathcal{L}_{S}^v + \mathcal{L}_{S}^t + \lambda\mathcal{L}_{I} + \gamma\mathcal{L}_{quantization}\\
&=\sum\limits_{i=1}^n -log\frac{exp(\eta\boldsymbol{\hat{u}}_i^v)}{exp(\eta\boldsymbol{\hat{u}}_i^v) + \sum\limits_{q\in\Omega_i}exp(\eta\boldsymbol{\hat{b}}_i^{vT}\boldsymbol{g}_q)} \\
&\ \ \ \ +\sum\limits_{i=1}^n -log\frac{exp(\eta\boldsymbol{\hat{u}}_i^t)}{exp(\eta\boldsymbol{\hat{u}}_i^t) + \sum\limits_{q\in\Omega_i}exp(\eta\boldsymbol{\hat{b}}_i^{tT}\boldsymbol{g}_q)} \\
&\ \ \ \ +\lambda\sum\limits_{i=1}^n -log\frac{exp(\eta\boldsymbol{\hat{u}}_i^v)}{exp(\eta\boldsymbol{\hat{u}}_i^t) + \sum\limits_{q\in\Omega_i}exp(\eta\boldsymbol{\hat{b}}_i^{tT}\boldsymbol{g}_q)} \\
&\ \ \ \ +\lambda\sum\limits_{i=1}^n -log\frac{exp(\eta\boldsymbol{\hat{u}}_i^t)}{exp(\eta\boldsymbol{\hat{u}}_i^v) + \sum\limits_{q\in\Omega_i}exp(\eta\boldsymbol{\hat{b}}_i^{vT}\boldsymbol{g}_q)} \\
& \ \ \ \ + \gamma(\sum\limits_{i=1}^n\left\|\boldsymbol{\hat{b}}_i^t - \boldsymbol{c}_i \right\|_F^2 + \sum\limits_{i=1}^n\left\|\boldsymbol{\hat{b}}_i^v - \boldsymbol{c}_i \right\|_F^2)\\
&\ \ \ s.t. \ \ \ \ \boldsymbol{\hat{u}}_i^v = \boldsymbol{\hat{b}}_i^{vT}\boldsymbol{\bar{g}}_i -\mu k, \\
&\ \ \ \ \ \ \ \ \ \ \ \ \boldsymbol{\hat{u}}_i^t = \boldsymbol{\hat{b}}_i^{tT}\boldsymbol{\bar{g}}_i -\mu k, \\
&\ \ \ \ \ \ \ \ \ \ \ \ \boldsymbol{\hat{b}}_i^v = \mathcal{H}^v(\boldsymbol{x}_i^v;\boldsymbol{W}^v),\\
&\ \ \ \ \ \ \ \ \ \ \ \ \boldsymbol{\hat{b}}_i^t = \mathcal{H}^t(\boldsymbol{x}_i^t;\boldsymbol{W}^t),\\
&\ \ \ \ \ \ \ \ \ \ \ \ \boldsymbol{c}_i = sgn(\boldsymbol{\hat{b}}_i^v + \boldsymbol{\hat{b}}_i^t).
\end{aligned}
\label{iloss}
\end{equation}
where $\gamma$ is a hyper-parameter.
The details of the algorithm are shown in Algorithm \ref{alg}.

\subsection{Out-of-Sample Extension}
\label{oe}
For any data point which is not in the training dataset, we can obtain its hash code by the learned modality-specific hashing networks. Specifically, given an unseen query data point of image modal $\boldsymbol{x}_i^v$, we can adopt forward propagation to generate the hash code $\boldsymbol{r}$:
\begin{equation}
\boldsymbol{r} = sgn(\boldsymbol{H}^v(\boldsymbol{x}_i^v;\boldsymbol{W}^v)
\end{equation}
Similarly, we can also use the learned hashing model to generate the hash code $\boldsymbol{r}$ for an unseen query data point of text modal $\boldsymbol{x}_i^t$:
\begin{equation}
\boldsymbol{r} = sgn(\boldsymbol{H}^t(\boldsymbol{x}_i^t;\boldsymbol{W}^t)
\label{lb}
\end{equation}

\begin{algorithm}[t]
	\caption{Learning algorithm for DCHML}
	\label{alg}
	\begin{algorithmic}[1]
		\REQUIRE
		Images $\boldsymbol{X}$, label matrix $\boldsymbol{L}$, category matrix $\boldsymbol{Y}$, and the length of  hash codes $k$.
		\ENSURE 
		Proxy hash codes $\boldsymbol{G}$, parameters of PHNet $\boldsymbol{\Theta}$, parameters of IHNet $\boldsymbol{W}^v$, and parameters of THNet $\boldsymbol{W}^t$.
		\STATE Initialize parameters:  $\boldsymbol{\Theta}$, $\boldsymbol{W}^v$, $\boldsymbol{W}^t$, $\alpha$, $\beta$, $\eta$, $\mu$ and $\gamma$. learning rate: $lr$, mini-batch size  $z$ (see Section \ref{bd}).
		\REPEAT
		\STATE Generate $\boldsymbol{\hat{G}} = \{\boldsymbol{\hat{g}}_i\}_{i=1}^c$ with $\boldsymbol{Y}$ as input.
		\STATE Update parameter $\boldsymbol{\Theta}$ by minimizing Formula (\ref{ploss})
		\UNTIL{Convergence}
		\STATE Generate $\boldsymbol{G}$ by PHNet with $\boldsymbol{Y}$ as input.
		\REPEAT
		\FOR{$j=1:\frac{n}{z}$}
		\STATE Randomly sample $z$ image from database as a mini-batch.
		\STATE Generate $\boldsymbol{\hat{b}}_i^v$ with image $\boldsymbol{x}_i^v$ as input by IHNet.
		\STATE Update parameters of IHNet $\boldsymbol{W}^v$ by minimizing Formula (\ref{iloss}) with $\boldsymbol{W}^t$ fixed by back propagation algorithm.
		\ENDFOR
		\FOR{$j=1:\frac{n}{z}$}
		\STATE Randomly sample $z$ text from database as a mini-batch.
		\STATE Generate $\boldsymbol{\hat{b}}_i^t$ with text $\boldsymbol{x}_i^t$ as input by THNet.
		\STATE Update parameters of THNet $\boldsymbol{W}^t$ by minimizing Formula (\ref{iloss}) with $\boldsymbol{W}^v$ fixed by back propagation algorithm.
		\ENDFOR
		\STATE Update $\boldsymbol{c}_i$ with $\boldsymbol{c}_i = sgn(\boldsymbol{\hat{b}}_i^v + \boldsymbol{\hat{b}}_i^t)$.
		\UNTIL{Convergence}
	\end{algorithmic}
\end{algorithm}

\begin{table*}[t]
	\begin{minipage}{\textwidth}
		\setlength{\abovecaptionskip}{0pt}
		\setlength{\belowcaptionskip}{0pt}
		\caption{MAP of Hamming Ranking for All the Methods with Different Number of Bits on the Three Datasets. The best accuracy is shown in boldface.}
		\label{map1}
		\resizebox{\linewidth}{!}{
			\begin{tabular}{|l|l||l|l|l|l||l|l|l|l||l|l|l|l|}
				\hline
				\multirow{2}{*}{Task} & \multirow{2}{*}{Method}                 & \multicolumn{4}{c|}{NUS-WIDE}  & \multicolumn{4}{c||}{MS COCO}                                & \multicolumn{4}{c|}{IAPR TC-12}                                                      \\ \cline{3-14} 
				&                         & 32bits & 64bits & 96bits & 128bits & 32bits & 64bits & 96bits & 128bits & 32bits & 64bits & 96bits & 128bits         \\ \hline \hline
				\multirow{8}{*}{$T \rightarrow I$}  
				& SCM                    & 0.520          & 0.488          & 0.459          & 0.447        & 0.748          & 0.710          & 0.703          & 0.685          & 0.343          & 0.328          & 0.322          & 0.318             \\  
				& DJSRH                  & 0.645          & 0.647          & 0.653          & 0.616          & 0.641          & 0.690          & 0.707          & 0.717          & 0.498          & 0.515          & 0.523          & 0.526            \\  
				& DCMH                 & 0.713          & 0.726          & 0.735           & 0.715          & 0.638          & 0.647          & 0.657       & 0.639          & 0.603          & 0.601          & 0.591          & 0. 579           \\ 
				& CMHH                   & 0.694          & 0.710          & 0.637          & 0.631          & 0.616          & 0.620          & 0.615          & 0.613          & 0.515          & 0.520          & 0.514          &  0.508           \\ 
				& SSAH                   & 0.719          & 0.729          & 0.722           & 0.707        & 0.679          & 0.697          & 0. 718         & 0.706          & \textbf{0.646 }         & 0.669          & 0. 662          & 0.653            \\ 
				& EGDH                    & 0.733          & 0.734          & 0.743          & 0.742         & 0.774          & 0.821          & 0.850          & 0.856          & 0.641          & 0.671          & 0.688          & 0.692           \\ 
				& SCAHN                  & 0.711          & 0.721          & 0.715          & 0.717            & 0.765          & 0.819          & 0.805          & 0.780          & 0.554          & 0.569          & 0.604          & 0.603          \\ 
				& DCHML                 & \textbf{0.763}        & \textbf{0.783} & \textbf{0.782} & \textbf{0.790}   & \textbf{0.840} & \textbf{0.865} & \textbf{0.870} & \textbf{0.872} & 0.608 & \textbf{0.675} & \textbf{0.693} & \textbf{0.704} \\ \hline \hline
				\multirow{8}{*}{$I \rightarrow T$}  
				& SCM                    & 0.505          & 0.468          & 0.437          & 0.418            & 0.699          & 0.670          & 0.657          & 0.643          & 0.412          & 0.370          & 0.353          & 0.339         \\ 
				& DJSRH                    & 0.615          & 0.676          & 0.692          & 0.698     & 0.629          & 0.655          & 0.670          & 0.673          & 0.499          & 0.516          & 0.519          & 0.527          \\ 
				& DCMH                   & 0.690          & 0.710           & 0. 734     & 0.748            & 0.649          & 0.656          & 0.673     & 0.660          & 0.619          & 0.624          & 0. 610          & 0.604        \\ 
				& CMHH                    & 0.676          & 0.658          & 0.642          & 0.625        & 0.589          & 0.522          & 0.487          & 0.459          & 0.549          & 0.527          & 0.522          & 0.515            \\ 
				& SSAH                    & 0.734          & 0.744          & 0.709          & 0.685            & 0.665          & 0.679          & 0.702       & 0.698          & 0.645          & 0.665          & 0.604          & 0.550          \\ 
				& EGDH                  & 0.764          & 0.766          & 0.782          & 0.762            & 0.728          & 0.763          & 0.790          & 0.806          & \textbf{0.651}          & 0.667          & 0.682          & 0.687           \\ 
				& SCAHN                    & 0.733          & 0.744          & 0.736          & 0.736              & 0.749          & 0.799          & 0.785          & 0.761          & 0.550          & 0.574          & 0.608          & 0.606         \\ 
				& DCHML                  & \textbf{0.773} & \textbf{0.787} & \textbf{0.794} & \textbf{0.786}   & \textbf{0.798} & \textbf{0.817} & \textbf{0.816} & \textbf{0.819} & 0.605 & \textbf{0.672} & \textbf{0.695} & \textbf{0.706}\\ \hline
		\end{tabular}}
	\end{minipage}
\end{table*}
\section{Experiments}
\subsection{Datasets} 
To evaluate our proposed DCHML, we conduct experiments on three widely used benchmark datasets, i.e., \textit{\textbf{NUS-WIDE}} dataset \cite{chua2009nus}, \textit{\textbf{MS COCO}} \cite{lin2014microsoft} and \textit{\textbf{IAPR TC-12}} \cite{escalante2010segmented}, which are described below.

The \textit{\textbf{NUS-WIDE}} dataset contains 269,648 image-text pairs crawled from Flickr. Each image-text pair is annotated with one or multiple labels from 81 concept labels. Only 195,834 image-text pairs that belong to the 21 most frequent labels are selected for our experiment. Moreover, each instance in the text modality is represented by a 1,000-dimensional BoW vector. 

The \textit{\textbf{MS COCO}} contains about 82,014 training and 40,204 validation image-text pairs. In total, 122,218 image-text pairs are selected for our experiments. Each instance in the text modality is represented by a 2,000-dimensional bag of words (BoW) vector. 

The \textit{\textbf{IAPR TC-12}} \cite{escalante2010segmented} consists of 20,000 image-text pairs which are annotated using 255 labels. After pruning the image-text pair that is without any text information, a subset of 19999 image-text pairs are select for our experiment. The text modality for each instance is represented as a 2000-dimensional BoW vector. 

For the NUS-WIDE and MS COCO datasets, we randomly selected 5,000 image-text pairs as test set, with the rest of the image-text pairs as retrieval set, and we randomly select 10,000 image-text pairs from the retrieval set as the training set.  Moreover, for the IAPR TC-12 dataset, we randomly sampled 2,000 image-text pairs as test set, with the rest as retrieval set. We randomly select 10,000 image-text pairs from retrieval set for training deep cross-modal baselines.


\subsection{Baselines and Implementation Details}
\label{bd}
We compare our  proposed DCHML with seven classical or state-of-the-art methods, including one shallow
architecture based cross-modal hashing method \textbf{SCM} \cite{zhang2014large}, and six deep network based cross-modal hashing methods, i.e., \textbf{DJSRH} \cite{su2019deep}, \textbf{CMHH} \cite{cao2018cross}, \textbf{DCMH} \cite{jiang2017deep}, \textbf{SSAH} \cite{li2018self}, \textbf{EGDH} \cite{shi2019equally} and \textbf{SCAHN} \cite{wang2020self}.

For the proposed method, the parameters of the first seven layers of IHNet  are initialized with the first seven layers of pre-trained AlexNet model on ImageNet \cite{russakovsky2015imagenet}. All the parameters of PHNet, THNet and the last layer of IHNet are initialized by Xavier initialization \cite{glorot2010understanding}. The image inputs are the $224 \times 224$ raw pixels, and the text inputs are the BoW vectors. The values of hyper-parameters $\alpha, \beta, \eta, \mu, \lambda$ and $\gamma$ in DCHML will discussed at subsection \ref{sp} in detail, respectively. We adopt SGD with a mini-batch size of 128 and a learning rate within $10^{-4}$ to $10^{-3}$ as our optimization algorithm.

\subsection{Evaluation Protocol} 
For hashing based cross-modal retrieval task, Hamming ranking and hash lookup are two widely used retrieval protocols to evaluate the performance of hashing methods \cite{li2018self,shi2019equally}. In our experiments, similar to \cite{tu2021partial}, we evaluate the retrieval quality based on three evaluation metrics: Mean Average Precision (\textbf{MAP}), Precision curves with respect to the number of top N returned results (\textbf{P@N}), Precision-Recall curves (\textbf{PR}). MAP, P@N are used to measure the accuracy of the Hamming ranking protocol. PR curve is used to evaluate the accuracy of the hash lookup protocol. Moreover, for the MAP, P@n and PR curve, an image data point $\boldsymbol{x}_i$ and a text data point $\boldsymbol{y}_j$ will be defined as a similar pair if $\boldsymbol{x}_i$ and $\boldsymbol{y}_j$ share at least one common label. Otherwise, they will be defined as a dissimilar pair.  

Specifically, given a query datapoint, the Average Precision (AP) score of top n retrieved datapoints is defined as:
\begin{equation}
AP = \sum_{i=1}^n \frac{I(i)}{N}\sum_{j=1}^i \frac{I(j)}{i}
\end{equation}
where $I(i)$ is an indicator function, if the $i^{th}$ retrieved datapoint is the relevant datapoint, i.e., sharing at least one common category with the query datapoint, $I(i)$ is 1; otherwise $I(i)$ is 0. $N$ represents the number of relevant datapoints in the returned top $n$ datapoints. Then, the Mean Average Precision (MAP)  is then defined as the average of APs for all queries. Moreover, for all the three datasets, we set n as 5000.

\subsection{Experimental results}
\subsubsection{Hamming Ranking Protocol}
Table \ref{map1} shows the MAP  results of all baselines and DCHML on MS COCO, NUS-WIDE and IAPR TC-12 datasets, respectively. The P@N curves on 128-bits over the three dataset are shown in Figure \ref{fig_p}. $``$I2T$"$ denotes retrieving texts with image queries, and $``$T2I$"$ denotes retrieving images with text queries. From the table and figures, it can be observed that our method outperforms all state-of-the-art baselines on all the evaluate metrics.  For instance, compared with the baseline SCAHN, the MAP results of DCHML for $``$I2T$"$ have an average increase of 4.8\%, 3.9\% and 8.5\% on datasets NUS-WIDE, MS COCO and IAPR TC-12, respectively. Moreover, compared with the baseline EGDH, the MAP results of DCHML for  $``$T2I$"$  have an average increase of 4.2\%,  and 3.7\% on datasets MS COCO and NUS-WIDE.  Furthermore, as shown in Figure \ref{fig_p} (a) - (f), the P@N curves of our method are the highest.
These results indicate that the hash codes generated by our proposed method DCHML can preserve not only more semantic similarity information, but also more semantic label information than state-of-the-art baselines.

\begin{figure*}[!t]
	\centering
	
	\subfigure[]{
		\begin{minipage}[]{0.33\textwidth}
			\centering
			\includegraphics[width=\linewidth]{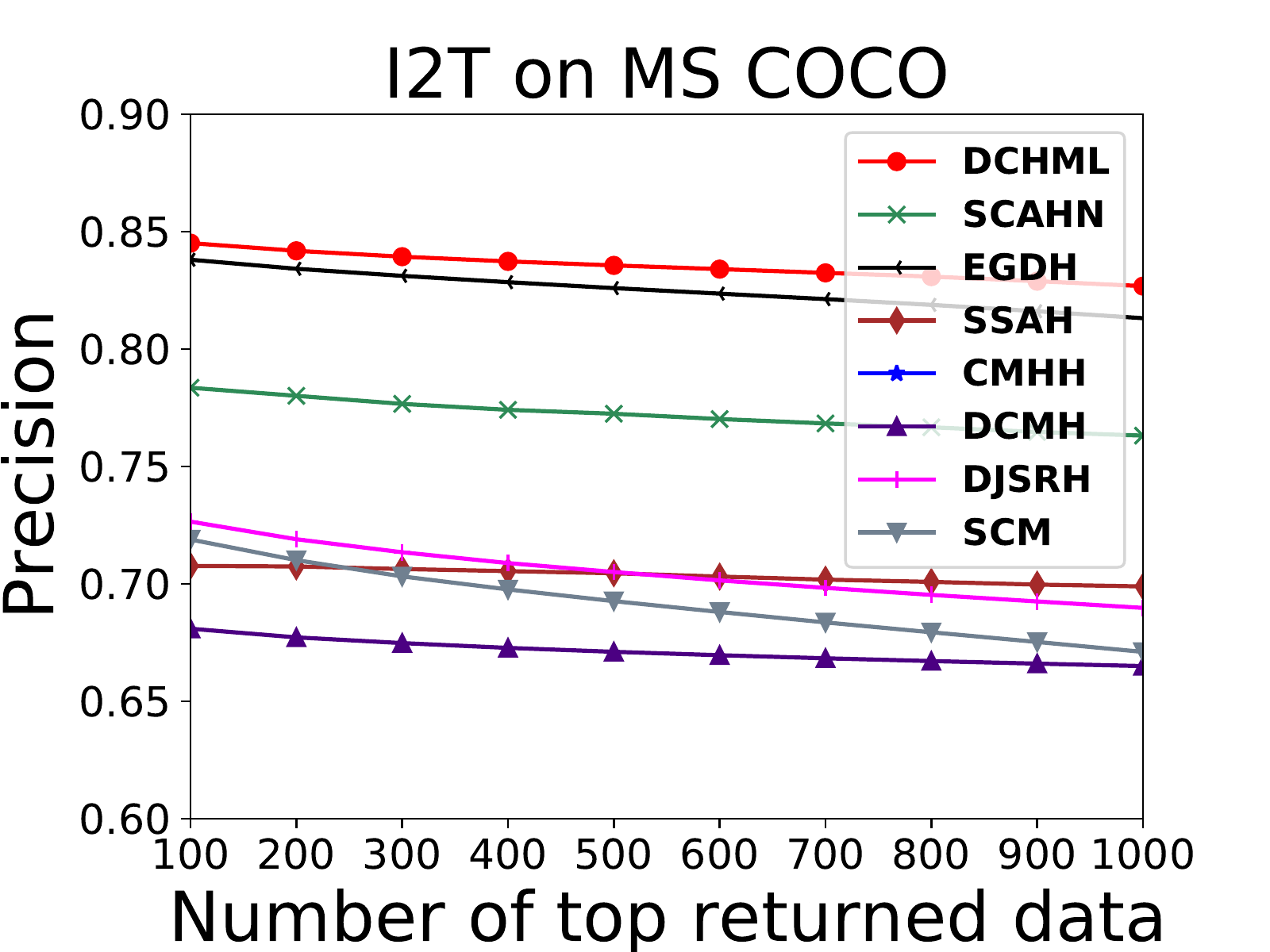}
		\end{minipage}%
	}%
	\subfigure[]{ 
		\begin{minipage}[]{0.33\textwidth}
			\centering
			\includegraphics[width=\linewidth]{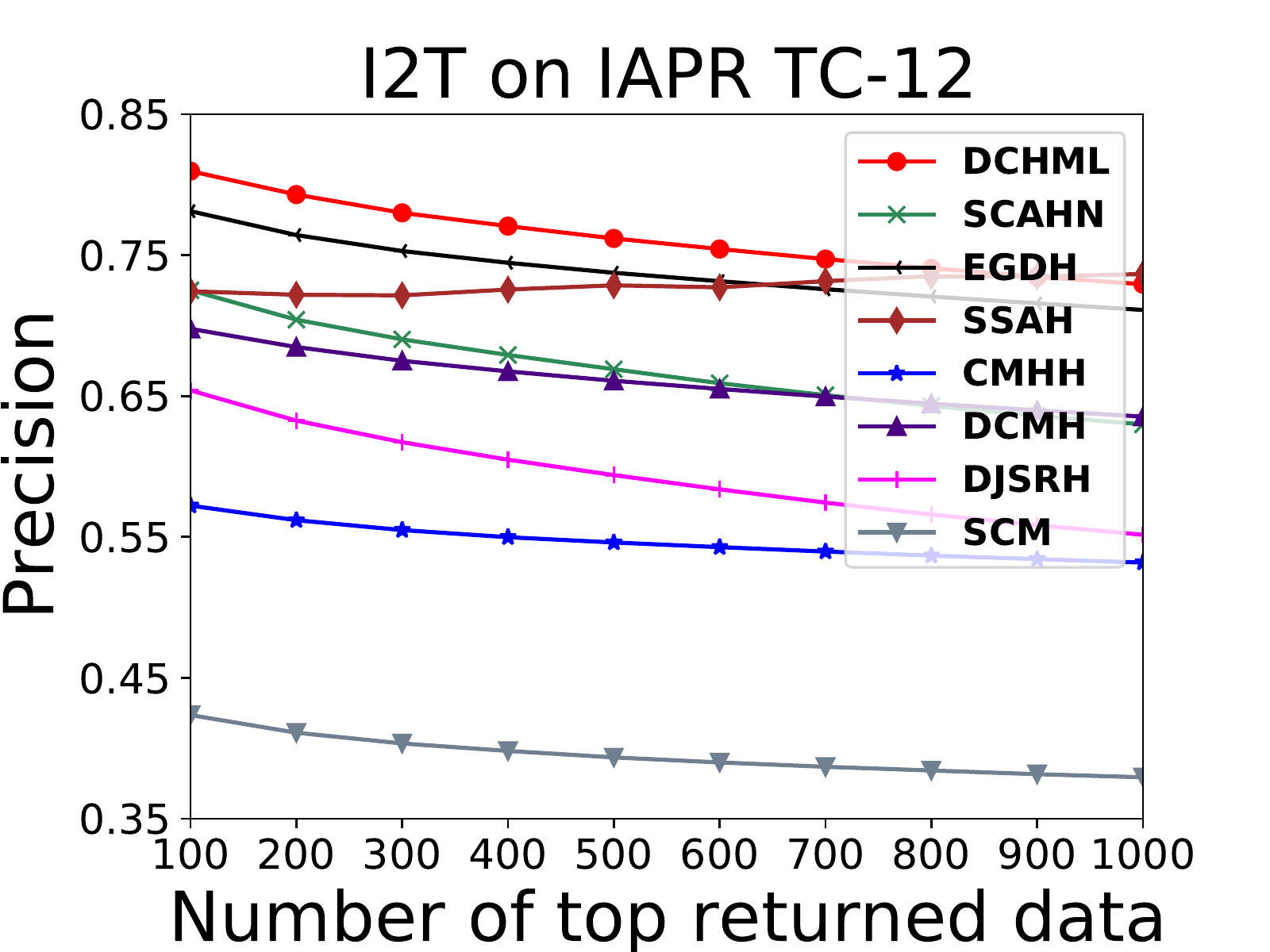}
		\end{minipage}%
	}%
	\subfigure[]{
		\begin{minipage}[]{0.33\textwidth}
			\centering
			\includegraphics[width=\linewidth]{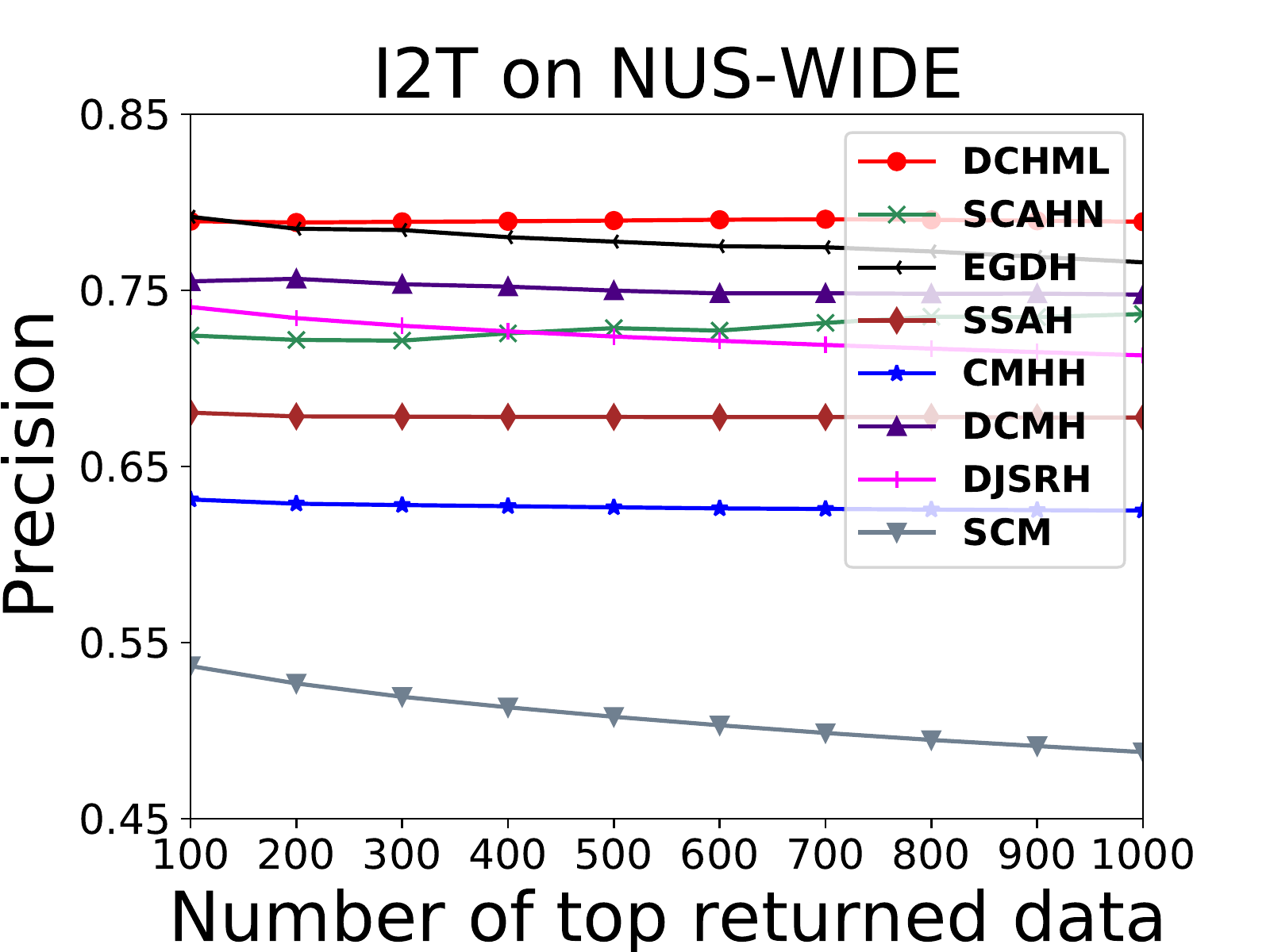}
		\end{minipage}%
	}%
	\quad
	\subfigure[]{
		\begin{minipage}[]{0.33\textwidth}
			\centering
			\includegraphics[width=\linewidth]{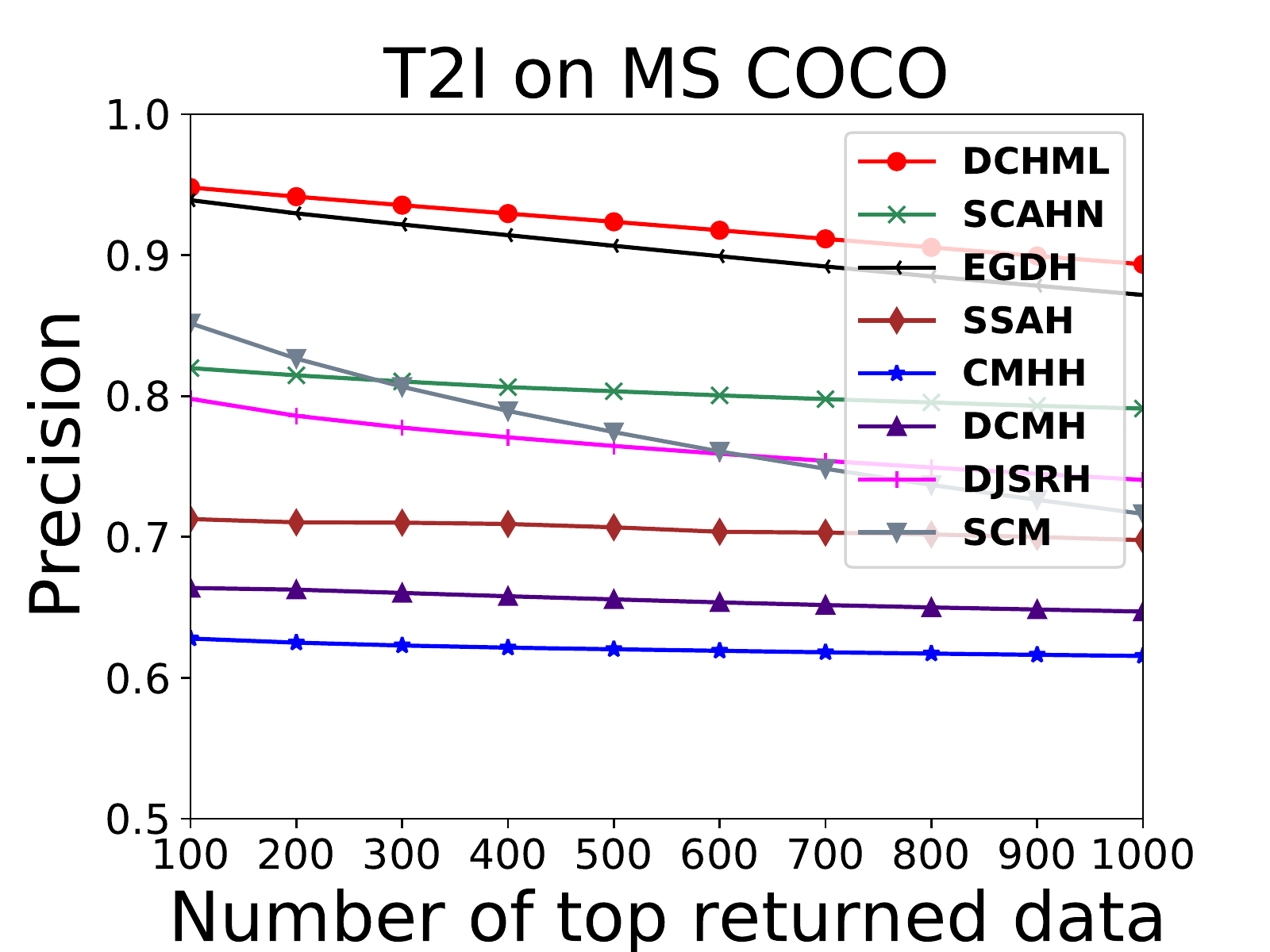}
		\end{minipage}%
	}%
	\subfigure[]{ 
		\begin{minipage}[]{0.33\textwidth}
			\centering
			\includegraphics[width=\linewidth]{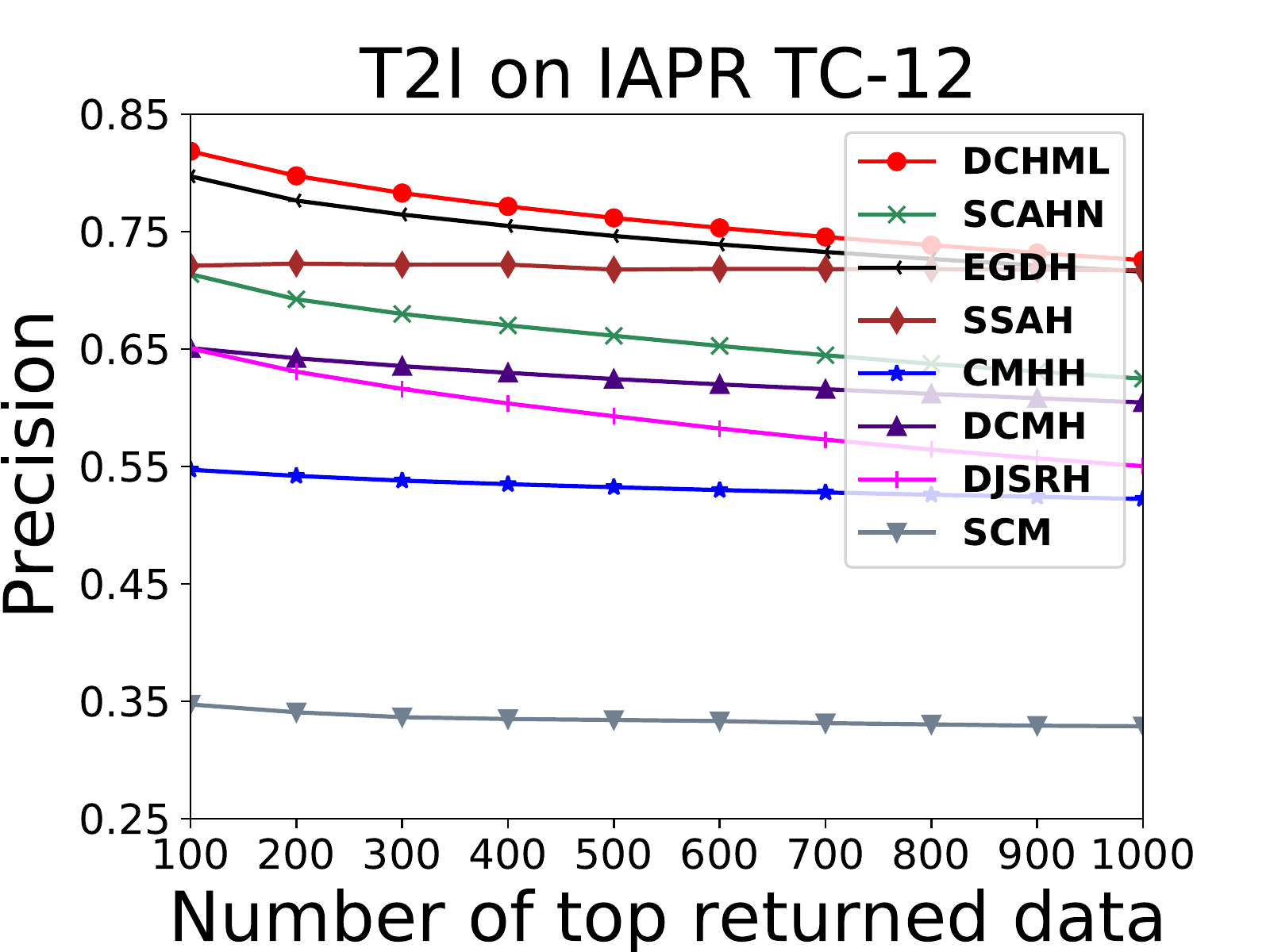}
		\end{minipage}%
	}%
	\subfigure[]{
		\begin{minipage}[]{0.33\textwidth}
			\centering
			\includegraphics[width=\linewidth]{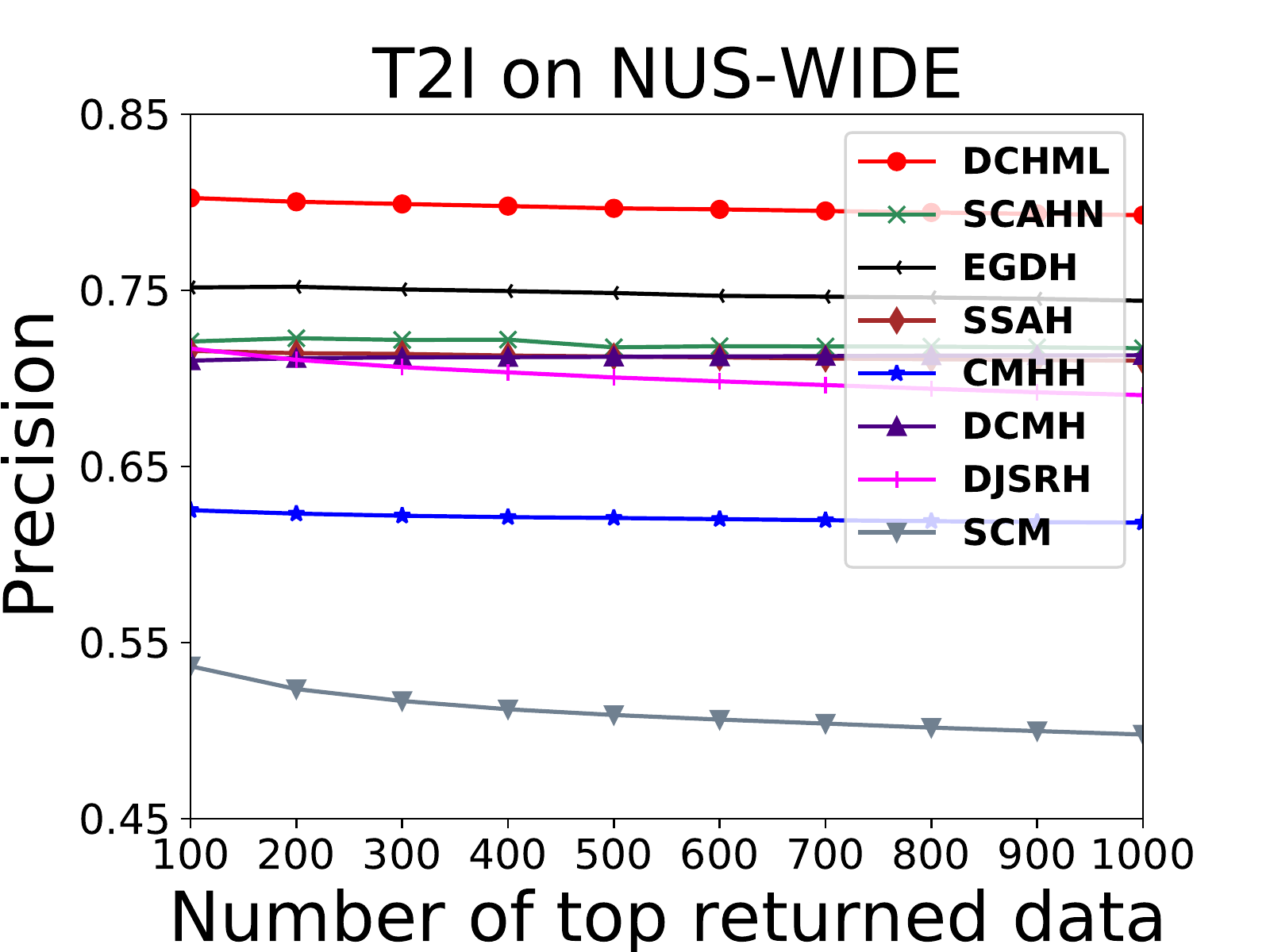}
		\end{minipage}%
	}%
	\caption{P@N curves of all methods over the three datasets}
	\label{fig_p}
\end{figure*}

\begin{figure*}[!t]
\centering
\subfigure[]{
	\begin{minipage}[]{0.33\textwidth}
		\centering
		\includegraphics[width=\linewidth]{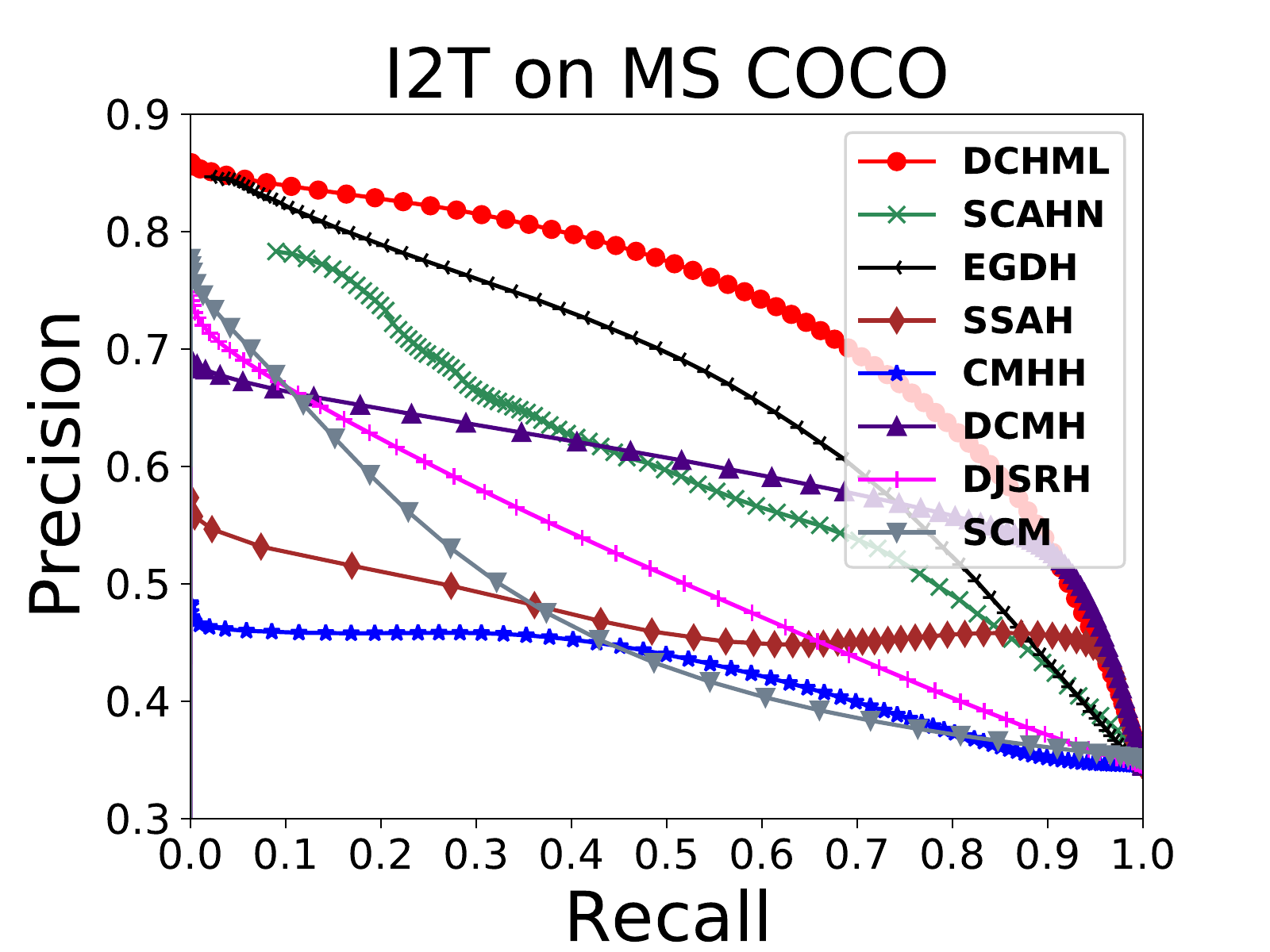}
	\end{minipage}%
}%
\subfigure[]{ 
	\begin{minipage}[]{0.33\textwidth}
		\centering
		\includegraphics[width=\linewidth]{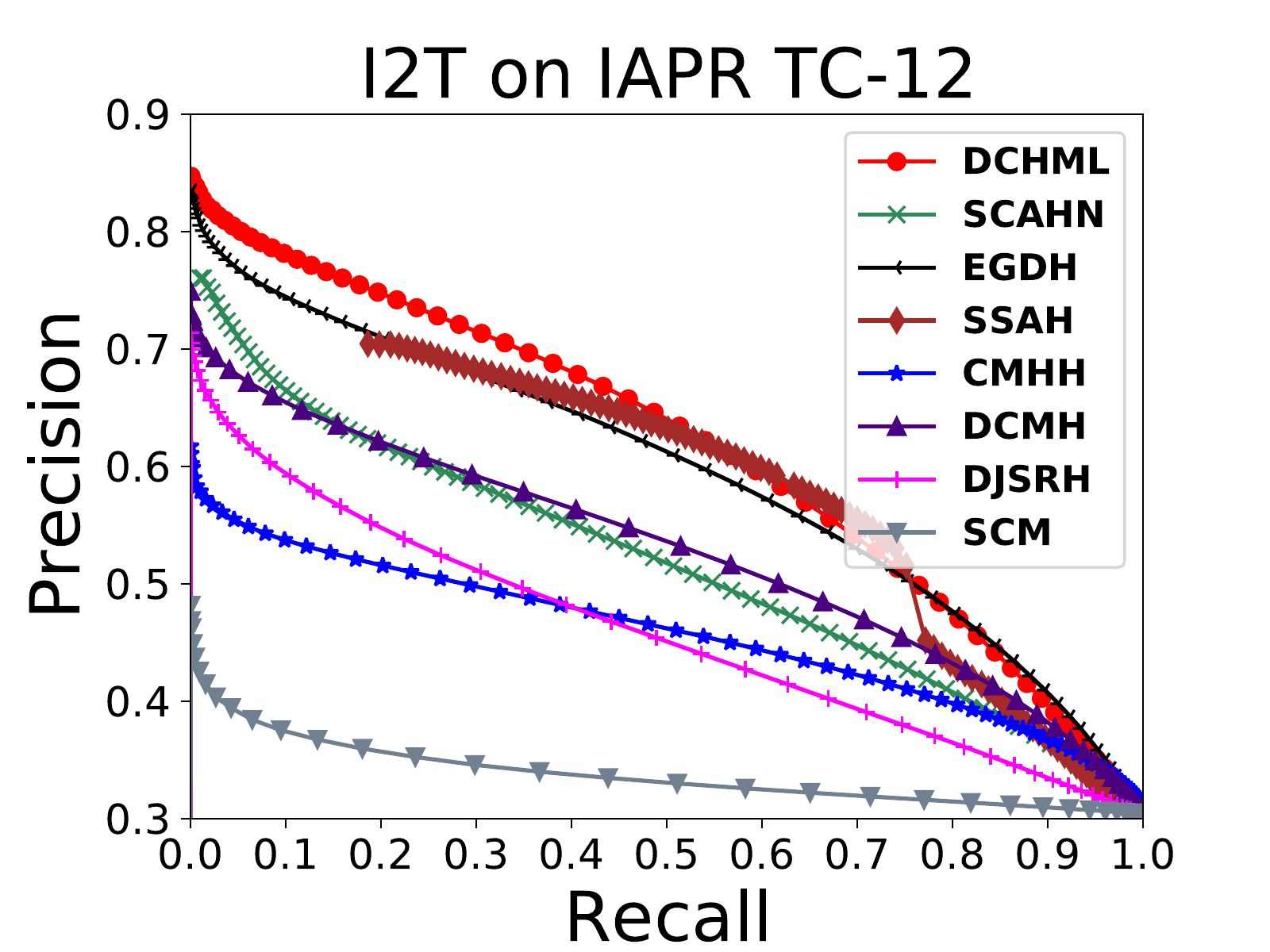}
	\end{minipage}%
}%
\subfigure[]{
	\begin{minipage}[]{0.33\textwidth}
		\centering
		\includegraphics[width=\linewidth]{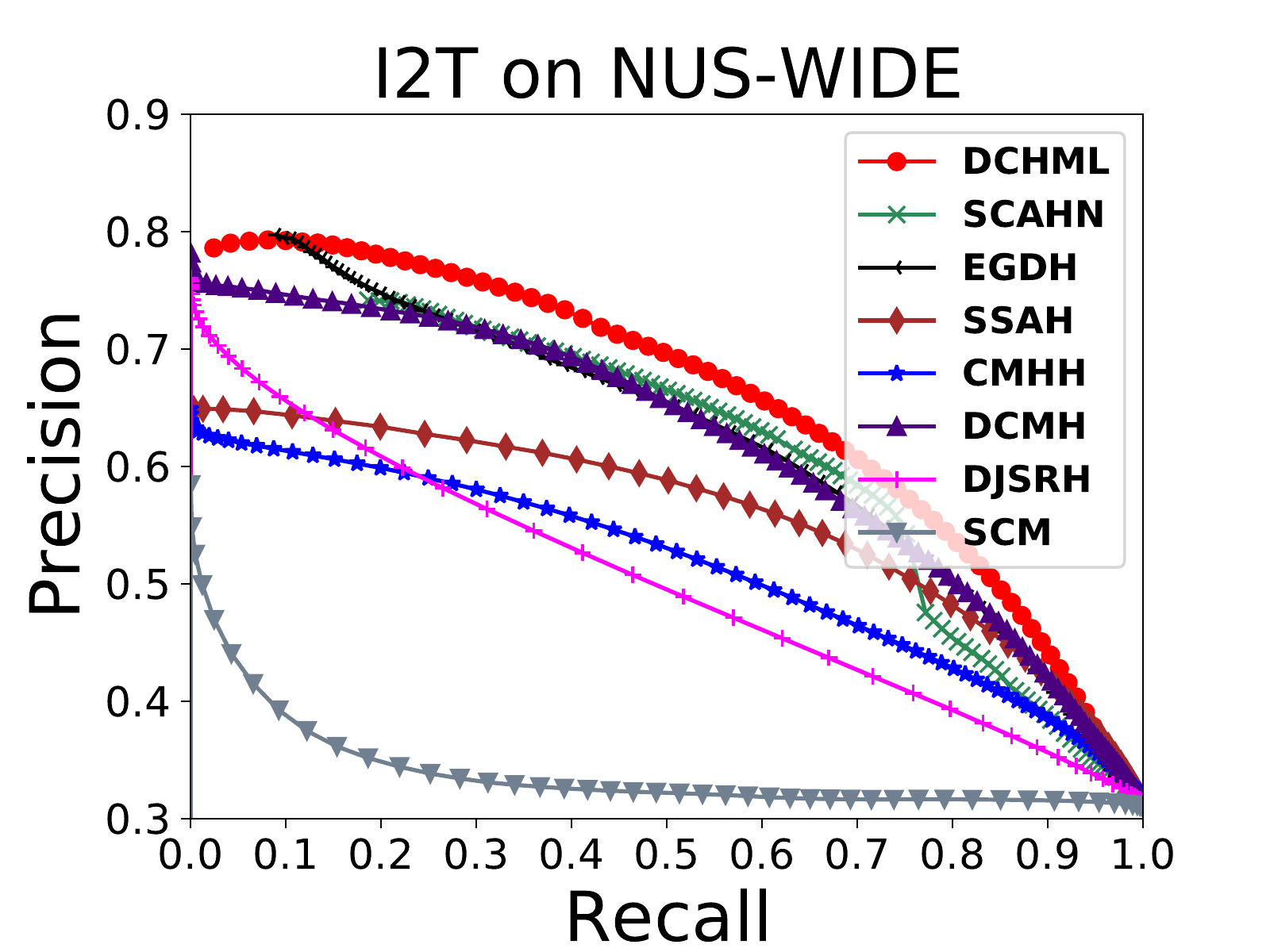}
	\end{minipage}%
}%
\quad
\subfigure[]{
	\begin{minipage}[]{0.33\textwidth}
		\centering
		\includegraphics[width=\linewidth]{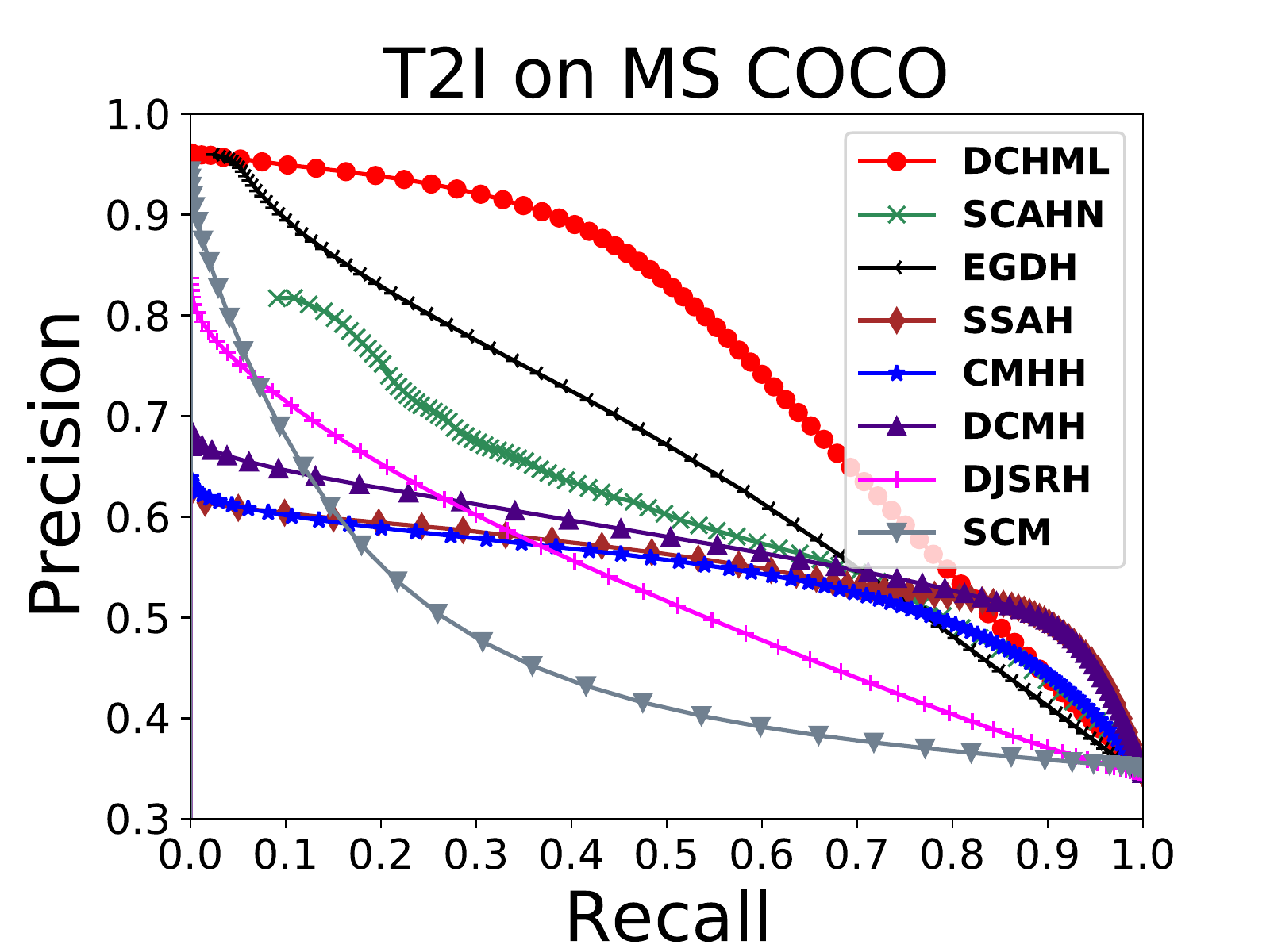}
	\end{minipage}%
}%
\subfigure[]{ 
	\begin{minipage}[]{0.33\textwidth}
		\centering
		\includegraphics[width=\linewidth]{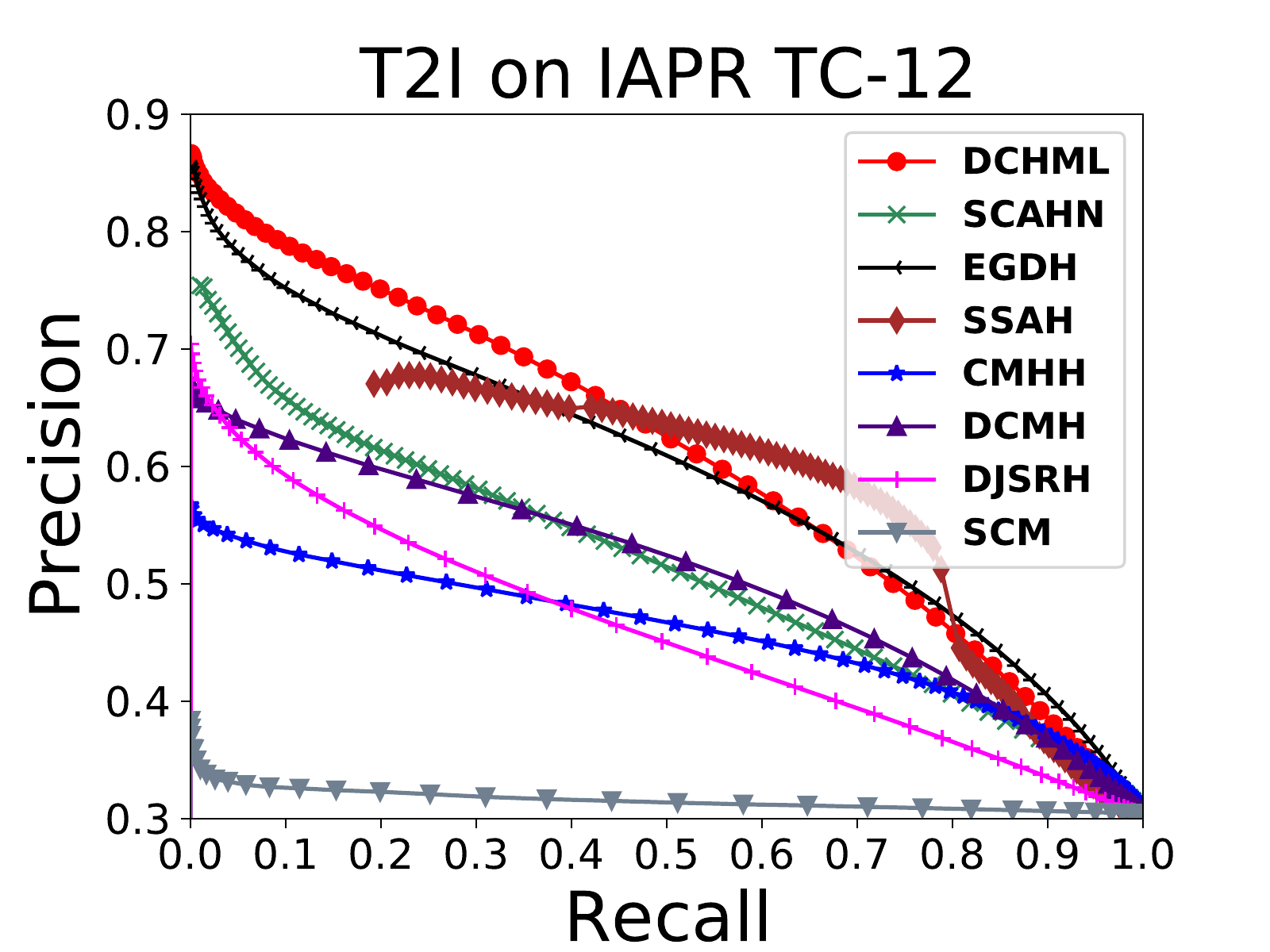}
	\end{minipage}%
}%
\subfigure[]{
	\begin{minipage}[]{0.33\textwidth}
		\centering
		\includegraphics[width=\linewidth]{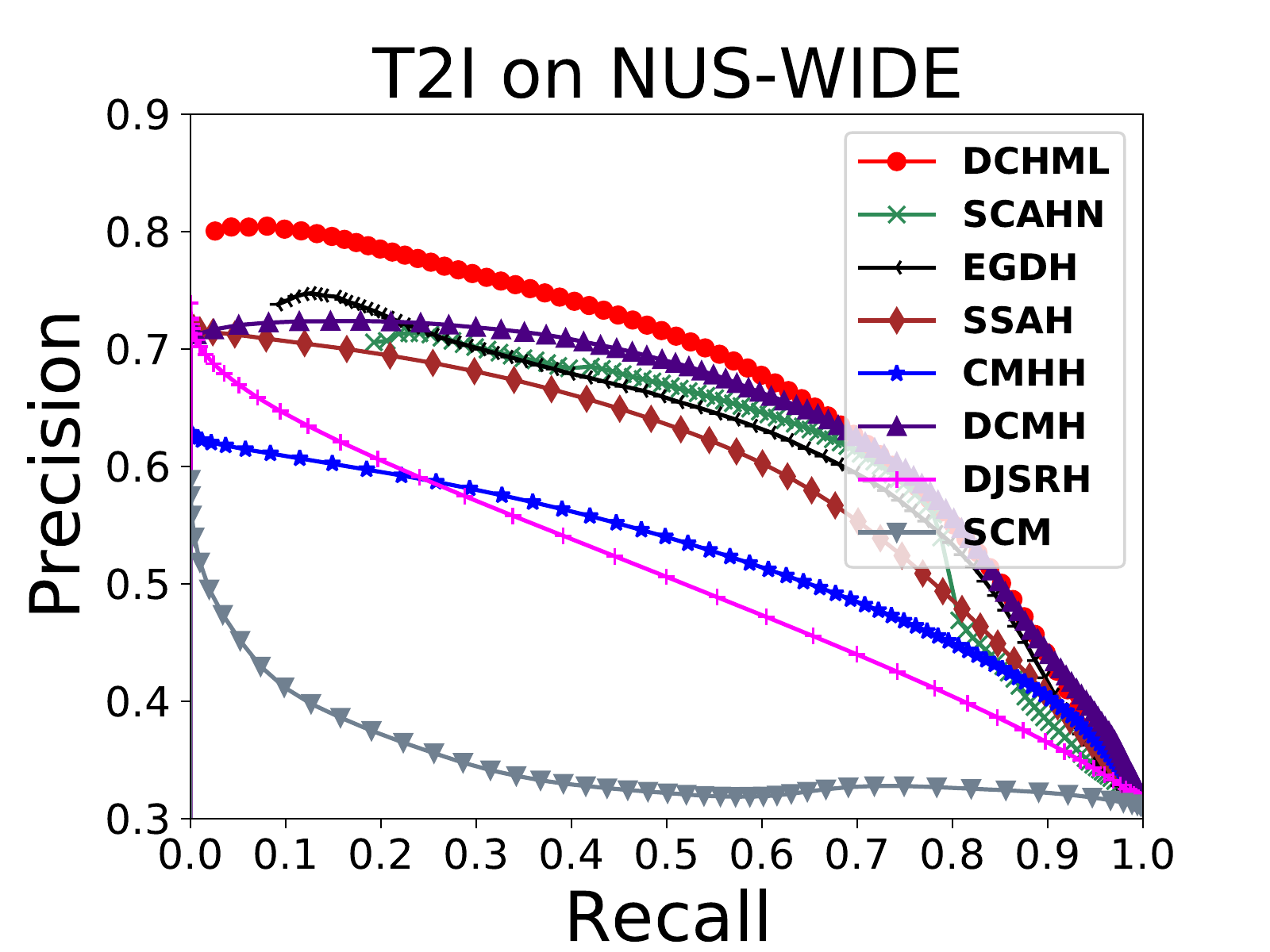}
	\end{minipage}%
}%
\caption{Precision-recall curves of all the methods over the three datasets}
\label{fig_pr}
\end{figure*}

\begin{figure*}[]
	\centering
	\subfigure[Loss function value]{
		\begin{minipage}[t]{0.33\textwidth}
			\centering
			\includegraphics[width=1\linewidth]{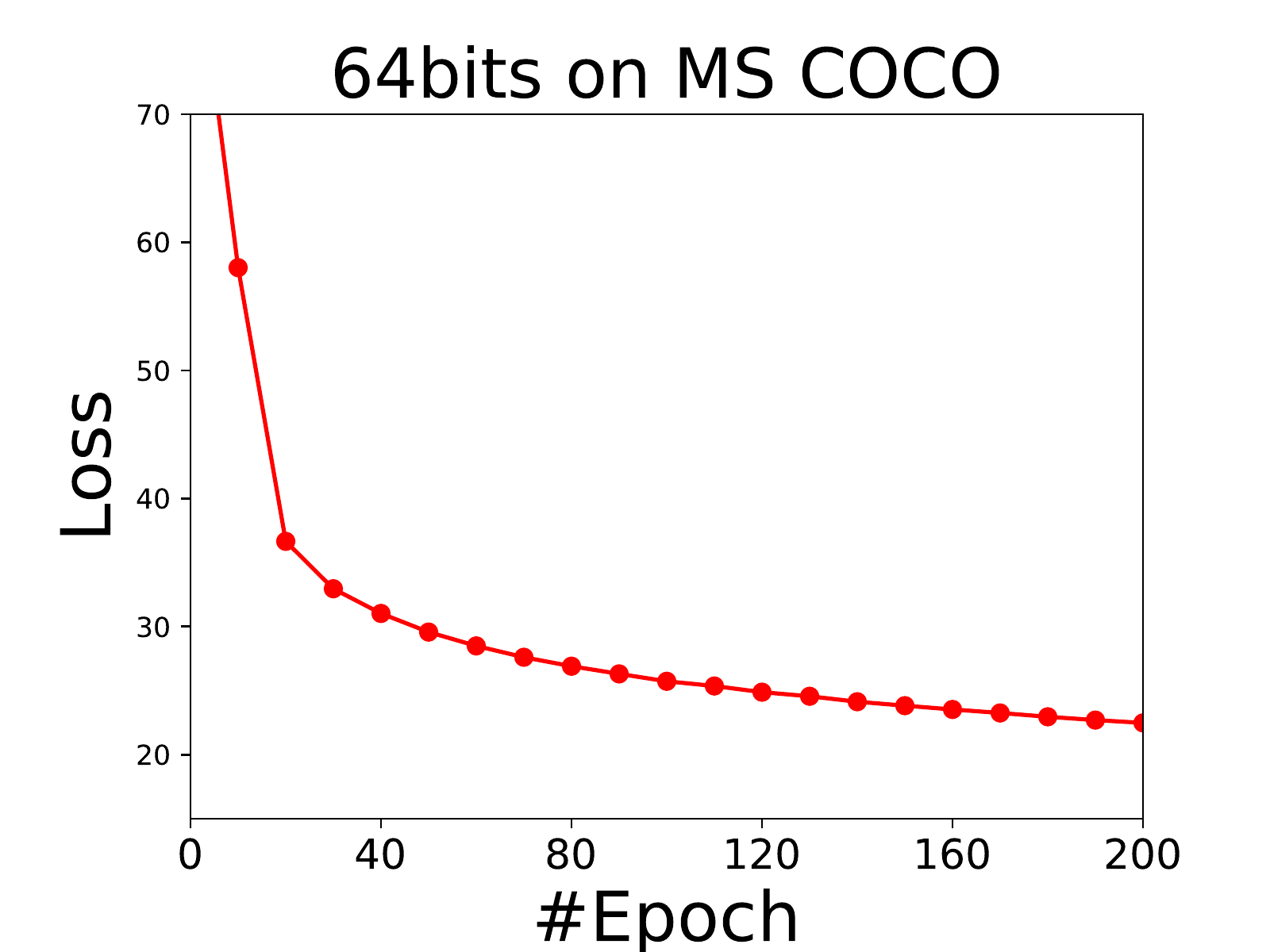}
			\label{loss64}
		\end{minipage}%
	}%
	\subfigure[MAP]{
		\begin{minipage}[t]{0.33\textwidth}
			\centering
			\includegraphics[width=1\linewidth]{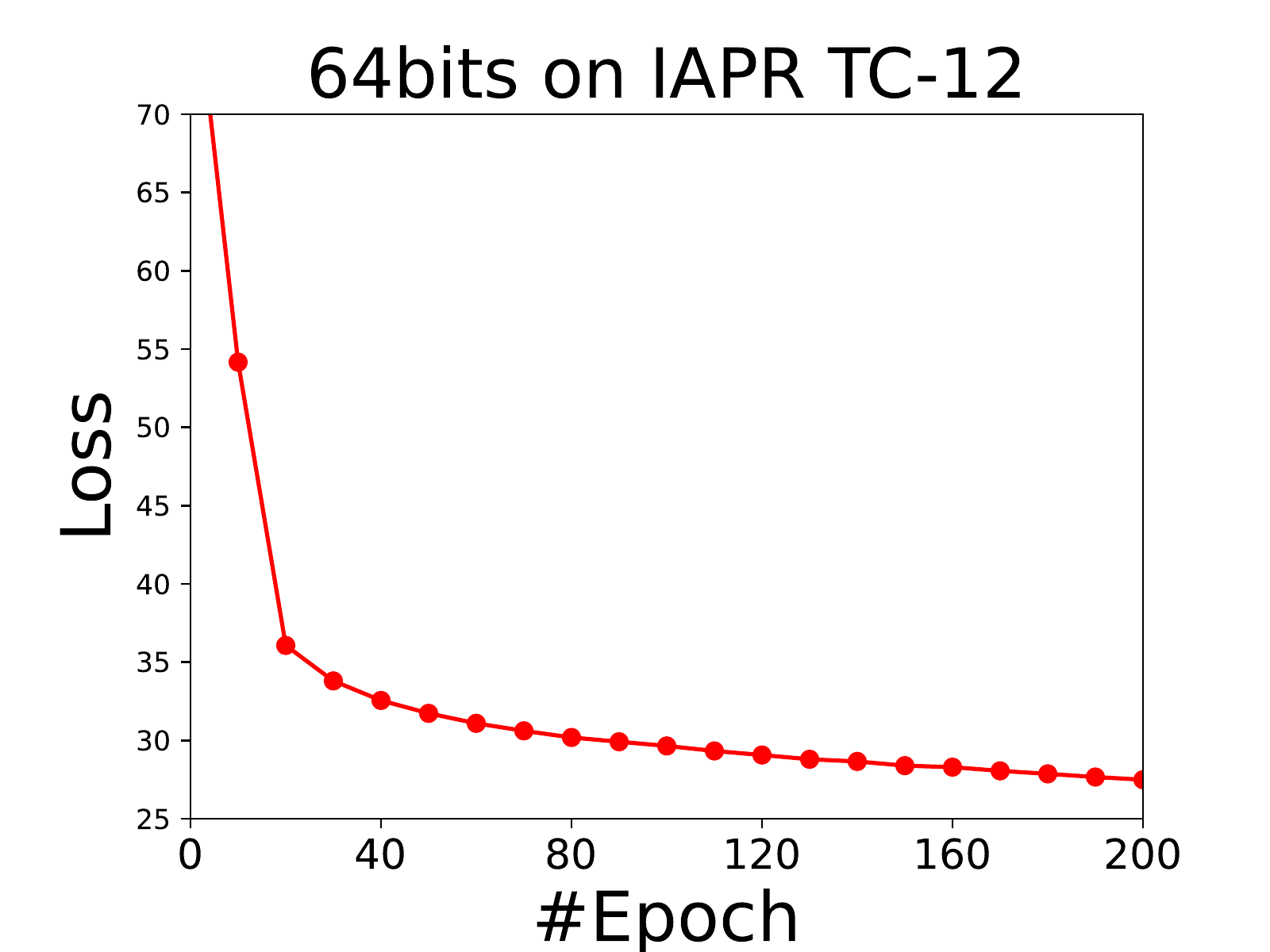}
			\label{map64}
		\end{minipage}%
	}%
	\subfigure[Loss function value]{
	\begin{minipage}[t]{0.33\textwidth}
		\centering
		\includegraphics[width=1\linewidth]{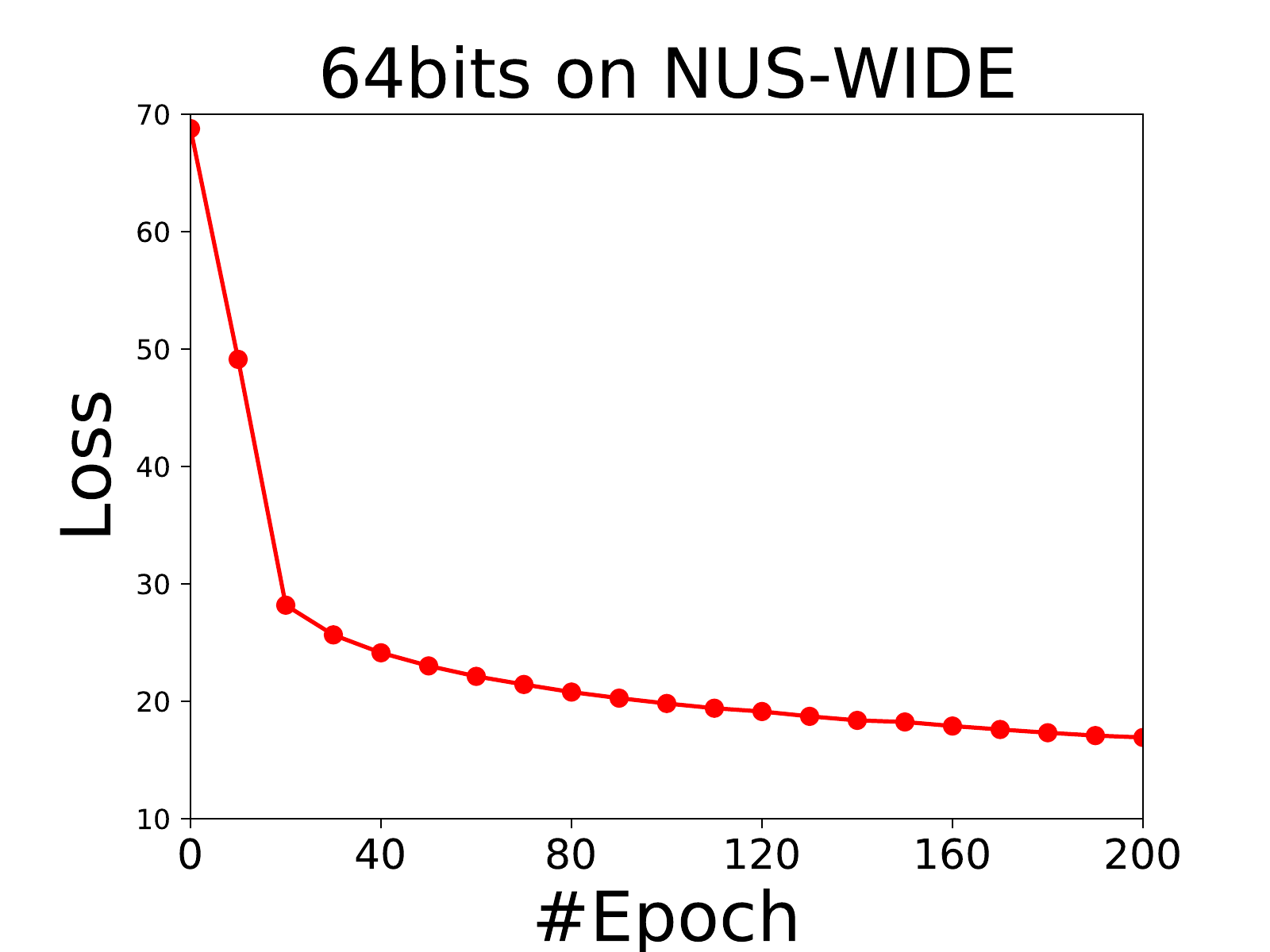}
		\label{loss64_iaprtc}
	\end{minipage}%
}%
\quad
\subfigure[MAP]{
	\begin{minipage}[t]{0.33\textwidth}
		\centering
		\includegraphics[width=1\linewidth]{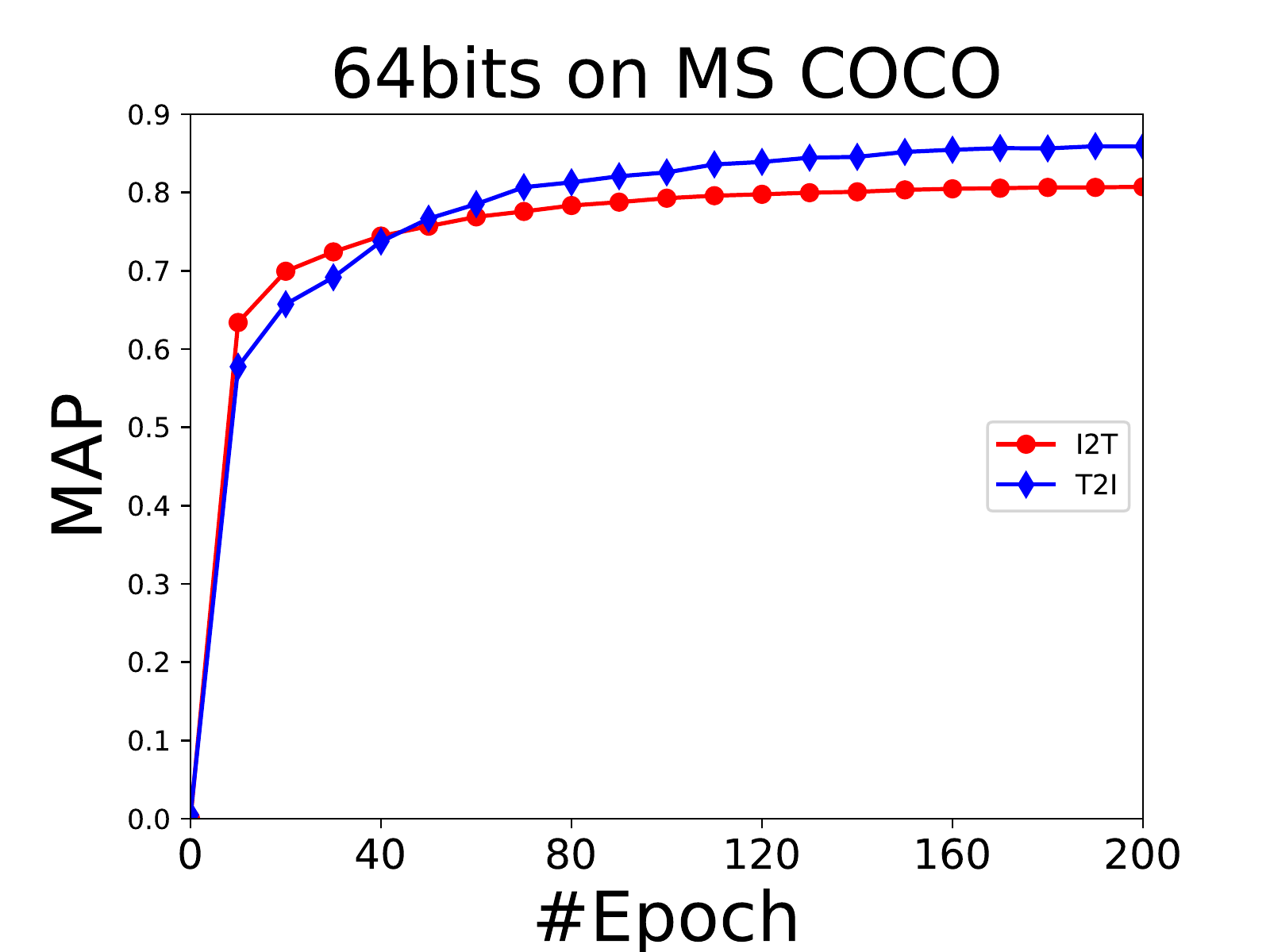}
		\label{map64_iaprtc}
	\end{minipage}%
}%
	\subfigure[Loss function value]{
	\begin{minipage}[t]{0.33\textwidth}
		\centering
		\includegraphics[width=1\linewidth]{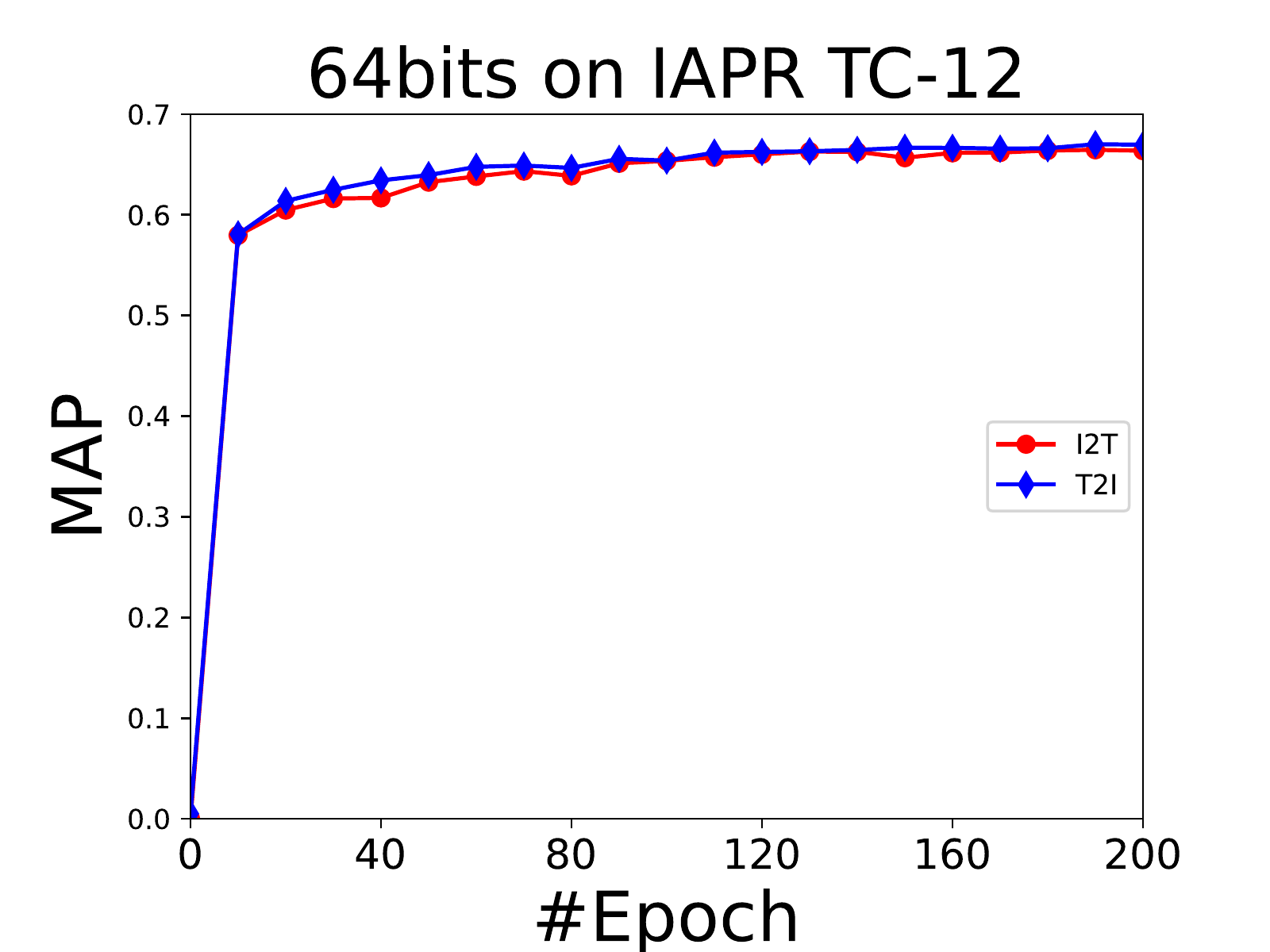}
		\label{loss64_nus}
	\end{minipage}%
}%
\subfigure[MAP]{
	\begin{minipage}[t]{0.33\textwidth}
		\centering
		\includegraphics[width=1\linewidth]{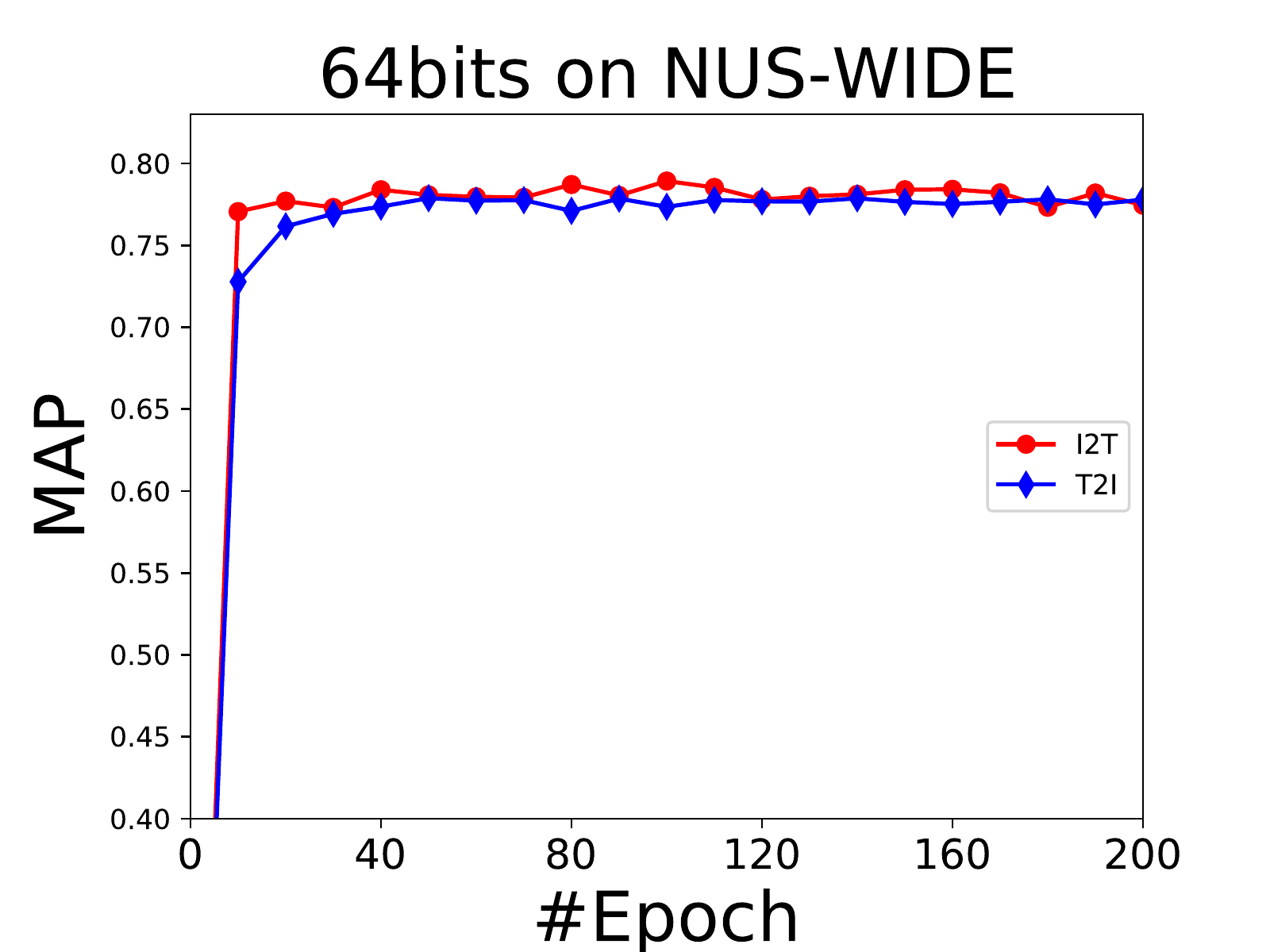}
		\label{map64_nus}
	\end{minipage}%
}%
	\caption{Loss function value and MAP of DCHML over the three datasets.}
	\label{loss_map}
\end{figure*}

\begin{table*}[]
	\begin{minipage}{\textwidth}
		\setlength{\abovecaptionskip}{0pt}
		\setlength{\belowcaptionskip}{0pt}
		\caption{MAP of Hamming Ranking for Different Number of Bits on the Three Datasets.}
		\label{map_vall}
		\resizebox{\linewidth}{!}{
			\begin{tabular}{|l|l||l|l|l|l||l|l|l|l||l|l|l|l|}
				\hline
				\multirow{2}{*}{Task} & \multirow{2}{*}{Method}                 & \multicolumn{4}{c|}{NUS-WIDE}  & \multicolumn{4}{c||}{MS COCO}                                & \multicolumn{4}{c||}{IAPR TC-12}                                                      \\ \cline{3-14} 
				&                         & 32bits & 64bits & 96bits & 128bits & 32bits & 64bits & 96bits & 128bits & 32bits & 64bits & 96bits & 128bits         \\ \hline \hline
				\multirow{4}{*}{$T \rightarrow I$}  
				& DCHML$_Q$                    & 0.755          & 0.775          & 0.779           & 0.782        & 0.832          & 0.859          & 0. 867         & 0.868          & 0.594            & 0.663          & 0.689     &0.701       \\ 
				& DCHML$_P$                   & 0.624          & 0.639            & 0.631          & 0.614          & 0.683         & 0.707 &0.653     & 0.663  &0.495  &0.528    & 0.562   &0.587\\  
				& DCHML                 & \textbf{0.763}        & \textbf{0.783} & \textbf{0.782} & \textbf{0.790}   & \textbf{0.840} & \textbf{0.865} & \textbf{0.870} & \textbf{0.872} & \textbf{0.608} & \textbf{0.675} & \textbf{0.693} & \textbf{0.704} \\ \hline \hline
				\multirow{4}{*}{$I \rightarrow T$}  
				& DCHML$_Q$                    & 0.767          & 0.772          & 0.786          & 0.779            & 0.795          & 0.831          & 0.815       & 0.817          & 0.594            & 0.664          & 0.688    &0.702      \\ 
				& DCHML$_P$                      & 0.612          & 0.636          & 0.623          & 0.616          & 0.695        & 0.716   &0.537    & 0.559 &0.476   &0.542     &0.570   &0.598    \\
				& DCHML                  & \textbf{0.773} & \textbf{0.787} & \textbf{0.794} & \textbf{0.786}   & \textbf{0.798} & \textbf{0.817} & \textbf{0.816} & \textbf{0.819} & \textbf{0.605} & \textbf{0.672} & \textbf{0.695} & \textbf{0.706}\\ \hline
		\end{tabular}}
	\end{minipage}
\end{table*}

\subsubsection{Hash Lookup Protocol}
When considering the lookup protocol, we compute the PR curve for the returned points given any Hamming radius. The PR curve can be obtained by varying the Hamming radius from 0 to $k$ with a step-size of 1. The PR curves on 128-bits for $``$I2T$"$ and $``$T2I$"$ tasks over the three datasets are shown in Figure \ref{fig_pr}, respectively. It can be found that DCHML can outperform state-of-the-art baselines for $``$I2T$"$ and $``$T2I$"$ tasks over all the three datasets. For example, in Figure \ref{fig_pr} (a) -(f), the PR curves of DCHML are higher than all the ones of baselines on the whole. These results demonstrate that the proposed DCHML generates hash codes for similar datapoints in a smaller Hamming radius.

\subsection{Discussion}

\subsubsection{Convergence Analysis} The convergence of loss function value and MAP of the proposed DCHML with hash code length being 64-bits are shown in Figure \ref{loss_map}. From the figure, it can be observed: (1) As shown in Figure \ref{loss_map} (a), (b), (c), the loss function value can convergence after 120 epochs on all the three datasets. It indicates our proposed \textit{margin-dynamic-softmax loss} function is convergent. (2) Combined with the results in Figure \ref{loss_map} (a) and (d), Figure \ref{loss_map} (b) and (e), and Figure \ref{loss_map} (c) and (f), it can be found when conducting experiments on all the three dataset,  the values of MAP increase with the values of loss function decreasing, and finally reach a steady-state. It indicates that our
proposed \textit{margin-dynamic-softmax loss} function can be used as a loss function to optimize  hashing networks to generate hash  codes with semantic similarity and semantic label information preserved well.

\begin{figure*} 
	\subfigure[]{
		\begin{minipage}[t]{0.33\textwidth}
			\centering
			\includegraphics[width=\linewidth]{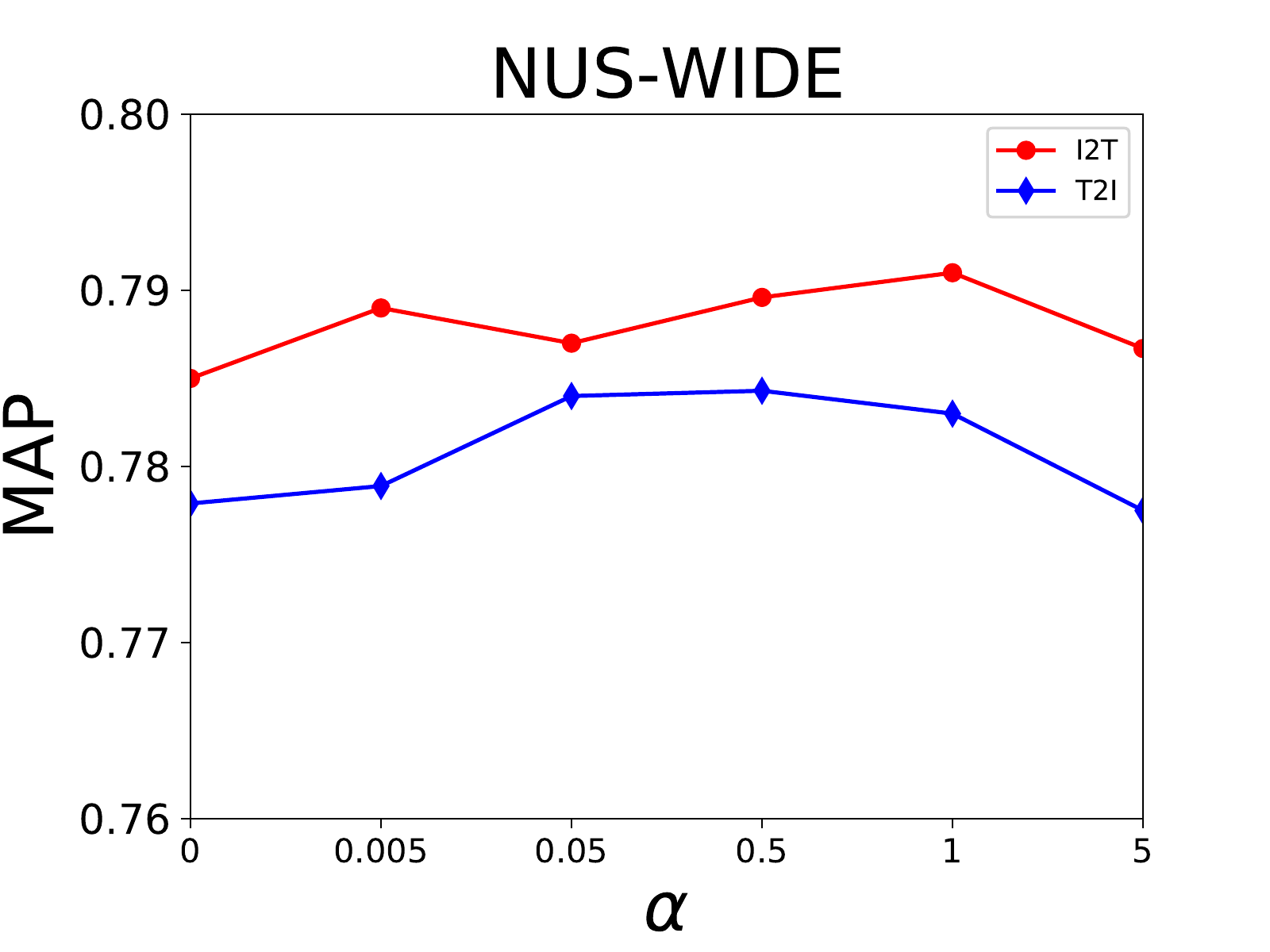}
		\end{minipage}%
	}%
	\subfigure[]{
		\begin{minipage}[t]{0.33\textwidth}
			\centering
			\includegraphics[width=\linewidth]{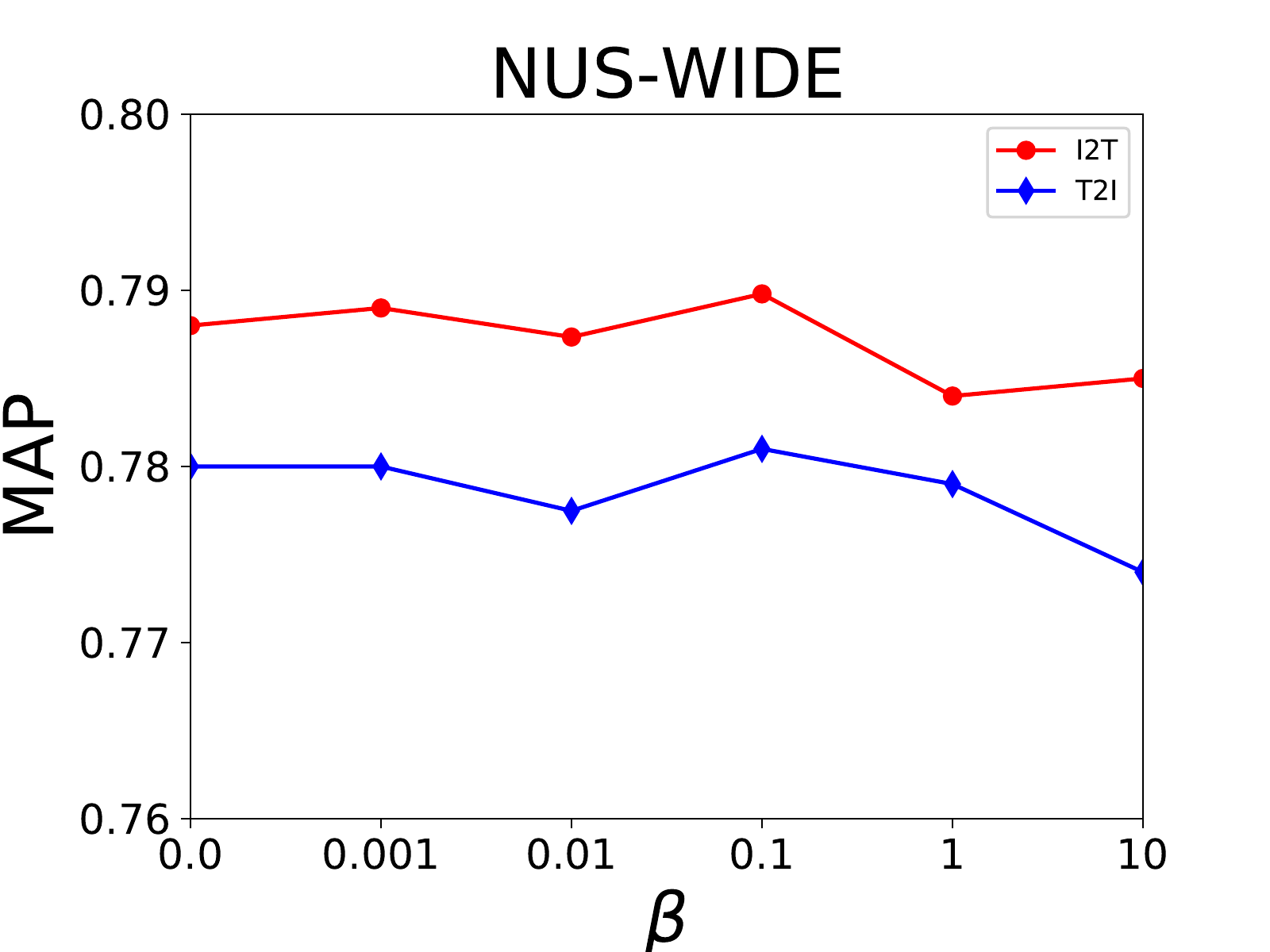}
		\end{minipage}%
	}%
	\subfigure[]{
		\begin{minipage}[t]{0.33\textwidth}
			\centering
			\includegraphics[width=\linewidth]{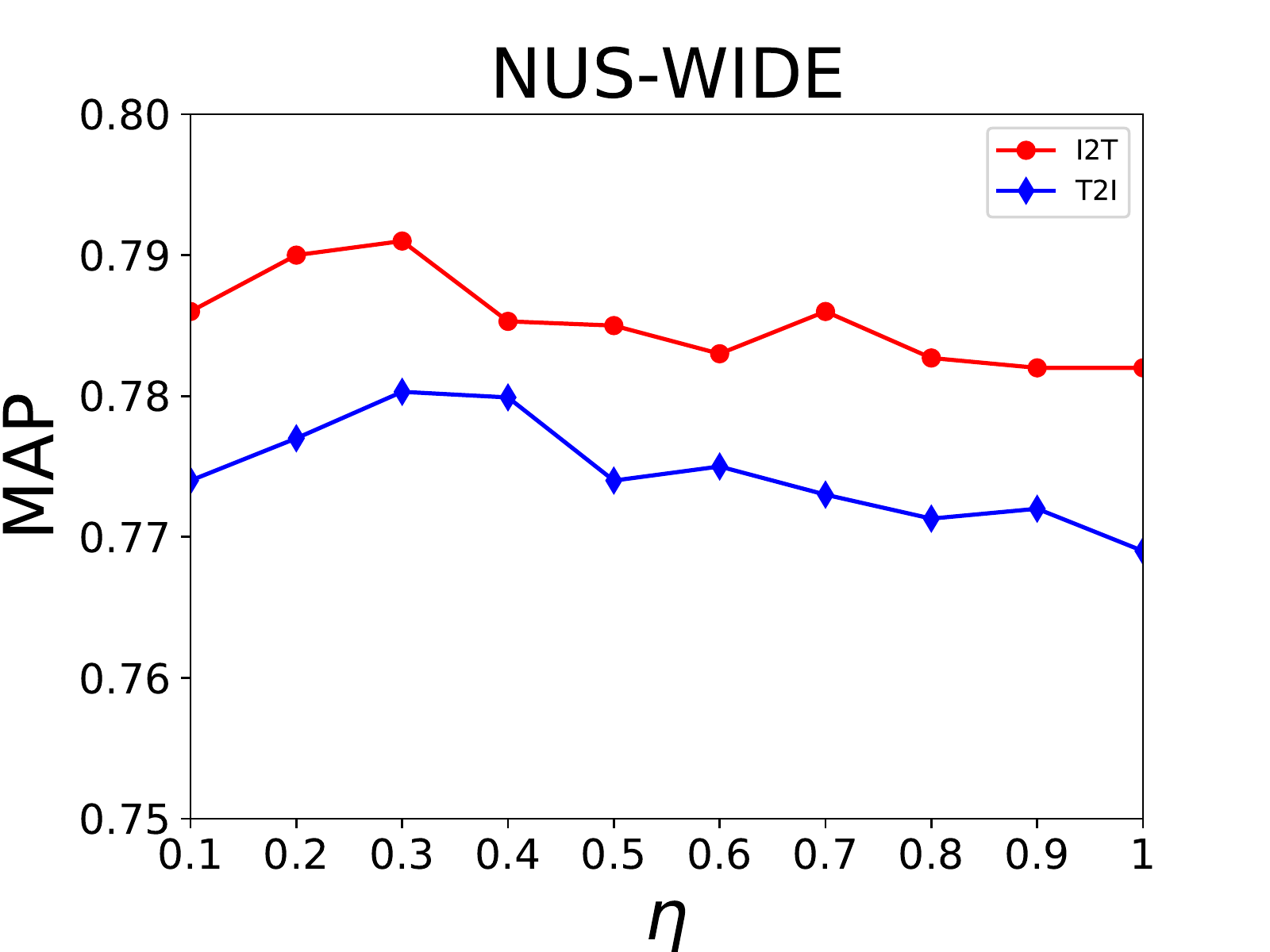}
		\end{minipage}%
	}%
	\quad
	\subfigure[]{
		\begin{minipage}[t]{0.33\textwidth}
			\centering
			\includegraphics[width=\linewidth]{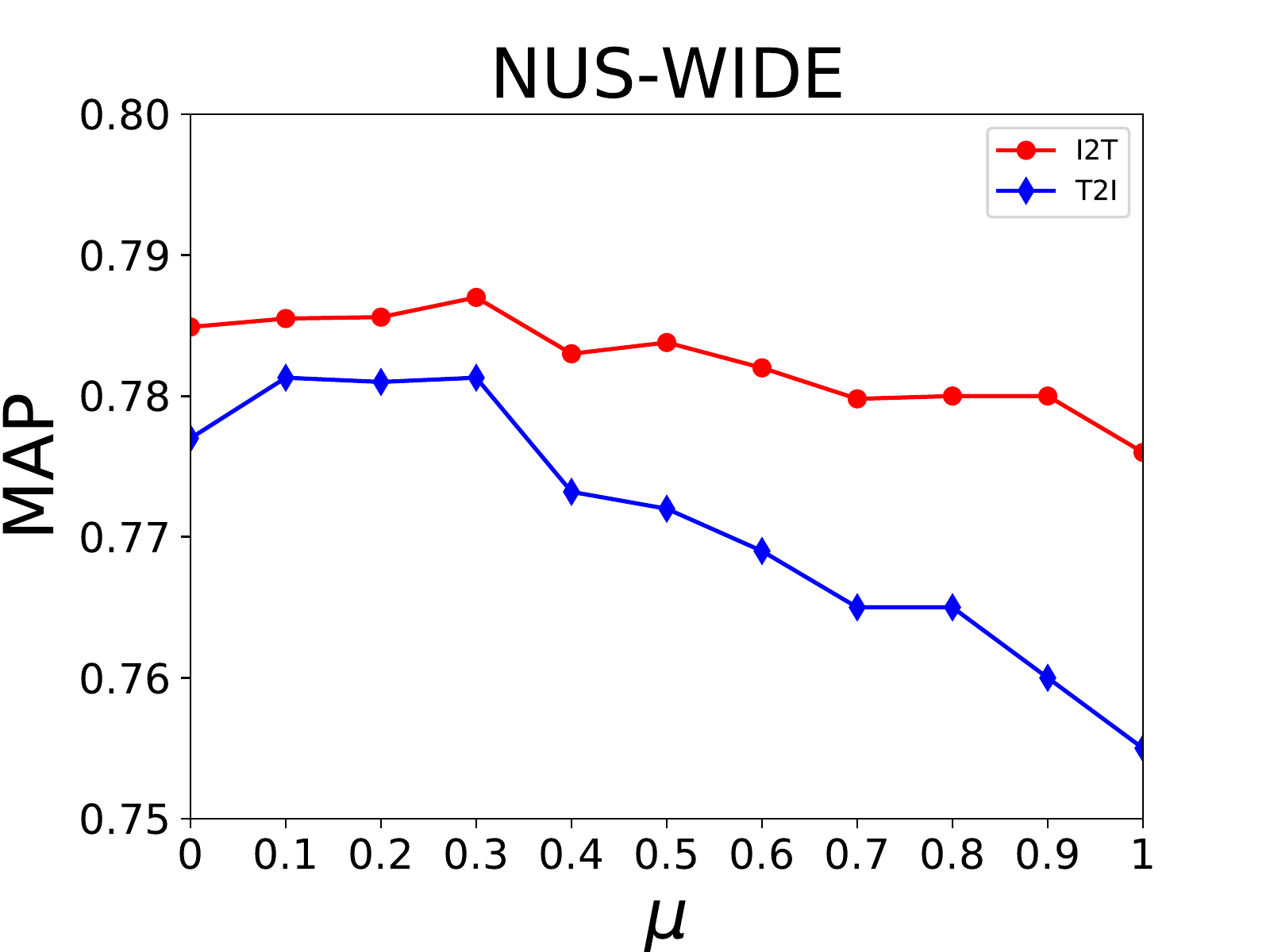}
		\end{minipage}%
	}%
	\subfigure[]{
		\begin{minipage}[t]{0.33\textwidth}
			\centering
			\includegraphics[width=\linewidth]{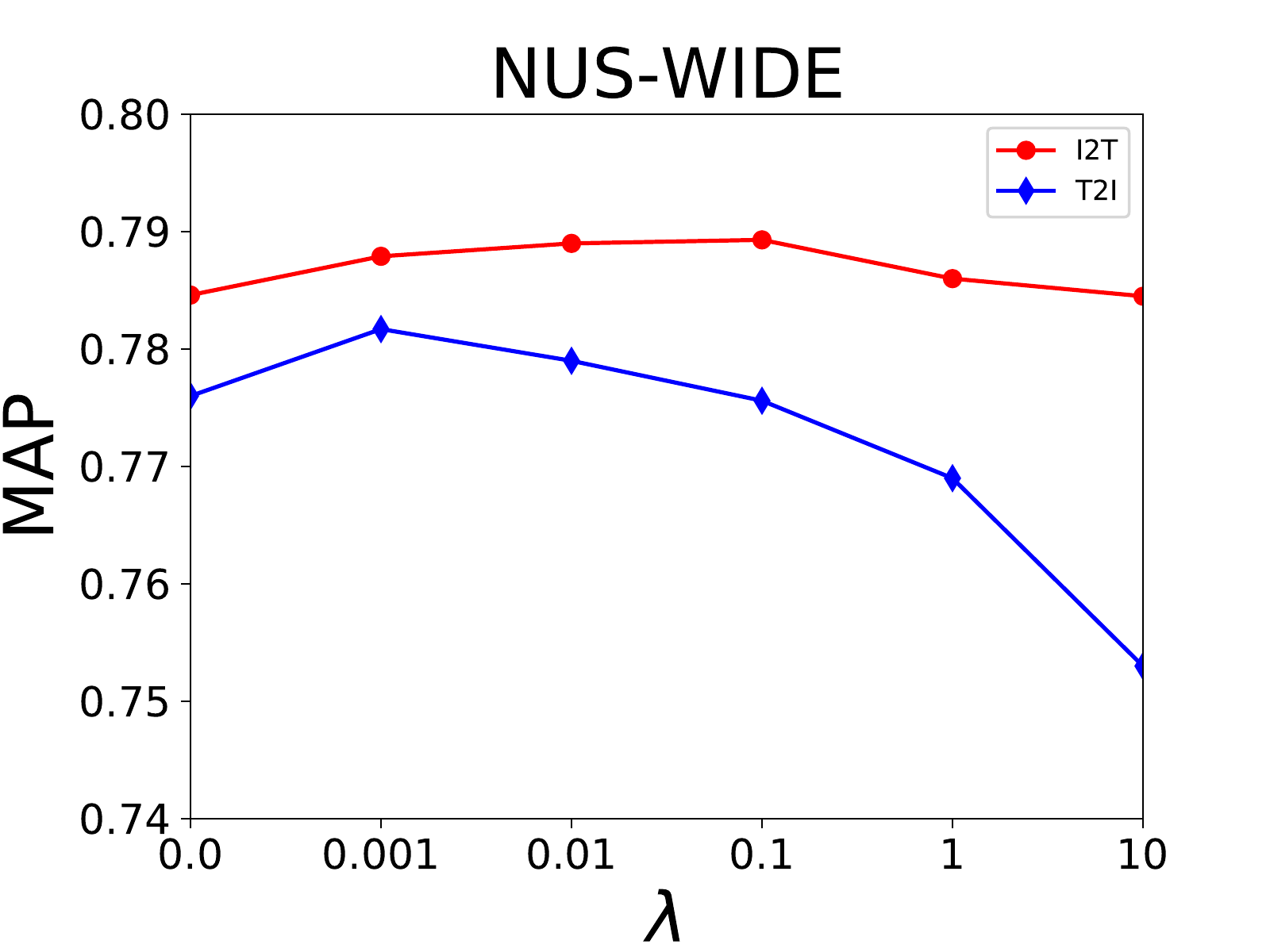}
		\end{minipage}%
	}%
	\subfigure[]{
		\begin{minipage}[t]{0.33\textwidth}
			\centering
			\includegraphics[width=\linewidth]{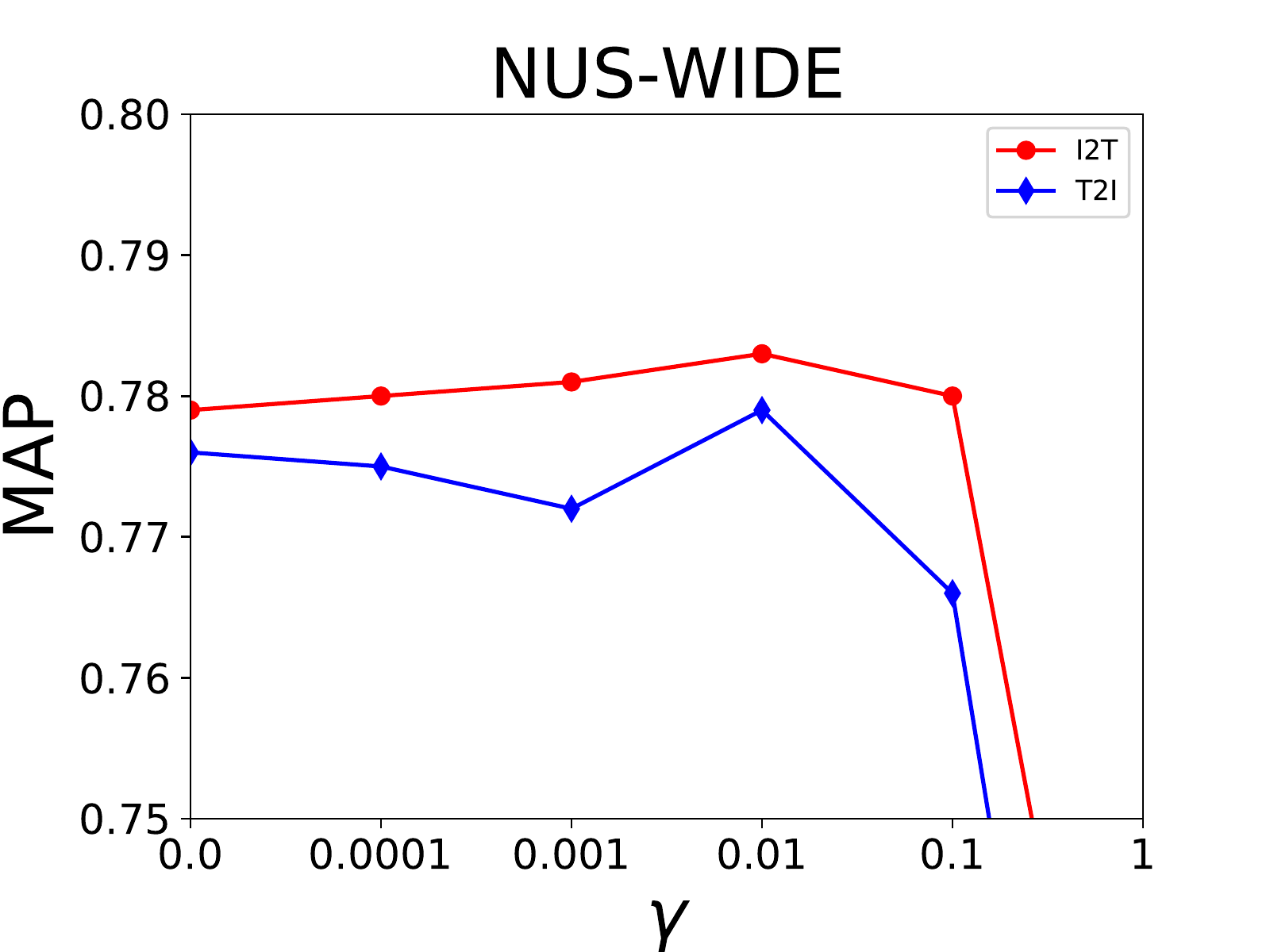}
		\end{minipage}%
	}%
	\quad
	\caption{A sensitivity analysis of the hyper-parameters over NUS-WIDE dataset.}
	\label{fig_par_nus}
\end{figure*}

\begin{figure*} 
	\subfigure[]{
		\begin{minipage}[t]{0.33\textwidth}
			\centering
			\includegraphics[width=\linewidth]{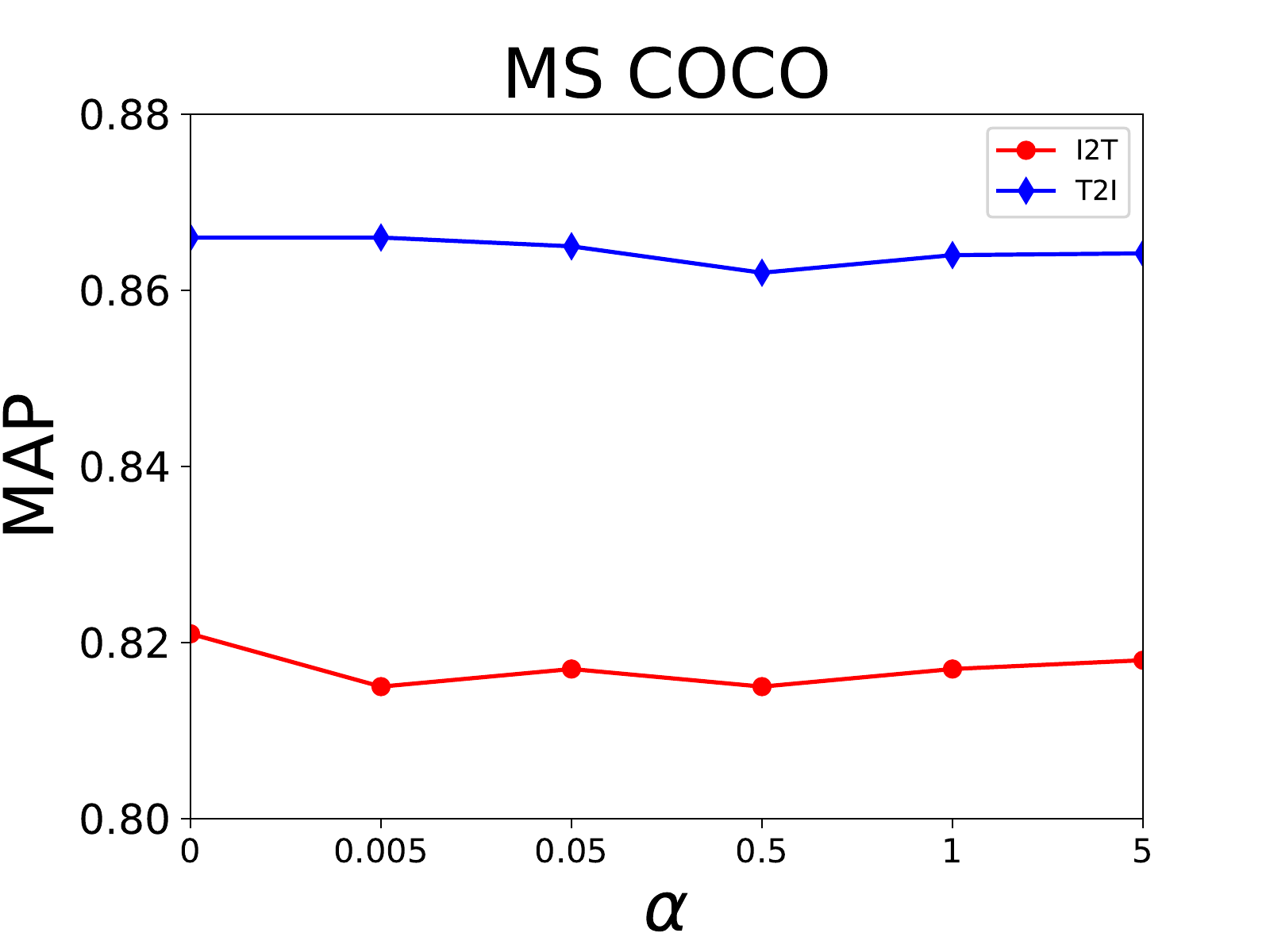}
		\end{minipage}%
	}%
	\subfigure[]{
		\begin{minipage}[t]{0.33\textwidth}
			\centering
			\includegraphics[width=\linewidth]{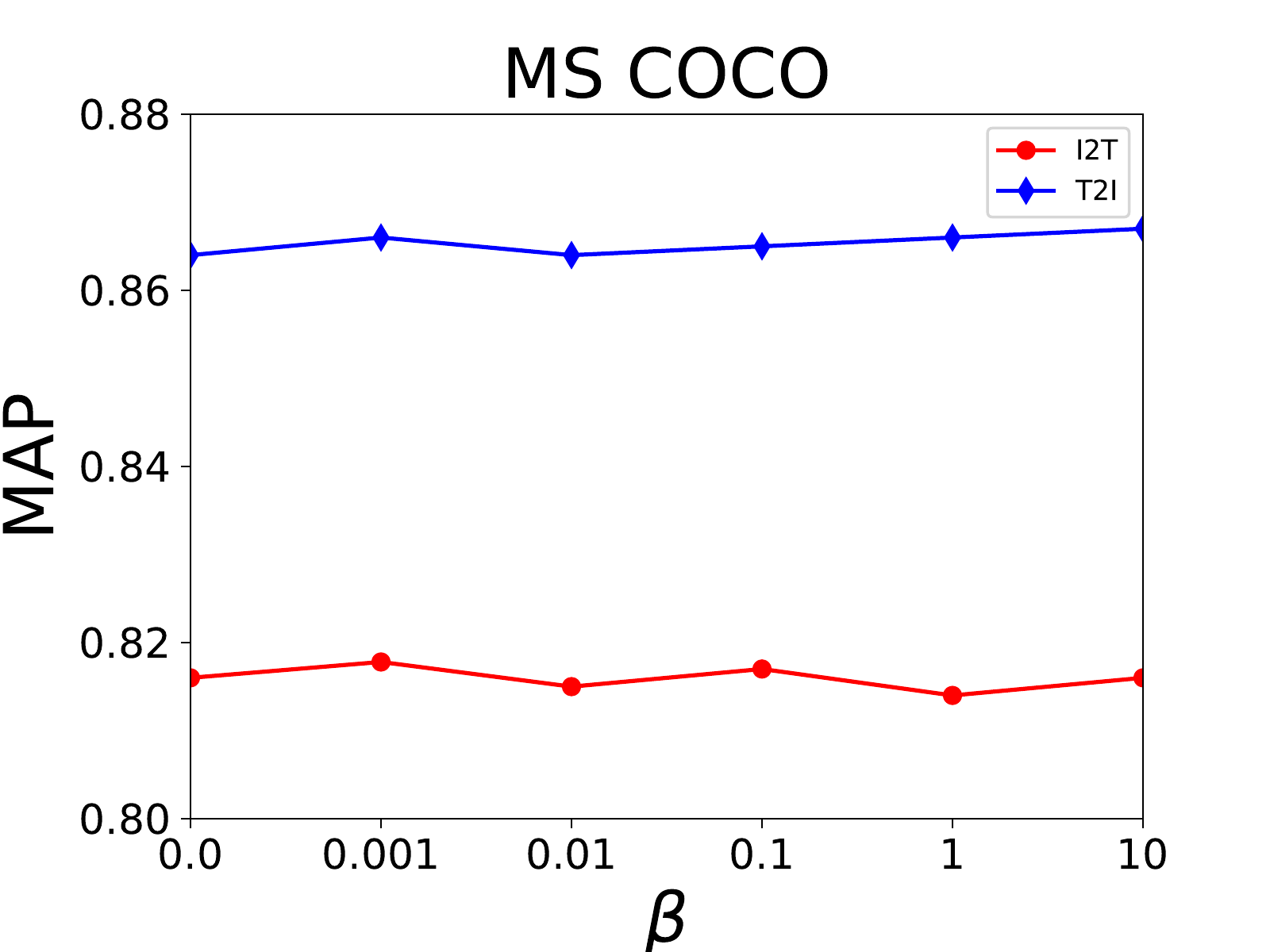}
		\end{minipage}%
	}%
	\subfigure[]{
		\begin{minipage}[t]{0.33\textwidth}
			\centering
			\includegraphics[width=\linewidth]{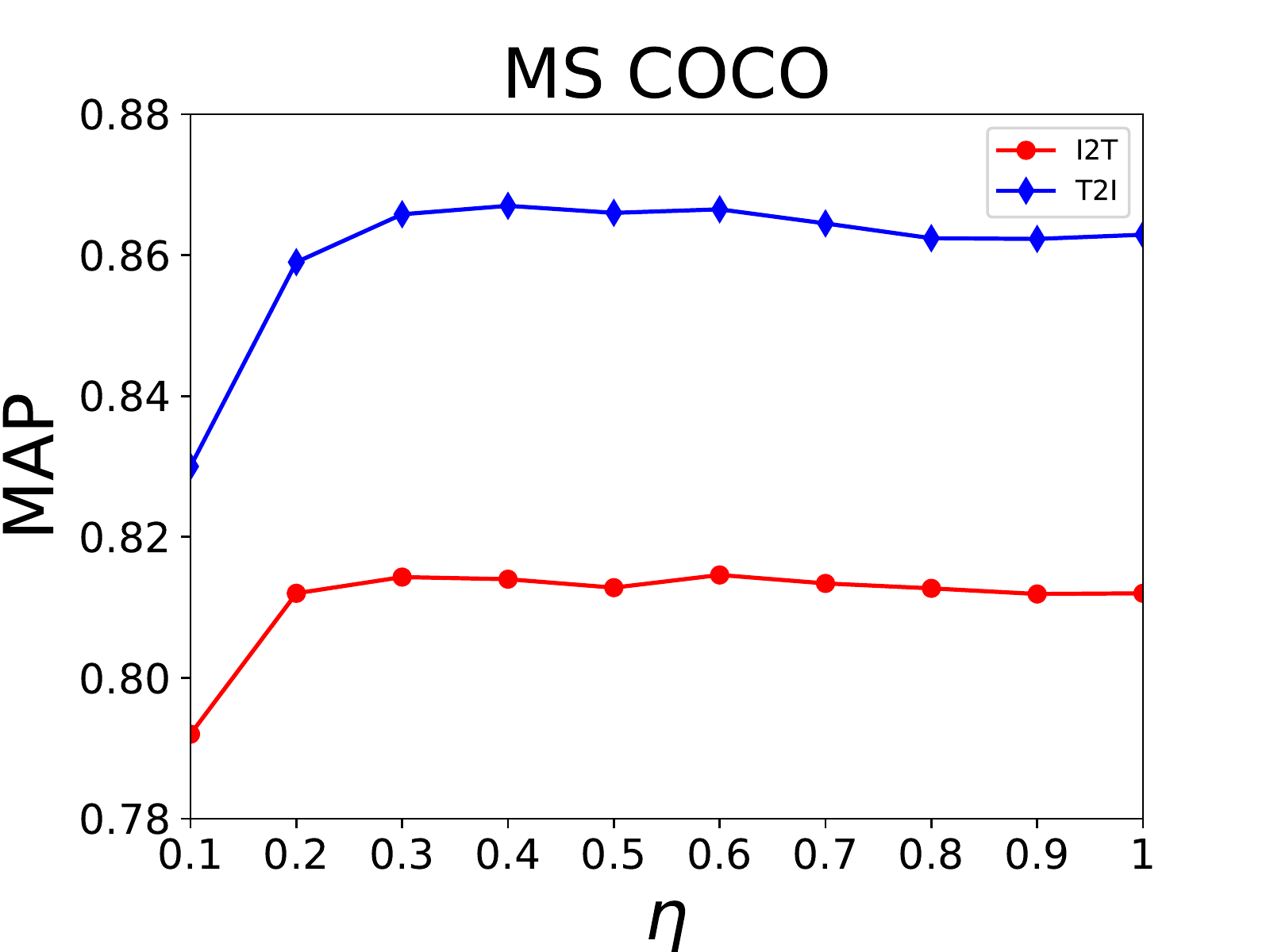}
		\end{minipage}%
	}%
	\quad
	\subfigure[]{
		\begin{minipage}[t]{0.33\textwidth}
			\centering
			\includegraphics[width=\linewidth]{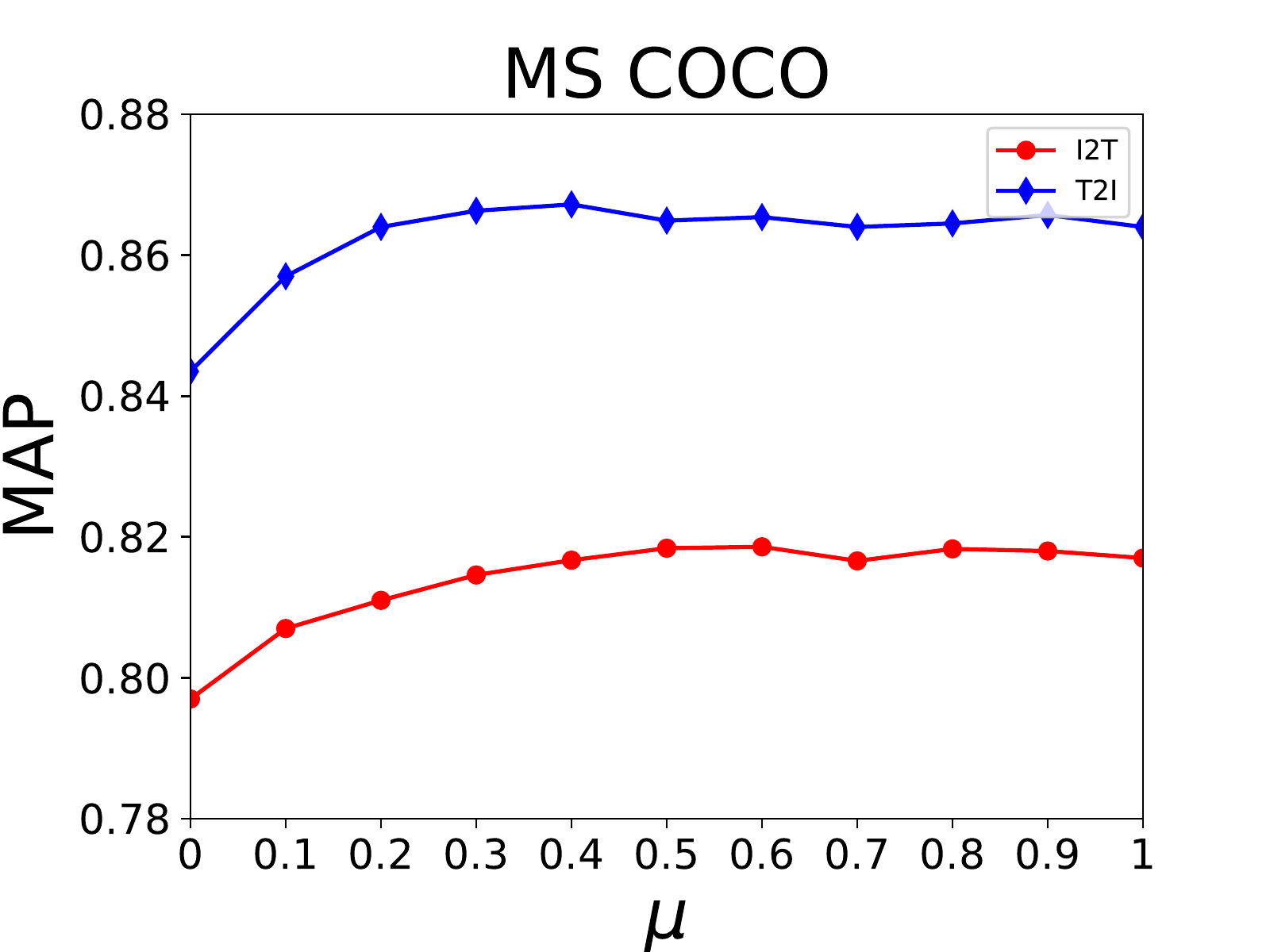}
		\end{minipage}%
	}%
	\subfigure[]{
		\begin{minipage}[t]{0.33\textwidth}
			\centering
			\includegraphics[width=\linewidth]{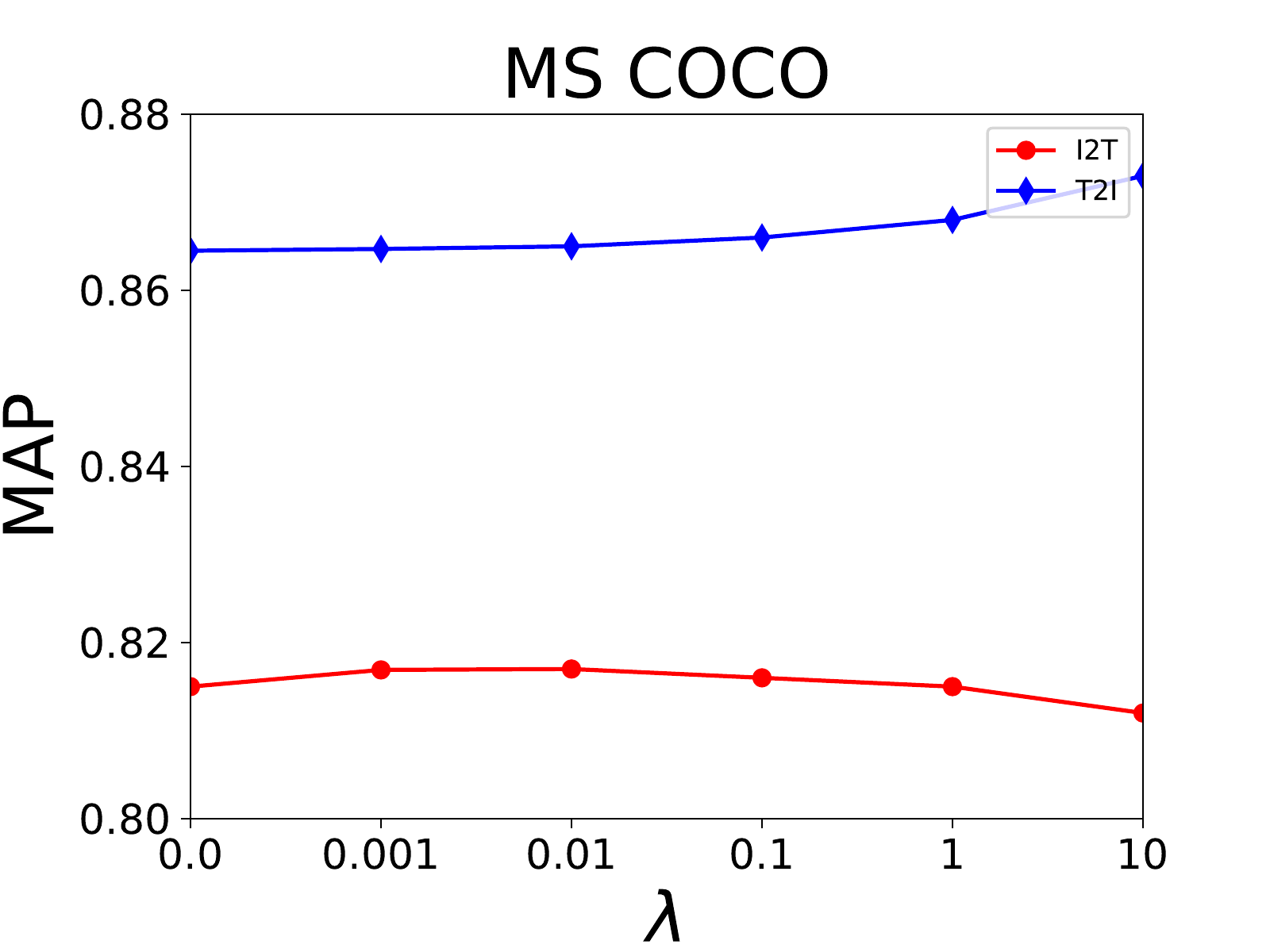}
		\end{minipage}%
	}%
	\subfigure[]{
		\begin{minipage}[t]{0.33\textwidth}
			\centering
			\includegraphics[width=\linewidth]{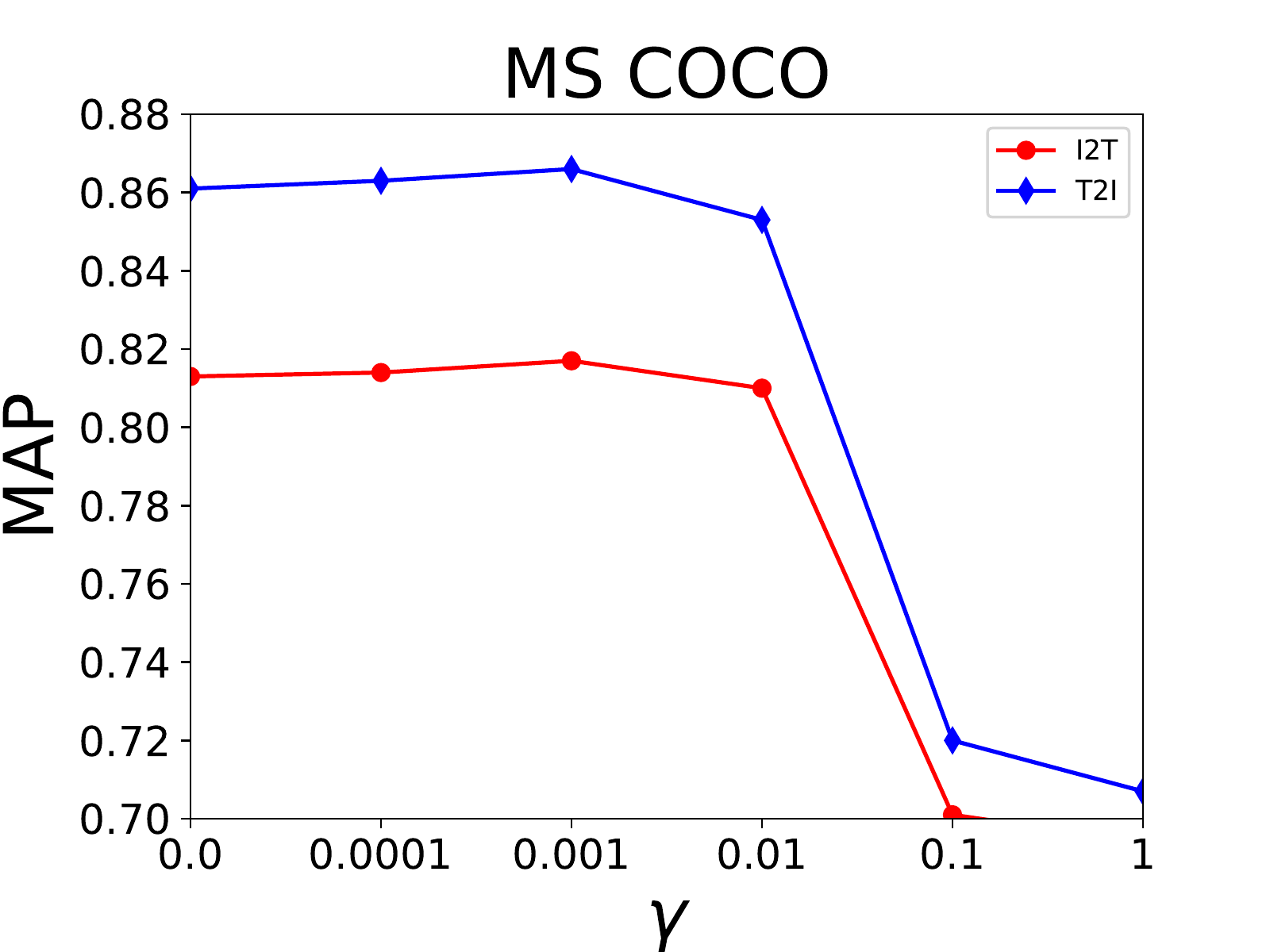}
		\end{minipage}%
	}%
	\quad
	\caption{A sensitivity analysis of the hyper-parameters over MS COCO dataset.}
	\label{fig_par_coco}
\end{figure*}

\begin{figure*} 
\subfigure[]{
	\begin{minipage}[t]{0.33\textwidth}
		\centering
		\includegraphics[width=\linewidth]{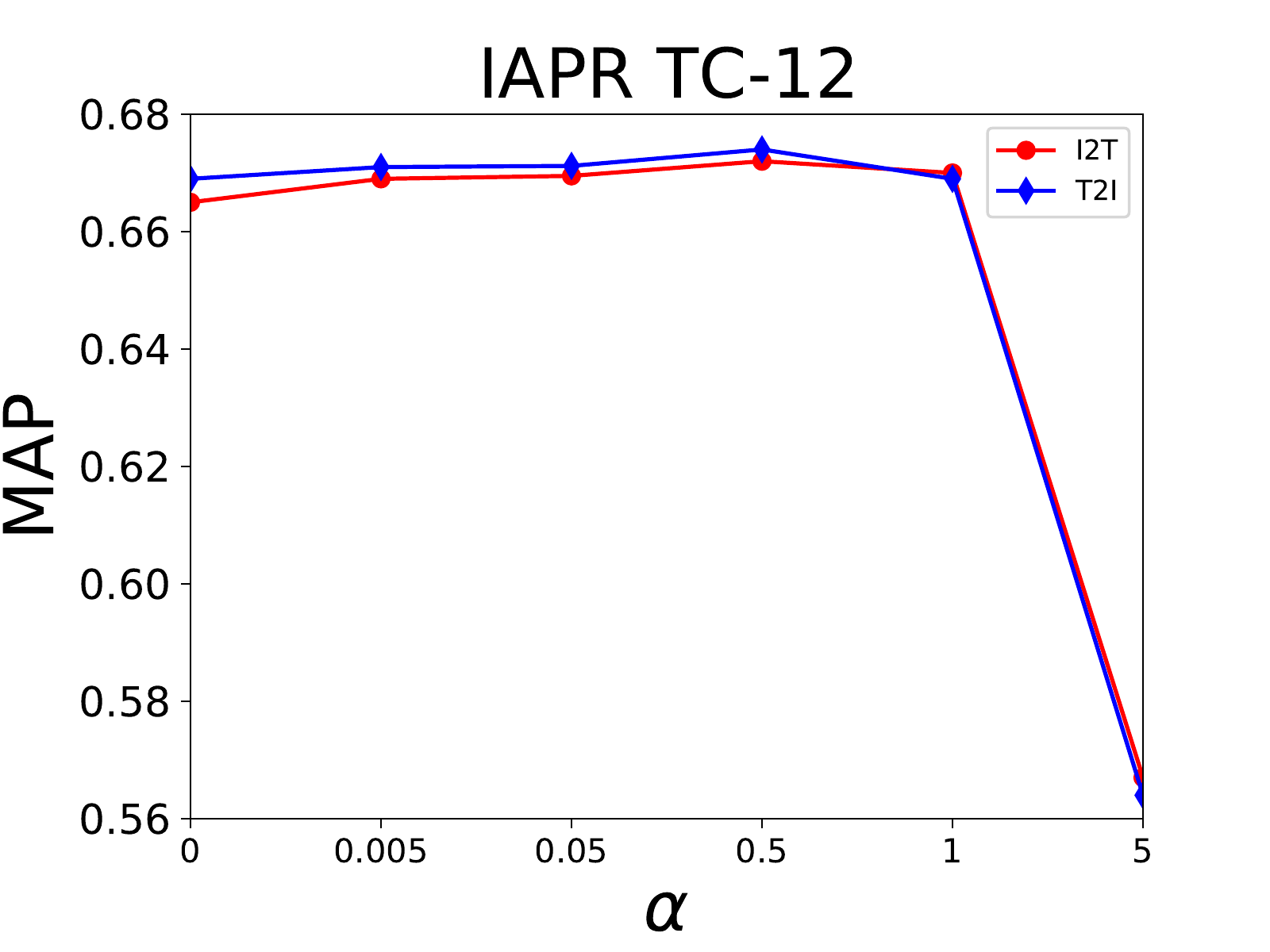}
	\end{minipage}%
}%
\subfigure[]{
	\begin{minipage}[t]{0.33\textwidth}
		\centering
		\includegraphics[width=\linewidth]{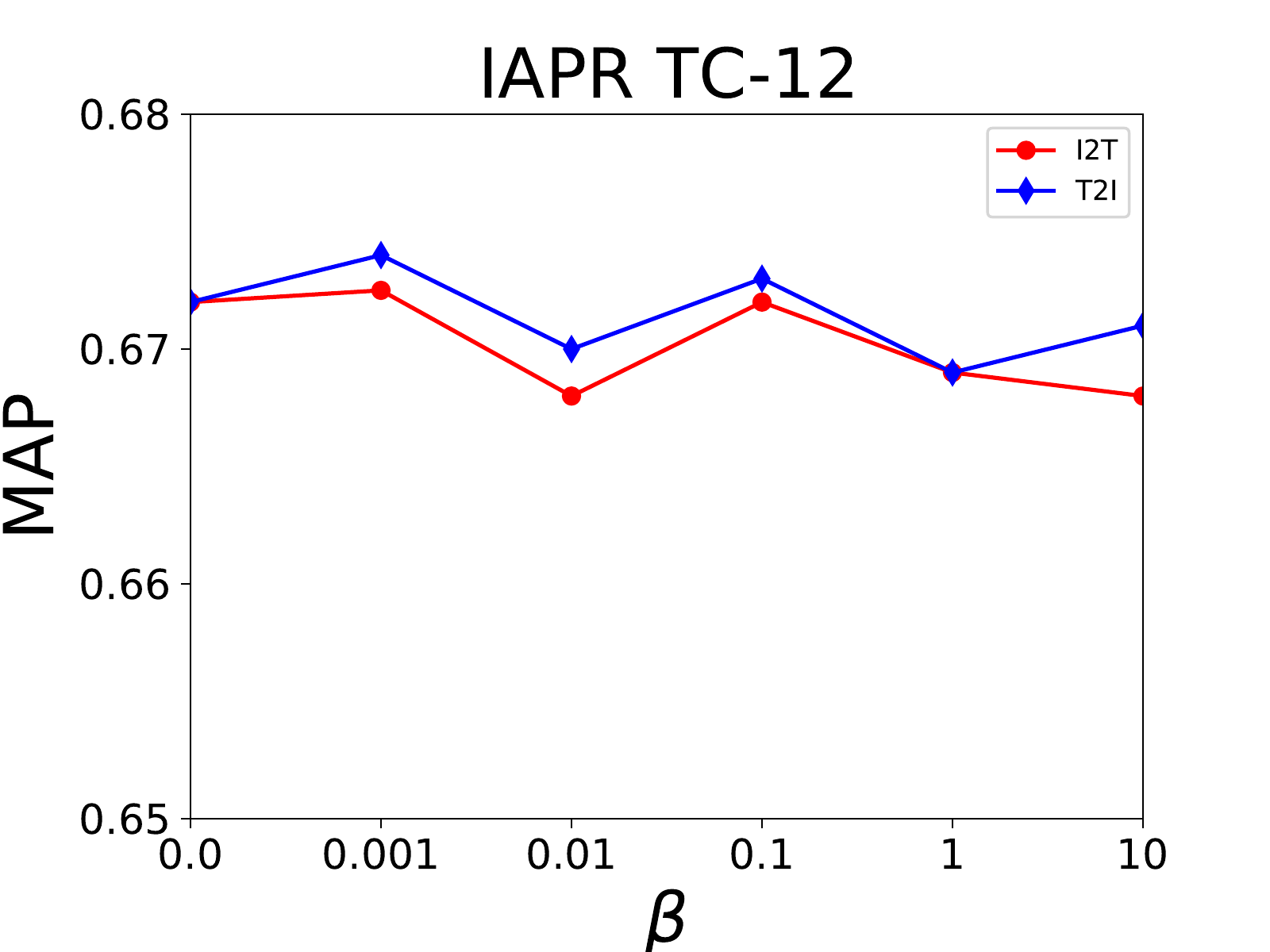}
	\end{minipage}%
}%
\subfigure[]{
	\begin{minipage}[t]{0.33\textwidth}
		\centering
		\includegraphics[width=\linewidth]{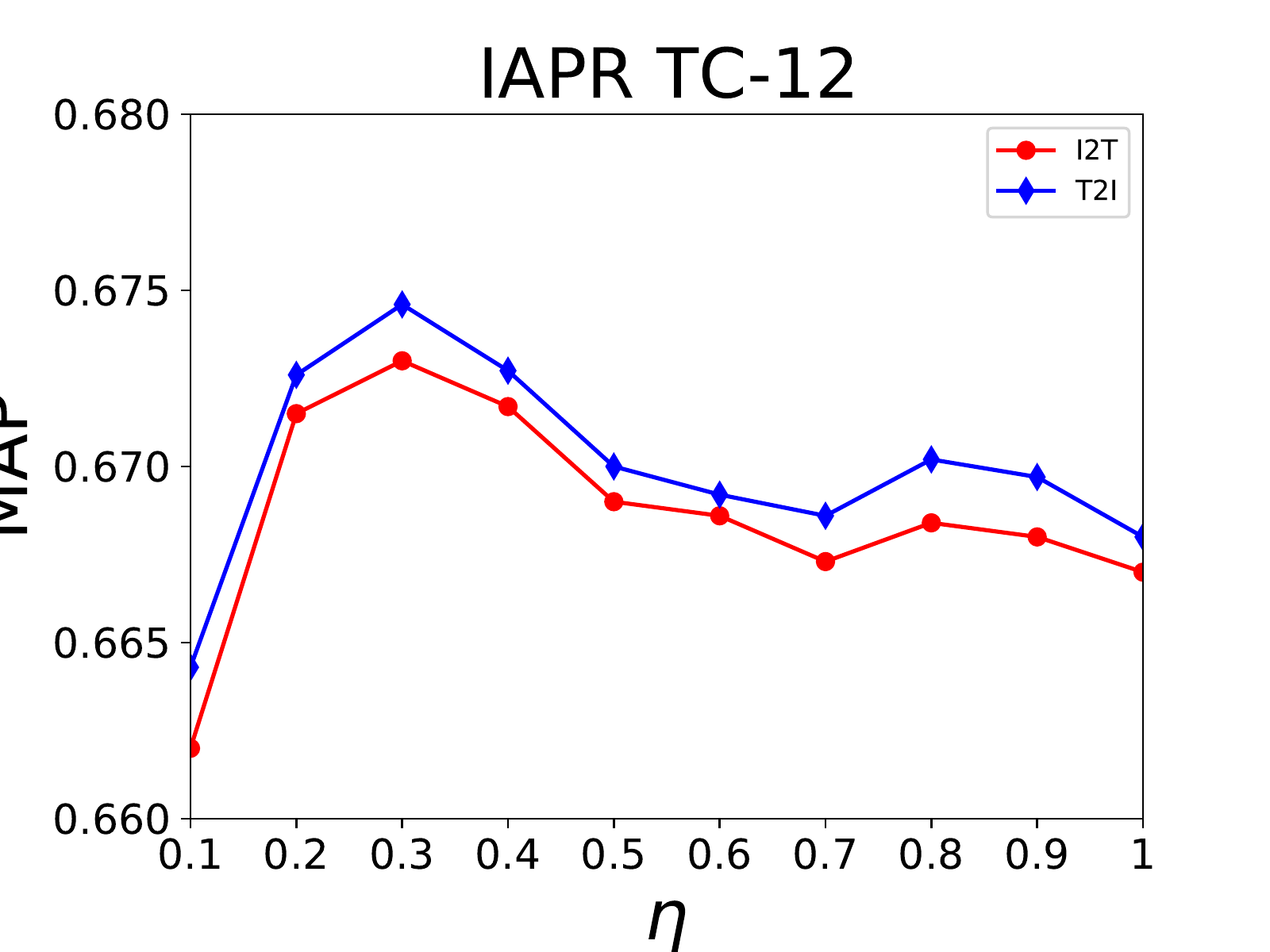}
	\end{minipage}%
}%
\quad
\subfigure[]{
	\begin{minipage}[t]{0.33\textwidth}
		\centering
		\includegraphics[width=\linewidth]{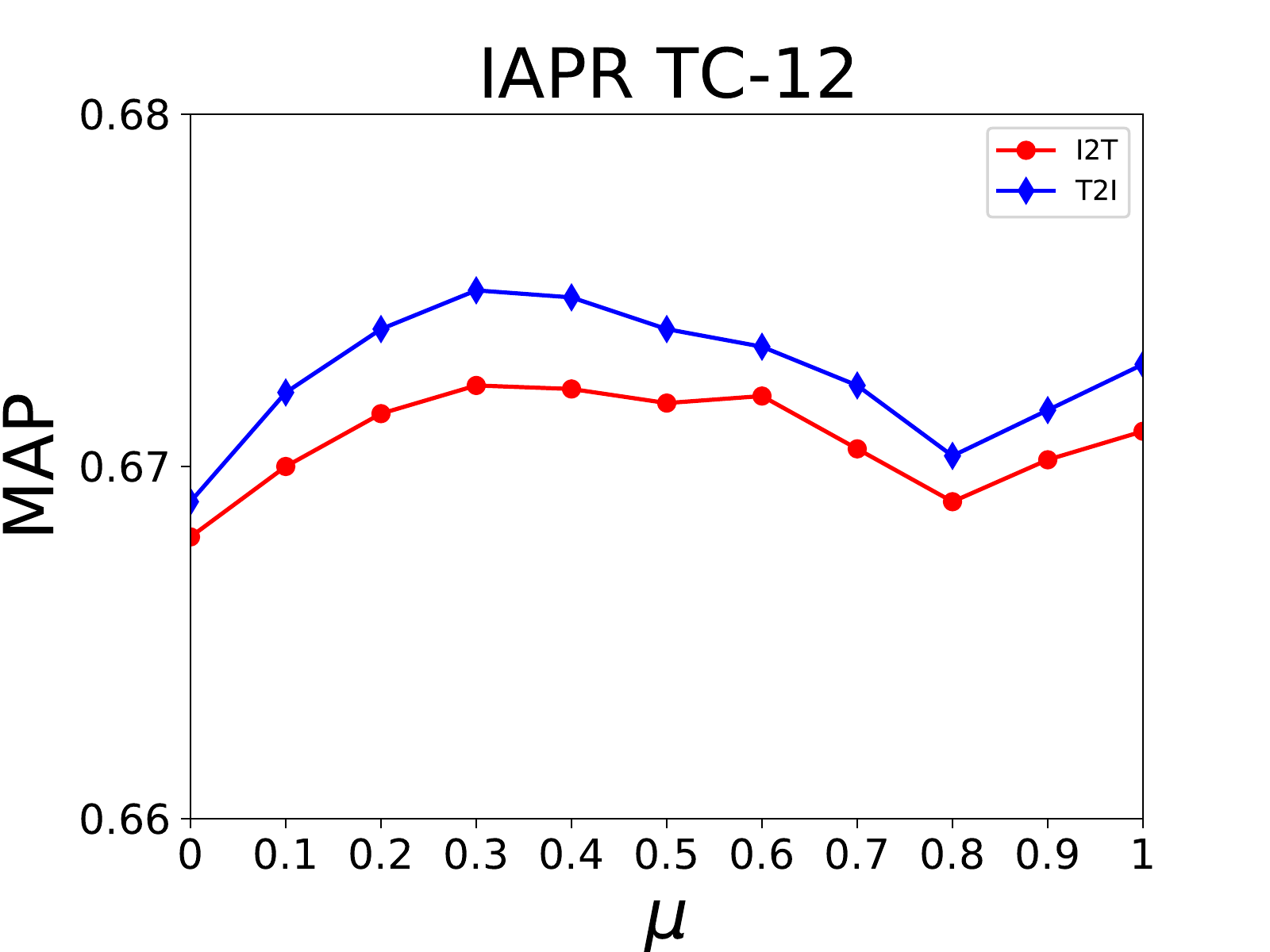}
	\end{minipage}%
}%
\subfigure[]{
	\begin{minipage}[t]{0.33\textwidth}
		\centering
		\includegraphics[width=\linewidth]{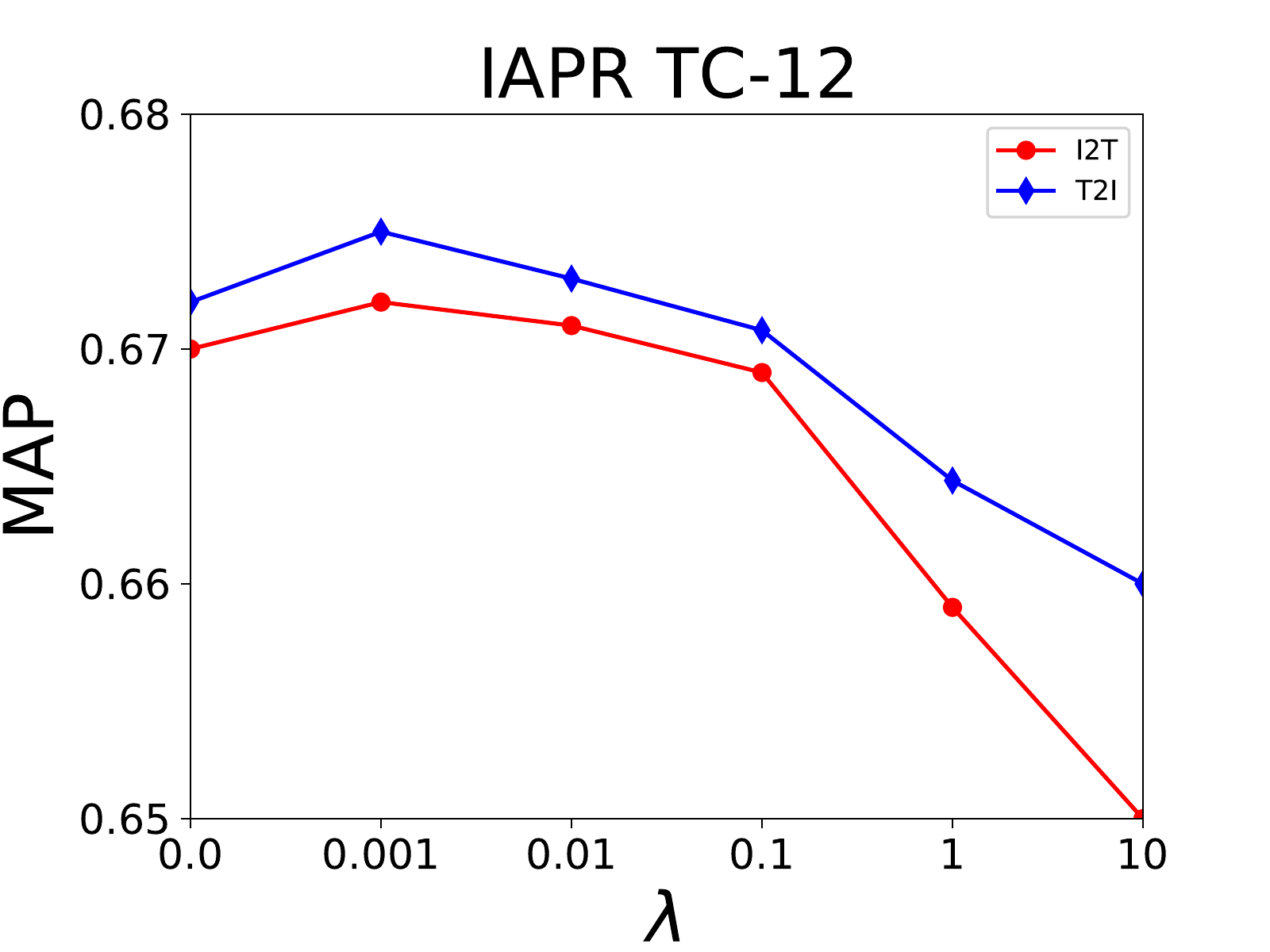}
	\end{minipage}%
}%
\subfigure[]{
	\begin{minipage}[t]{0.33\textwidth}
		\centering
		\includegraphics[width=\linewidth]{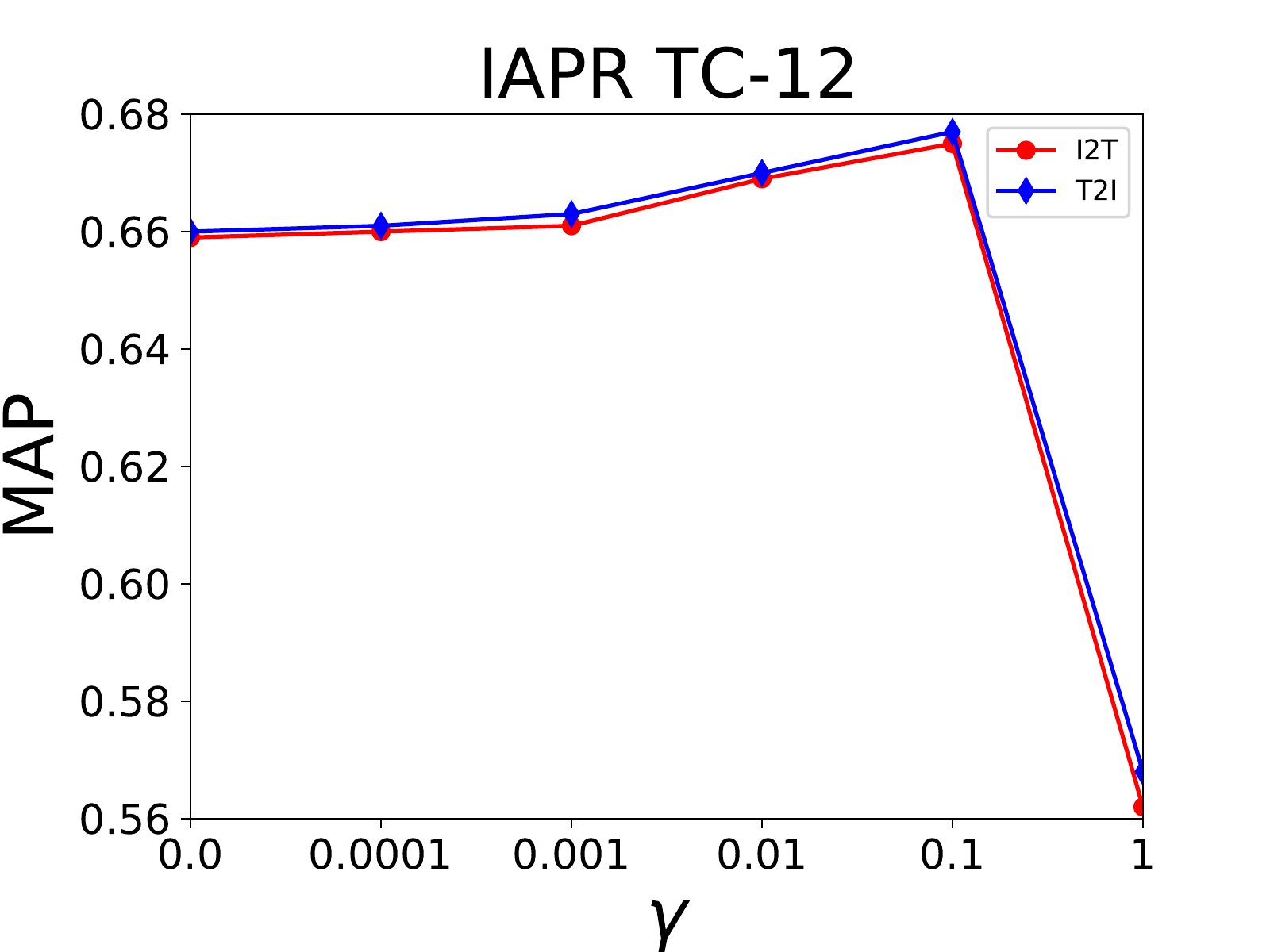}
	\end{minipage}%
}%
\quad
\caption{A sensitivity analysis of the hyper-parameters over IAPR TC-12 dataset.}
\label{fig_par_iaprtc}
\end{figure*}

\subsubsection{Ablation Study}
We investigate two DCHML variants: (1) DCHML$_{Q}$ is a DCHML variant without quantization loss, i.e., $\gamma = 0$;  (2) DCHML$_{P}$ is a DCHML variant replacing the $\mathcal{L}_{S}^m$ with the typical pairwise loss $\mathcal{L}_{P}^m$ which is similar to \cite{li2018self} and defined as follow:
\begin{equation}
\begin{aligned}
\min\limits_{\boldsymbol{W}^m}\mathcal{L}^m_P &=- \sum\limits_{i=1}^n \sum\limits_{j = 1} ^ c(s_{ij}\Omega_{ij}^m - log(1 + e^{\Omega_{ij}}))^m \\
& \ \ \ \ + \gamma\sum\limits_{i=1}^n\left\|\boldsymbol{\hat{b}}_i^m - \boldsymbol{c}_i \right\|_F^2.
\end{aligned}
\label{lp}
\end{equation}
where $\Omega_{ij}^m= \frac{1}{2}\boldsymbol{\hat{b}}_i^{mT}\boldsymbol{g}_j$; $\boldsymbol{g}_j$ is the proxy hash code of category$j$; $\boldsymbol{\hat{b}}_i^m = \mathcal{H}^m(\boldsymbol{x}_i^m;\boldsymbol{W}^m)$ is the output of modality-specific hashing network  with a datapoint $\boldsymbol{x}_i^m$ as input, and $\boldsymbol{W}^m$ represents the set of parameters in the hashing network of modality $m$;  $s_{ij}$ is the similarity between a datapoint $\boldsymbol{x}_i^m$ ($m=v\ or\ t$) and category $j$, i.e., when datapoint $\boldsymbol{x}_i^m$ belong to category $j$, $s_{ij}=1$, otherwise $s_{ij}=0$; and $\boldsymbol{c}_i = sgn(\boldsymbol{\hat{b}}_i^v + \boldsymbol{\hat{b}}_i^t)$. 

The MAP results of DCHML and its variants with different hash code length on the three datasets are shown in Table \ref{map_vall}.  From Table \ref{map_vall}, it can be observed that: (1) It will increase the retrieval performance by employing the quantization loss. These results indicate that it is  necessary  to control the quantization error though a quantization loss. 
(3) Our proposed \textit{margin-dynamic-softmax loss} is better than the traditional pairwise loss. For example, For instance, compared with DCMH$_P$, the MAP results of DCHML for $``$I2T$"$ have an average increase of 16.3\%, 18.5\% and 12.3\% on datasets NUS-WIDE, MS COCO and IAPR TC-12, respectively. These results demonstrate that our proposed \textit{margin-dynamic-softmax loss} can make the learned hash codes preserve semantic structures information sufficiently to further improve the cross-modal retrieval performance.

\subsubsection{Sensitivity to Hyper-parameters}
\label{sp}
We study the influence of the hyper-parameters $\alpha,\beta,\eta, \mu$, $\lambda$ and $\gamma$  with the hash code length being 64-bits on all the three datasets, and the results are shown in the Figure \ref{fig_par_nus}, \ref{fig_par_coco} and \ref{fig_par_iaprtc}. For the hyper-parameter $\alpha$, to investigate its influence, we evaluate the MAP values of our proposed DCHML by varying $\alpha$ from $0$ to $5$ with the other hyper-parameters fixed, and the results on the three datasets are shown in the Figure \ref{fig_par_nus} (a), \ref{fig_par_coco} (a) and \ref{fig_par_iaprtc} (a), respectively. It can be found that DCHML achieves good performance with $\alpha$ being 0.05 for all the three datasets. To investigate the influence of hyper-parameter $\beta$, we evaluate the MAP values of DCHML by varying $\beta$ from $0$ to $10$ with the other hyper-parameters fixed, and the results on the three datasets are shown in the Figure \ref{fig_par_nus} (b), \ref{fig_par_coco} (b) and \ref{fig_par_iaprtc} (b), respectively. It can be found that DCHML achieves good performance with $\beta$ being 0.1 for all the three datasets. Moreover, to investigate the influence of hyper-parameter $\eta$, we evaluate the MAP values of DCHML by varying $\eta$ from $0.1$ to $1$ with the other hyper-parameters fixed, and the results on the three datasets are shown in the Figure \ref{fig_par_nus} (c), \ref{fig_par_coco} (c) and \ref{fig_par_iaprtc} (c), respectively. It can be found that DCHML achieves good performance with $\eta$ being 0.3 for all the three datasets. For the hyper-parameter $\mu$, to investigate its influence, we evaluate the MAP values of our proposed DCHML by varying $\mu$ from $0$ to $1$ with the other hyper-parameters fixed, and the results on the three datasets are shown in the Figure \ref{fig_par_nus} (d), \ref{fig_par_coco} (d) and \ref{fig_par_iaprtc} (d), respectively. It can be found that DCHML achieves good performance with $\alpha$ being 0.3 for all the three datasets. To investigate the influence of hyper-parameter $\lambda$, we evaluate the MAP values of DCHML by varying $\lambda$ from $0$ to $10$ with the other hyper-parameters fixed, and the results on the three datasets are shown in the Figure \ref{fig_par_nus} (e), \ref{fig_par_coco} (e) and \ref{fig_par_iaprtc} (e), respectively. It can be found that DCHML achieves good performance with $\lambda$ being 0.001 for all the three datasets. Finally, to investigate the influence of hyper-parameter $\gamma$, we evaluate the MAP values of DCHML by varying $\gamma$ from $0$ to $1$ with the other hyper-parameters fixed, and the results on the three datasets are shown in the Figure \ref{fig_par_nus} (f), \ref{fig_par_coco} (f) and \ref{fig_par_iaprtc} (f), respectively. It can be found that DCHML achieves good performance with $\lambda$ being 0.01, 0.001 and 0.1 on NUS-WIDE, MS COCO and IAPR TC-12, respectively. 

\section{Conclusion}
In this paper, we have proposed a novel Deep Cross-Modal Hashing via Margin-dynamic-softmax Loss, called DCHML. Without defining the similarity between datapoints, DCHML can generate high-quality hash codes with cross-modal similarity and semantic label information preserved sufficiently by minimizing a novel \textit{margin-dynamic-softmax loss}. Extensive experiments on three widely used benchmark datasets have demonstrated that the proposed DCHML method outperforms the state-of-the-art baselines in the cross-modal retrieval task.


%

\appendices


\section*{Acknowledgment}
The work is supported by National Key R\&D Plan (No. 2018YFB1005100), NSFC (No. 61772076, 61751201 and 61602197), NSFB (No. Z181100008918002), and the funds of Beijing Advanced Innovation Center for Language Resources (No. TYZ19005).

\ifCLASSOPTIONcaptionsoff
  \newpage
\fi



\bibliographystyle{IEEEtran}
\bibliography{DCHML}
%

%

\begin{IEEEbiography}[{\includegraphics[width=1in,height=1.25in,clip,keepaspectratio]{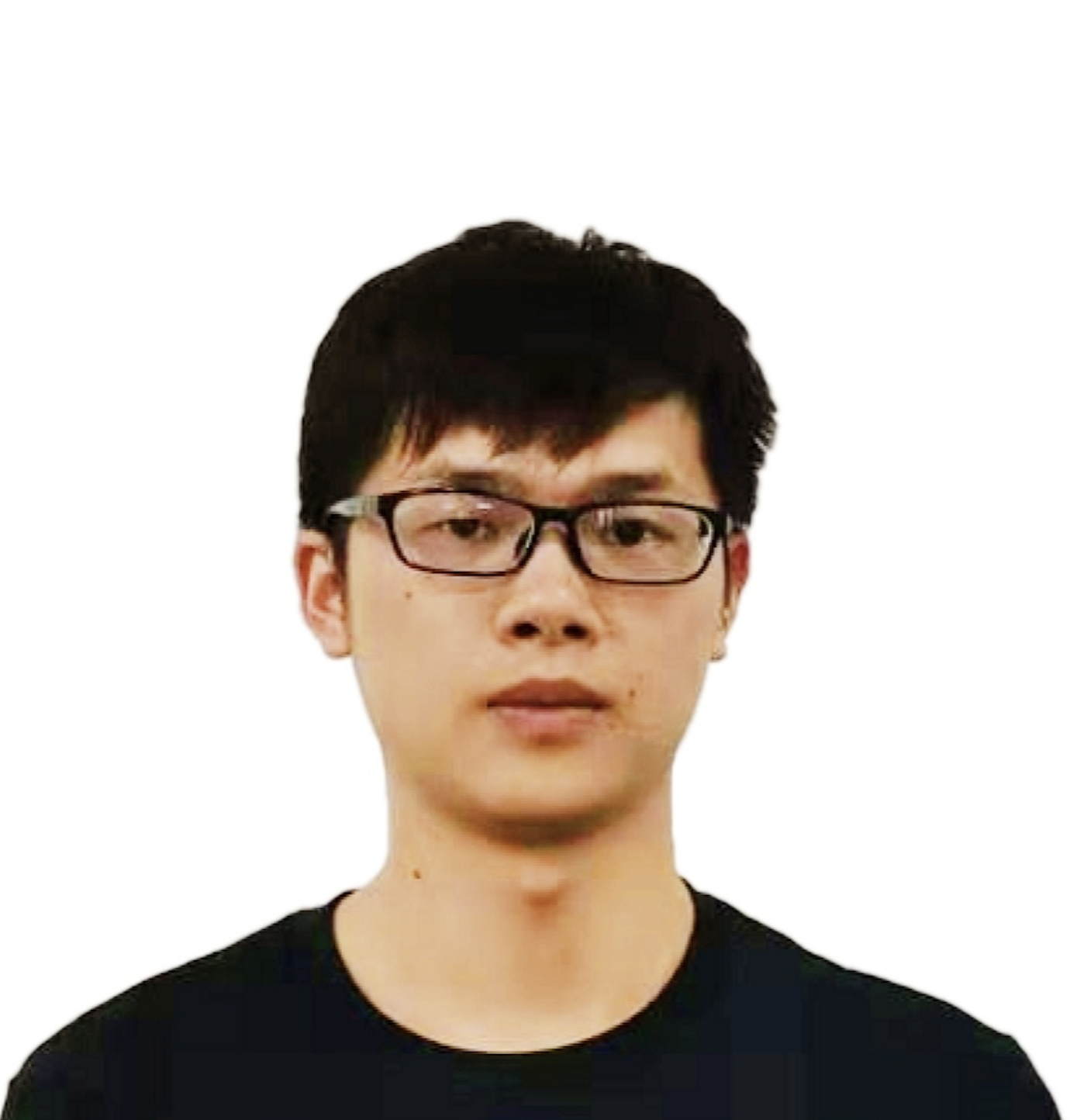}}]{Rong-Cheng Tu}
	received the bachelor's degree from Beijing Institute of Technology, China, in 2018. He is currently working toward the Ph.D in the Department of Computer Science and Technology, Beijing Institute of Technology, China. His research interests are in deep learning and learning to hash.
\end{IEEEbiography}

\begin{IEEEbiography}[{\includegraphics[width=1in,height=1.25in,clip,keepaspectratio]{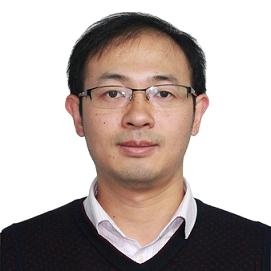}}]{Xian-Ling Mao}
	received the Ph.D. science degree from Peking University, China, in 2012. He is currently an Associate Professor with the Department of Computer Science and Technology, Beijing Institute of Technology, China. His major research interests include deep learning, machine learning, information retrieval, natural language processing, artificial intelligence and network data mining.
\end{IEEEbiography}

\begin{IEEEbiography}[{\includegraphics[width=1in,height=1.25in,clip,keepaspectratio]{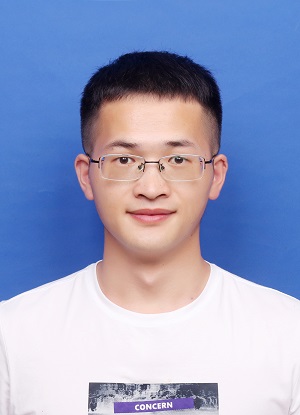}}]{Rongxin Tu} 
	is currently a graduate student in the school of Information Management, Jiangxi University of Finance and Economics, Nanchang, China. His research interests include multimedia security, image processing, etc.
\end{IEEEbiography}

\begin{IEEEbiography}[{\includegraphics[width=1in,height=1.25in,clip,keepaspectratio]{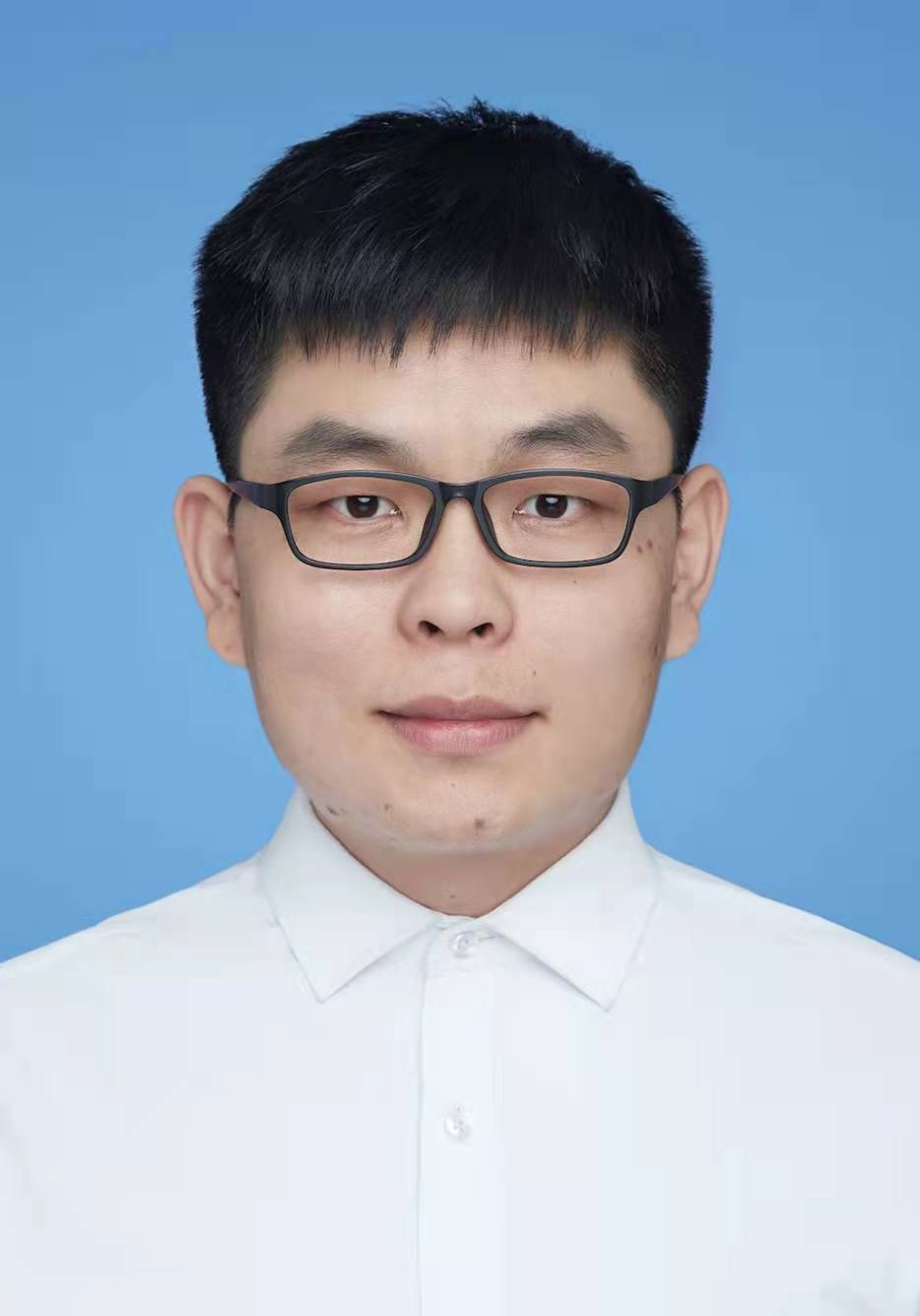}}]{Binbin Bian}
	is currently a assistant researcher in Artificial Intelligence and Big Data Research Center, Beijing Academy of Science and Technology, and a Ph.D. candidate of Beijing Institute of Technology. His research interests span information retrieval, machine learning, deep learning, data mining, and natural language processing.
\end{IEEEbiography}

\begin{IEEEbiography}[{\includegraphics[width=1in,height=1.25in,clip,keepaspectratio]{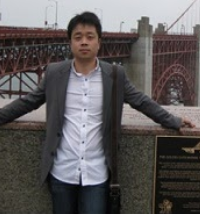}}]{Wei wei}
received the Ph.D. degree from the Huazhong University of Science and Technology, China, in 2012. He is currently an Associate Professor with School of Computer Science and Technology and the Director of Cognitive Computing and Intelligent Information Processing (CCIIP) Laboratory in Huazhong University of Science and Technology, China. His major research interests include information retrieval, natural language processing, artificial intelligence, data mining (text mining), statistics machine learning, social media analysis and mining recommender system.
\end{IEEEbiography}

\begin{IEEEbiography}[{\includegraphics[width=1in,height=1.25in,clip,keepaspectratio]{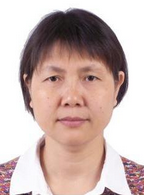}}]{Heyan Huang}
	received the bachelor's degree from Wuhan University of Surveying and Mapping, China, in 1983, the master's degree from  National University of Defense Technology, China, in 1986, and the Ph.D. degree from the Institute of Computing Technology, Chinese Academy of Sciences, China, in 1989. She is currently a professor and the Dean with the Department of Computer Science and Technology, Beijing Institute of Technology, China. Her major research interests include natural language processing, information content security, intelligent application system.
\end{IEEEbiography}



\end{document}